\documentclass[11pt]{article}
\usepackage{bbm}
\usepackage{eqnarray,amsmath,amsfonts,amsthm,mathrsfs}
\usepackage{color}
\usepackage{bm}
\usepackage{amssymb}
\usepackage{epsfig}
\usepackage{epsf}
\usepackage{float}
\usepackage{subfigure}
\usepackage{amsfonts,amsmath,amsthm,amssymb,graphicx,float,fancyhdr,multirow,hyperref}
\usepackage{booktabs,longtable,authblk}
\usepackage{mathrsfs,hhline}

\usepackage{fancyhdr}
\usepackage[top=2.5cm, bottom=2.5cm, left=3cm, right=3cm]{geometry}
\usepackage{graphicx}
\newtheorem{theorem}{Theorem}
\newtheorem{assumption}{Assumption}
\newtheorem{corollary}{Corollary}
\newtheorem{definition}{Definition}

\newtheorem{lemma}{Lemma}[section]
\newtheorem{proposition}{Proposition}[section]
\newtheorem{remark}{Remark}[section]
\newtheorem{example}{Example}[section]

\allowdisplaybreaks[4]

\def\begeqn{\begin{equation}}
\def\endeqn{\end{equation}}
\def\begth{\begin{theorem}}
\def\endth{\end{theorem}}
\def\begprop{\begin{proposition}}
\def\endprop{\end{proposition}}
\def\begcor{\begin{corollary}}
\def\endcor{\end{corollary}}
\def\begdef{\begin{definition}}
\def\enddef{\end{definition}}
\def\beglemm{\begin{lemma}}
\def\endlemm{\end{lemma}}
\def\begexm{\begin{example}}
\def\endexm{\end{example}}
\def\begrem{\begin{remark}}
\def\endrem{\end{remark}}
\def\begassum{\begin{assumption}}
\def\endassum{\end{assumption}}

\newtheorem{prop}{Proposition}
\newtheorem{cor}{Corollary}
\newtheorem{rem}{Remark}

\title{Coefficient-based Regularized Distribution Regression$^\dag$\footnotetext{\dag~The work described in this paper is supported by the National Natural Science Foundation of China (Grants Nos. U21A20426, 12171039, 12071356). The work by Z. C. Guo described in this paper is supported partially by the Zhejiang Provincial Natural Science Foundation of China [Grant No. LR20A010001] and the Fundamental Research Funds for the Central Universities [Project No. K20210337]. Lei Shi is also supported by Shanghai Science and Technology Research Program [No. 21JC1400600 and Project No. 20JC1412700].  The corresponding author is Zheng-Chu Guo. Email addresses: ymao@zju.edu.cn (Y. Mao), leishi@fudan.edu.cn (L. Shi), guozhengchu@zju.edu.cn (Z. C. Guo).}}
\author{Yuan Mao$^1$, Lei Shi$^2$ and Zheng-Chu Guo$^1$\\
\small $^1$ School of Mathematical Sciences, Zhejiang University, Hangzhou 310058, P. R. China \\
\small $^2$ School of Mathematical Sciences, Shanghai Key Laboratory for \\
\small      Contemporary Applied Mathematics, Fudan University, Shanghai 200433, P. R. China\\}
\date{}
\begin{document}

\maketitle
\begin{abstract}
In this paper, we consider the coefficient-based regularized distribution regression which aims to regress from probability measures to real-valued responses over a reproducing kernel Hilbert space (RKHS), where the regularization is put on the coefficients and kernels are assumed to be indefinite. The algorithm
involves two stages of sampling, the first stage sample consists of distributions
and the second stage sample is obtained from these distributions. Asymptotic behavior of the algorithm in different regularity ranges of the regression function are comprehensively studied and learning rates are derived via integral operator techniques. We get the optimal rates  under some mild conditions, which matches the one-stage sampled minimax optimal rate. Compared with the kernel methods for distribution regression in the literature, the algorithm under consideration does not require the kernel to be symmetric and positive semi-definite and hence provides a simple paradigm for designing indefinite kernel methods, which enriches the theme of the distribution regression. To the best of our knowledge, this is the first result for distribution regression with indefinite kernels, and our algorithm can improve the saturation effect.
\end{abstract}

\section{Introduction}\label{section: introduction}
    In machine learning, regression usually aims at making prediction from vectors in $\mathbb{R}^d$ to real valued response, but it may not be suitable to deal with some complex datasets such as functional data or probability distributions \cite{fan2019rkhs,szaboTwostageSampledLearning2015,szaboLearningTheoryDistribution2016,fangOptimalLearningRates2020,yuRobustKernelbasedDistribution2021,mueckeStochasticGradientDescent2021}. In this paper, we study distribution regression (DR) which analyzes and processes probability distributions on a compact metric space $ \mathcal{X} $ instead of vectors in Euclidean spaces with a two-stage sampling.
    For the first stage, dataset is given as $ \widetilde{D} = \{ (x_{i},y_{i}) \}_{i=1}^{m} \subseteq X \times Y $ where $ X $ is the input space of Borel probability measures on $ \mathcal{X} $ and $ Y = \mathbb{R} $ is the output space.
    Considering that the distributions cannot be observed directly, we can only obtain a second-stage samples $ \{ ( x_{i,j} )_{j=1}^{N} \}_{i=1}^{m} \subseteq \mathcal{X} $ drawn from the distributions $ \{ x_{i} \}_{i=1}^{m} $.
    This framework of DR is quite general and applicable to many important tasks in machine learning, statistics and inverse problems.
    For example, one is given a set of labelled bags in multiple-instance learning \cite{dietterichSolvingMultipleInstance1997,doolyMultipleinstanceLearningRealvalued2002,maronFrameworkMultipleInstanceLearning1997} and each instance in one bag is generated from a corresponding probability distribution in an independent identically distributed manner.
    The task aims at finding the mapping from the bags to the labels. Another example on medical applications is presented in \cite{szaboLearningTheoryDistribution2016} where the input space is a pool of patients identified with a set of probability distributions on $ \mathcal{X} = [0,1] $ and the $ i- $th patient $ x_{i} $ in a sample $ \{ x_{i} \}_{i=1}^{m} $ can be periodically assessed by blood tests $ \{ x_{i,j} \}_{j=1}^{N} $ at moments $ \{ j / N \}_{j=1}^{N} $. $ \{ y_{i} \}_{i=1}^{m} $ are some healthy indicator of the patients which might be inferred from the blood tests. The goal is to learn a mapping from the set of blood tests to the health indicator by observations on a large group of patients. Also, in finance, one may be interested in the relation between the future excess return of the portfolio in a particular sector and the distribution of prices for stocks and bonds therein.

    In the past few years, one popular kernel method, referred to as kernel mean embedding \cite{smolaHilbertSpaceEmbedding2007,muandetKernelMeanEmbedding2017a}, has emerged as a powerful tool for machine learning.
    It maps distributions to a reproducing kernel Hilbert space(RKHS) and learns a functional relationship between these embeddings and outputs.
    Specifically, let $ \{ \mathcal{H}_{k}, \| \cdot \|_{k} \}  $ be a RKHS associated with a reproducing kernel $ k: \mathcal{X} \times \mathcal{X} \to \mathbb{R} $  and $ \{ \mathcal{X}, \mathcal{B}(\mathcal{X}) \}$ is a measurable space, here $ \mathcal{B}(\mathcal{X}) $ is the Borel $ \sigma- $algebra on $ \mathcal{X} $.
    We use $ \mathcal{M}_{1}^{+}(\mathcal{X}) $ to denote the space of probability measures over $ \{ \mathcal{X}, \mathcal{B}(\mathcal{X}) \}$.
    Then the kernel mean embedding of a distribution $ x \in \mathcal{M}_{1}^{+}(\mathcal{X}) $ is defined by
    \begin{equation}
        \mu_{x} = \int_{\mathcal{X}} k(\cdot, s) dx(s) \in \mathcal{H}_{k}.
    \end{equation}
    Through this transformation, the whole arsenal of kernel methods can be extended to probability distributions.
    \cite{muandetKernelMeanEmbedding2017a} shows that if $ \mathbb{E}_{s \sim x} \sqrt{k(s,s)} < \infty $, the embedding $ \mu_{x} $ exists and belongs to $ \mathcal{H}_{k} $.
    Moreover, $ \mathbb{E}_{s \sim x} f(s) = \left< f, \mu_{x} \right>_{\mathcal{H}_{k}} $, through which we can directly compute the expectation of a function $ f $ in the RKHS w.r.t the distribution $ x $ by means of the inner product between the function $ f $ and the mean embedding $ \mu_{x} $ without any restriction on the form of underlying distribution.
    This property can be extended to conditional distribution and has proven useful in certain applications such as graphical model \cite{bootsHilbertSpaceEmbeddings2013,songHilbertSpaceEmbeddings2010,songKernelEmbeddingsLatent2011a} and probability inference \cite{chenCausalDiscoveryReproducing2014} that require an evaluation of expectation with respect to the model.
    For a class of kernel functions known as \emph{characteristic kernels}, the mean embedding captures all information of the distribution.
    In other words, the map $ x \to \mu_{x} $ is injective, implying that $ \| \mu_{x} - \mu_{x^\prime} \|_{\mathcal{H}_{k}} = 0 $ if and only if $ x = x^\prime $, i.e., $ x $ and $ x^\prime $ are the same distribution.
    This makes it suitable for applications that require a unique characterization of distributions such as two-sample homogeneity hypothesis \cite{grettonKernelTwoSampleTest2012}, (conditional) independence tests \cite{fukumizuKernelMeasuresConditional2007,zhangKernelbasedConditionalIndependence2011}, domain adaptation \cite{muandetDomainGeneralizationInvariant2013,zhangDomainAdaptationTarget2013} and differential privacy \cite{balogDifferentiallyPrivateDatabase2018,chatalicCompressiveLearningPrivacy2022}.

    Denote the set of mean embeddings of the distributions by $ X_{\mu} = \{ \mu_{x}, x \in X \} \subset \mathcal{H}_{k} $ which is a separable compact set of continuous functions on $ \mathcal{X} $ \cite{szaboLearningTheoryDistribution2016}.
    In the framework of the least square regression, let $ X_{\mu} $ be the input space and $ \rho $ be a Borel probability measure on $ Z = X_{\mu} \times Y $.
    For a function $ f:X_{\mu} \to Y $ and $ (\mu_{x},y) \in Z $, the prediction error is defined by the least square loss $ ( f(\mu_{x}) - y )^2 $.
    We aim to learn an ideal prediction function which minimizes the expected risk
    \begin{equation}
        \label{eq:regress}
        \min \mathcal{E}(f), \qquad  \mathcal{E}(f) = \int_{Z} ( f(\mu_{x}) - y )^2 d \rho,
    \end{equation}
    over all measurable functions.
    The target is referred to as the regression function
    \begin{equation}
		f_{\rho}(\mu_{x}) = \int_{Y} y d\rho(y|\mu_{x}),\quad \mu_{x} \in X_{\mu}
	\end{equation}
    with $ \rho(\cdot | \mu_{x} ) $ being the conditional probability measure at $ \mu_{x} $ induced by $ \rho $.
    Since $ \rho $ is generally unknown, one may learn $ f_{\rho} $ in a non-parametric setting by implementing some learning algorithms over the first stage sample $ D = \left\{ (\mu_{x_{i}}, y_{i}) \right\}_{i=1}^{m} $ where $ \{\mu_{x_{i}}\}_{i=1}^m $ is the set of the mean embeddings of $ \{x_{i}\}_{i=1}^m $.
    In DR, it follows from the spirit of two-stage sampling that the probability distributions $ \{ x_{i} \}_{i=1}^{m} $ is still unknown and each of those can be approximately estimated by a random sample $ \{ ( x_{i,j} )_{j=1}^{N} \}_{i=1}^{m} $.
    Hence, the goal of DR is to estimate $ f_{\rho} $ based on the sample $ \hat{D} = \left\{( \{ x_{i,j}\}_{j=1}^{N}, y_{i} )\right\}_{i=1}^{m} $, which is obtained in a two-stage sampling process.
    As an extension of the classical kernel ridge regression (KRR) scheme \cite{bauerRegularizationAlgorithmsLearning2007, caponnettoOptimalRatesRegularized2007, cuckerLearningTheoryApproximation2007, smaleLearningTheoryEstimates2007, guo2020modeling}, the KRR method for distribution regression \cite{fangOptimalLearningRates2020,szaboTwostageSampledLearning2015,szaboLearningTheoryDistribution2016} in a RKHS $ (\mathcal{H}_{K}, \| \cdot \|_{K}) $ associated with a Mercer kernel $ K: X_{\mu}\times X_{\mu} \to \mathbb{R} $ is defined as
    \begin{equation}\label{KRR}
        f_{\hat{D}}^{K} = \arg\min_{f \in \mathcal{H}_{K}} \left\{ \frac{1}{m} \sum_{i=1}^{m} \left( f(\mu_{\hat{x}_{i}}) - y_{i} \right)^2 + \lambda \| f \|_{K}^2  \right\},
    \end{equation}
    where $ \hat{x}_{i} = \frac{1}{N} \sum_{i=1}^{N} \delta_{x_{i,j}} $ is the empirical distribution of $ x_{i} $ determined by the sample $ \{ x_{i,j} \}_{j=1}^{N} $ and $ \mu_{\hat{x}_{i}} = \frac{1}{N} \sum_{j=1}^{N} k(\cdot, x_{i,j}) $ is the corresponding mean embedding, $ \lambda > 0  $ is a regularization parameter.
    Due to the well-known representer theorem \cite{scholkopfGeneralizedRepresenterTheorem2001}, the estimator $ f_{\hat{D}}^{K} $ belongs to the sample-dependent hypothesis space
    \[
        \mathcal{H}_{K,\hat{D}} = \left\{ f_\alpha \in \mathcal{H}_{K}  \mid f_\alpha = \sum_{i=1}^{m} \alpha_{i} K(\cdot,\mu_{\hat{x}_{i}}), \alpha_{1}, \cdots ,\alpha_{m} \in \mathbb{R} \right\}
    \]
    and can be uniquely determined as
    \[
        f_{\hat{D}}^{K} (\mu_{x}) = \sum_{i=1}^{m} \alpha_{\hat{D},i}^{K}  K(\mu_{x},\mu_{\hat{x}_{i}}) \quad \text{with} \quad \alpha_{\hat{D}}^{K}  = (\lambda m \mathbb{I}_{m} + \hat{\mathbb{K}}_{m} )^{-1} \mathbf{y},\quad \mu_{x} \in X_{\mu},
    \]
    where $ \mathbb{I}_{m} $ is the identity matrix on $ \mathbb{R}^{m}\times \mathbb{R}^{m}$, $ \mathbf{y} = (y_1,\cdots ,y_{m})^{T} \in \mathbb{R}^{m} $ is a vector composed of the output data of sample $ \hat{D} $ and $ \hat{\mathbb{K}}_{m} = [ K(\mu_{\hat{x}_{i}},\mu_{\hat{x}_{j}} ) ]_{i,j=1}^{m}  $ is the kernel matrix evaluated on $ \hat{D} $.

    Regularization can also be put on the coefficients $\{\alpha_i\}$. Define the regularization term $ \Omega(\alpha) $ by a non-negative function $ \Omega(\cdot) $ on $ \mathbb{R}^{m}.$ Then a coefficient-based regularized algorithm for distribution regression is defined by
    \begin{equation}
		\label{eq:fDhatop}
		f_{\hat{D}}=f_{\alpha_{\hat{D}}} \quad {\rm where} \quad \alpha_{\hat{D}}= \arg \min_{\alpha \in \mathbb{R}^{m}} \left\{ \frac{1}{m} \sum_{i=1}^{m} \left( f_\alpha(\mu_{\hat{x}_{i}}) - y_{i} \right)^2 + \lambda \Omega(\alpha)  \right\}.
	\end{equation}
    If the kernel $ K $ is positive semi-definite then the KRR for DR (\ref{KRR}) can be viewed as a special case of algorithm (\ref{eq:fDhatop}) by taking $ \Omega(\alpha) = \alpha^{T} \hat{\mathbb{K}}_{m} \alpha $. The coefficient-based regularization learning algorithms have some advantages \cite{wuRegularizationNetworksIndefinite2013}. Firstly, the learning algorithm (\ref{eq:fDhatop}) only optimizes the coefficients $ \{ \alpha_i \} $, which is a finite dimensional optimization problem and easy to be adapted to other algorithms. Secondly, we can choose the regularizer for different purpose, for example we can choose $\ell^1$ regularization term for sparse representation \cite{shiLearningTheoryEstimates2013}. Moreover, the kernel $K$ is not required to be positive semi-definite. The classical coefficient regularization regression algorithm with input space in Euclidean space is studied in \cite{shi2011concentration,shi2019distributed,guoOptimalRatesCoefficientbased2019,shiLearningTheoryEstimates2013, sunLeastSquareRegression2011, wuRegularizationNetworksIndefinite2013}.

    In this paper, we focus on the regularization scheme of form (\ref{eq:fDhatop}) with penalty term $ \Omega(\alpha) = m \| \alpha \|_{2}^2 $ and the corresponding algorithm is the $ \ell^2 $ regularization for distribution regression with a regularizaiton parameter $ \lambda m $,
      \begin{equation}
		\label{eq:fDhat}
		f_{\hat{D}}=f_{\alpha_{\hat{D}}} \quad {\rm where} \quad \alpha_{\hat{D}}= \arg \min_{\alpha \in \mathbb{R}^{m}} \left\{ \frac{1}{m} \sum_{i=1}^{m} \left( f_\alpha(\mu_{\hat{x}_{i}}) - y_{i} \right)^2 + \lambda m \| \alpha \|_{2}^2   \right\}.
	\end{equation}
    Note that the estimator  $f_{\hat{D}}$ has the following explicit expression
    \begin{equation}
        \label{eq:fDhat explicit form with matrix}
		f_{\hat{D}}(\mu_{x}) = \sum_{i=1}^{m} \alpha_{\hat{D},i} K(\mu_{x},\mu_{\hat{x}_{i}})\quad \text{with} \quad \alpha_{\hat{D}} = (\lambda m^2 \mathbb{I}_{m} + \hat{\mathbb{K}}_{m}^{T}\hat{\mathbb{K}}_{m})^{-1} \hat{\mathbb{K}}_{m}^{T} \mathbf{y}.
    \end{equation}
    Despite the great development in theory and application, positive semi-definite kernels may be inappropriate in some situations where non-metric pairwise proximity is used \cite{schleifIndefiniteProximityLearning2015}.
    For example, the pairwise proximity may be asymmetric and negative for object comparisons in text documents, graphs and semi-groups. As a result, indefinite kernels are proposed naturally to deal with these problems.
    However, directly applying the indefinite kernels in the original algorithms may lead to complication such as non-convex optimization problems.
    In this paper, the kernel $ K $  is not required to be positive semi-definite.
    In addition, compared with the KRR method for DR, applying the coefficient-based regularized scheme for DR can improve the saturation effect \cite{bauerRegularizationAlgorithmsLearning2007,sunLeastSquareRegression2011}, which means that if the regression function possesses higher regularity, algorithm (\ref{eq:fDhat}) can obtain better convergence performance than the KRR scheme (\ref{KRR}).

    To the best of our knowledge, kernels in approaches and theoretical analysis for DR are required to be positive semi-definite, and there are no theoretical results for DR with indefinite kernels.
    These facts motivate us to fill this gap by investigating asymptotic properties of the proposed algorithm (\ref{eq:fDhat}) with more general kernels.
    In this paper we derive finite sample bounds of the coefficient-based regularized distribution regression problem with least squares loss and further present explicit learning rates. Before giving the main results of our paper, we first review some existing results on DR in the literature, and we will compare with them in Section \ref{section: related work}.
    The first theoretical result for the kernel-based regularized DR was established in \cite{szaboLearningTheoryDistribution2016} via kernel mean embedding, where optimal convergence rates are derived under some mild conditions on the regression function, the positive semi-definite kernel $K$ and the second-stage sample size $N$.
    \cite{fangOptimalLearningRates2020} also considers the KRR method for DR with scalar output and slightly improved the results of \cite{szaboLearningTheoryDistribution2016} on the second-stage sample size by using a novel second order decomposition for inverse of operators.
    Recently, a novel robust DR scheme is investigated in \cite{yuRobustKernelbasedDistribution2021}, which replaces the least square loss with a more general robust function $ l_{\sigma} $ and enriches the analysis of DR.
    Theoretical guarantees for DR using stochastic gradient descent and a two-stage sampling strategy are established in \cite{mueckeStochasticGradientDescent2021} under some mild conditions.
    Coefficient-based regularized distribution regression (\ref{eq:fDhat}) with a Mercer kernel is studied in \cite{dongDistributedLearningDistribution2021}.
    In addition to the kernel mean embedding approach, another line of research employs kernel density estimation(KDE) to address the consistency of regression on distributions under the assumption that the true regressor is H\"{o}lder continuous and the meta distribution has
    finite doubling dimension \cite{olivaFastDistributionReal2014,poczosDistributionFreeDistributionRegression2013}.

    The rest of this paper is organized as follows.
    Our assumptions and main results are detailed in Section \ref{section: main results}.
    Section \ref{section: related work} collects discussions on the comparison of our results and related work in the literature.
    In Section \ref{section: preliminary results and error decomposition}, we give the key analysis and error decomposition for algorithm (\ref{eq:fDhat}).
    The proofs of main results are presented in Section \ref{section: proof of main results}.
    Appendix contains the proofs of some useful lemmas and propositions.

\section{Main Results}\label{section: main results}

    Throughout this paper we assume the sample $ D = \{ ( \mu_{x_{i}}, y_{i})\}_{i=1}^{m} $ is drawn according to $ \rho $ in an i.i.d manner and $ \{ x_{i,j}\}_{j=1}^{N} $ is drawn from the distribution $ x_{i} $ independently for $ i = 1,\cdots ,m $. The main results are based on the following assumptions.

    The first basic assumption presents the boundedness for the output $ y $ and the embedding kernel $ k $.
	\begin{assumption}
		\label{assum:1}
		There exists a constant $M >0 $ such that $ |y| \leqslant M $ almost surely. The Mercer kernel $ k $ is bounded, which asserts $ B_{k} = \sup_{s \in \widetilde{X}} k(s,s) < \infty $ almost surely.
	\end{assumption}
    The boundedness for the kernel $ k $  in Assumption \ref{assum:1} ensures the existence of mean embeddings for all probability distributions on $ \mathcal{X} $ \cite{smolaHilbertSpaceEmbedding2007}.
    The boundedness assumption of output $ y $ is slightly stricter than a moment hypothesis on data distribution which is adopted in \cite{carratinoLearningSGDRandom2018,guoConcentrationEstimatesLearning2013}, and our analysis in this paper can be easily extended to more general situations by assuming moment conditions.

    Let $ \rho_{X_{\mu}} $ be the marginal distribution of $ \rho $ on $ X_{\mu} $ which is  assumed to be non-degenerate throughout this paper.
    The error analysis of the concerned problem is based on the properties of the kernel $ K $ and the marginal distribution $ \rho_{X_{\mu}} $, which are represented by the integral operator $L_K$.
    Let $ L_{\rho_{X_{\mu}}}^{2} $ be the Hilbert space of square-integrable functions on $ X_{\mu} $ with the norm $\|\cdot\|_{\rho_{X_{\mu}}}$ induced by the inner product $\langle f, g\rangle_{\rho_{X_{\mu}}}=\int_{X_{\mu}} f(\mu_{x}) g(\mu_{x}) d \rho_{X_{\mu}}(\mu_{x}) $.
    Define the integral operator $L_{K}: L_{\rho_{X_{\mu}}}^{2}\to L_{\rho_{X_{\mu}}}^{2} $ associated with an arbitrary continuous kernel $ K:X_{\mu}  \times X_{\mu} \to \mathbb{R} $ as
    \begin{equation*}
        L_{K}f (\cdot)=\int_{X_{\mu}} K(\cdot, \mu_{x}) f\left(\mu_{x}\right) d \rho_{X_{\mu}}(\mu_{x}), \quad \forall f \in L_{\rho_{X_{\mu}}}^{2},
    \end{equation*}
    whose adjoint is given by
    \begin{equation*}
        L_{K}^{*}  f (\cdot)=\int_{X_{\mu}} K(\mu_{x},\cdot) f\left(\mu_{x}\right) d \rho_{X_{\mu}}(\mu_{x}), \quad \forall f \in L_{\rho_{X_{\mu}}}^{2}.
    \end{equation*}
    Since the set $ X_{\mu} $ is compact and $ K $ is continuous, $ L_{K} $ and $ L_{K}^{*} $ are both compact operators and thus bounded in $ L_{\rho_{X_{\mu}}}^{2} $.
    When $ K $ is positive semi-definite, $ L_{K} $ is a compact positive operator of trace class on $ L_{\rho_{X_{\mu}}}^{2},$ then for any $ r >0 $, its $ r- $th power $ L_{K}^{r} $ is well-defined according to the spectral theorem.
    The RKHS $ \mathcal{H}_{K} $ induced by a Mercer kernel $ K $ is defined to be the completion of the linear span of $ \{ K(\cdot,\mu_{x}), \mu_{x} \in X_{\mu} \} $ under the inner product $ \left< K(\cdot, \mu_{x}),K(\cdot, \mu_{x^\prime}) \right> = K(\mu_{x}, \mu_{x^\prime}) $ for any $ \mu_{x},\mu_{x^\prime} \in X_{\mu} $.
    By the reproducing property $ f(\mu_{x}) = \left< f, K(\cdot,\mu_{x}) \right>_{K} $, there holds
    \begin{equation}
        \| f \|_{\infty} \leqslant \sup_{\mu_{x} \in X_{\mu}} \sqrt{K(\mu_{x},\mu_{x})} \| f \|_{K},
    \end{equation}
    which asserts that $ \mathcal{H}_{K} $ is embedded into the space of continuous functions on $ X_{\mu} $ with the norm $ \| f \|_{\infty} = \sup_{\mu_{x} \in X_{\mu}} | f(\mu_{x}) | $.
    Moreover, Corollary 4.13 in \cite{cuckerLearningTheoryApproximation2007} states that $ L_{K}^{\frac{1}{2}} $ is an isomorphism from $ L_{\rho_{X_{\mu}}}^2 $ to $ \mathcal{H}_{K} $. Namely, $ \forall f \in L_{\rho_{X_{\mu}}}^2 $, $ L_{K}^{\frac{1}{2}} f \in \mathcal{H}_{K} $ and
    \begin{equation}
        \label{eq:normrelationship}
        \| f \|_{\rho_{X_{\mu}}} = \| L_{K}^{\frac{1}{2}} f \|_{K}.
    \end{equation}
    To establish theoretical analysis for coefficient-based regularized distribution regression (\ref{eq:fDhat}) with indefinite kernels, we state a condition on the kernel $K$ in the following which is adopted in \cite{guoOptimalRatesCoefficientbased2019,maNystromSubsamplingMethod2019,wuRegularizationNetworksIndefinite2013}.
    The  operator $ L_{K} $  admits  the singular value decomposition
    \begin{equation*}
        L_{K} = \sum_{\ell\geqslant 1} \sigma_{\ell} \phi_{\ell} \otimes \psi_{\ell},
    \end{equation*}
    where $ \{ \sigma_{\ell} \}_{\ell \geqslant 1} $ are the singular values of $ L_{K} $ arranged in a descending order, $ \{\phi_{\ell}\}_{\ell\geqslant 1} $ and $ \{\psi_{\ell}\}_{\ell\geqslant 1} $ are two orthonormal systems of $ L_{\rho_{X_{\mu}}}^2 $ satisfying $L_K\psi_\ell=\sigma_\ell\phi_\ell$ and $L_K^*\phi_\ell=\sigma_\ell\psi_\ell$.
    The second assumption is a boundedness condition in terms of $ \{ \sigma_{\ell}, \phi_{\ell}, \psi_{\ell} \}_{\ell \geqslant 1} $.
    \begin{assumption}
        \label{assum:2}
        There exists a constant $ \kappa \geqslant 1 $ such that
        \begin{equation}
            \label{eq:assum2}
            \sup_{\mu_{x} \in X_{\mu}} \sum_{\ell\geqslant 1} \sigma_{\ell} \phi_{\ell}^{2} (\mu_{x}) \leqslant \kappa^{2} \  \rm{and}\ \sup_{\mu_{x} \in X_{\mu}} \sum_{\ell\geqslant 1} \sigma_{\ell} \psi_{\ell}^{2} (\mu_{x}) \leqslant \kappa^{2}
        \end{equation}
    \end{assumption}
    Note that the continuity of the kernel $ K $ ensures the continuity of two orthonormal systems $ \{\phi_{\ell}\}_{\ell\geqslant 1} $ and $ \{\psi_{\ell}\}_{\ell\geqslant 1} $, thus the series in Assumption \ref{assum:2} can be defined point-wisely on $ X_{\mu} $.
    While this hypothesis is difficult to be confirmed without the prior information about $ \rho_{X_{\mu}} $, it still holds for some specific and popular kernels. For instance, this assumption is satisfied when the integral operator $ L_{K} $ is finite rank namely $ \sigma_{\ell} = 0 $ for sufficiently large $ \ell $. Moreover, for positive semi-definite kernels it can be verified  by Mercer's Theorem with $ \kappa = \max \left\{ 1, \sup_{\mu_{x}\in X_{\mu}} \sqrt{K(\mu_{x},\mu_{x})} \right\} $. It is proved in \cite{guoOptimalRatesCoefficientbased2019} that  if the kernel $ K $ admits a Kolmogorov decomposition, the Assumption  \ref{assum:2} is satisfied. And a wide range of representative indefinite kernels such as the linear combination of positive definite kernels \cite{ongLearningKernelHyperkernels2005} and conditionally positive definite kernels \cite{wendlandScatteredDataApproximation2004} admit a Kolmogorov decomposition.

    Under Assumption \ref{assum:2} we can define two positive semi-definite kernels on $ X_{\mu} \times X_{\mu} $
    \begin{equation}
        \label{eq:K0K1}
        K_{0}(\mu_{x},\mu_{x^\prime}) = \sum_{\ell \geqslant 1} \sigma_{\ell} \phi_{\ell}(\mu_{x}) \phi_{\ell}(\mu_{x^\prime}) \  \text{and} \  K_{1}(\mu_{x},\mu_{x^\prime}) = \sum_{\ell \geqslant 1} \sigma_{\ell} \psi_{\ell}(\mu_{x}) \psi_{\ell}(\mu_{x^\prime}),
    \end{equation}
    which induce RKHSs $ \mathcal{H}_{K_{0}} $ and $ \mathcal{H}_{K_{1}} $.
    Note that the integral operators $ L_{K_{0}} $ and $ L_{K_{1}} $ have the same eigenvalues $ \{ \sigma_{\ell} \}_{\ell\geqslant 1} $.

    The next assumption is the H\"{o}lder continuity for the kernel $ K $. We first recall some basic notations in operator theory for further statement. Let $A: {\cal H} \to {\cal H}' $ be a linear operator, where $({\cal H},\langle\cdot,\cdot\rangle_{{\cal H}})$ and $({\cal H}',\langle\cdot,\cdot\rangle_{{\cal H}'})$ are Hilbert spaces with the corresponding norms $\|\cdot\|_{{\cal H}}$ and $\|\cdot\|_{{\cal H}'}$. The set of bounded linear operators from ${\cal H}$ to ${\cal H}'$ is a Banach space with respect to the operator norm $\|A\|_{{\cal H},{\cal H'}}=\sup_{\|f\|_{{\cal H}}=1}\|Af\|_{{\cal H'}}$, which is denoted by $\mathscr{B}({\cal H}, {\cal H'})$ or $\mathscr{B}({\cal H})$ if ${\cal H}={\cal H'}$. When ${\cal H}$ and ${\cal H'}$ are clear from the context, we will omit the subscript and simply denote the operator norm as $\|\cdot\|$.
    \begin{assumption}[H\"{o}lder continuity]
        \label{assum:holder}
        Let $ h_1,h_2 \in (0,1] $ and $ L>0 $, we assume that the mappings $ K_{(\cdot)} : X_{\mu} \to \mathcal{H}_{K_0} $ $( K_{\mu_{x}} = K(\cdot, \mu_{x}) )$ and $ K^{*}_{(\cdot)} $  $( K^{*}_{\mu_{x}} = K(\mu_{x},\cdot) )$ are H\"{o}lder continuous if
        \begin{equation}
            \| K_{\mu_{a}} - K_{\mu_{b}} \|_{K_0} \leqslant L \| \mu_{a} - \mu_{b} \|_{\mathcal{H}_{k}}^{h_1} \quad {\rm and } \quad \| K^{*}_{\mu_{a}} - K^{*}_{\mu_{b}} \|_{K_1} \leqslant L \| \mu_{a} - \mu_{b} \|_{\mathcal{H}_{k}}^{h_{2}},
        \end{equation}
        holds for all $ \mu_{a}, \mu_{b} \in X_{\mu} $.
    \end{assumption}
    When the kernel $ K $ is positive semi-definite, $ K_{(\cdot)} = K_{(\cdot)}^{*} $ and $ h_1 = h_2 $.
    The assumption above for positive semi-definite kernel is adopted in \cite{dongDistributedLearningDistribution2021,fangOptimalLearningRates2020,mueckeStochasticGradientDescent2021,szaboTwostageSampledLearning2015,szaboLearningTheoryDistribution2016,yuRobustKernelbasedDistribution2021} , and some examples of H\"{o}lder continuous kernels such as Gaussian, exponential and Cauchy kernels are presented in \cite{szaboTwostageSampledLearning2015} .
    However, the indefinite kernels are not always symmetric.
    Lemma 4.1 in \cite{guoOptimalRatesCoefficientbased2019} characterizes the fact that there exists a linear isometry $ U \in \mathscr{B}(\mathcal{H}_{K_1},\mathcal{H}_{K_0}) $ such that $ \forall \mu_{x} \in X_{\mu} $, there holds
    \[
        K_{\mu_x}^*=K(\mu_{x}, \cdot)=U^{*} K_{0}(\mu_{x}, \cdot) \in \mathcal{H}_{K_{1}},\quad K_{\mu_x}=K(\cdot, \mu_{x})=U K_{1}(\cdot, \mu_{x}) \in \mathcal{H}_{K_{0}},
    \]
    where $ U^{*} \in \mathscr{B}(\mathcal{H}_{K_0},\mathcal{H}_{K_{1}}) $ is the adjoint operator of $ U $. In light of this fact we obtain the assumption of H\"{o}lder continuity for indefinite kernels.

    The following regularity condition on $ f_{\rho} $ is standard in learning theory, which is also known as a \emph{source condition} in the inverse problems literature.
    \begin{assumption}[regularity condition]
        \label{assum:regular}
        \begin{equation}
            \label{eq:regularity}
            f_{\rho} = L_{K_0}^{r} (g_{\rho}), \quad \rm{ for~ some } ~r>0 ~\rm{ and }~ g_{\rho} \in {\textit L}_{\rho_{X_{\mu}}}^2.
        \end{equation}	
    \end{assumption}
       This assumption states that $ f_{\rho} $ lies in the range space of $ L_{K_0}^{r} $, which is usually used in approximation theory to control the bias of the estimator.
    We recall here that the $ r $th power of $ L_{K_0} $ is well-defined as $ L_{K_0}^{r} = \sum_{\ell\geqslant 1} \sigma_{\ell}^{r} \phi_{\ell} \otimes \phi_{\ell} $ by the spectral theorem.
    As an operator on $ L_{\rho_{X_{\mu}}}^2 $, the range space of $ L_{K_0}^{r} $ can be characterized as
    \[
        L_{K_0}^{r} ( L_{\rho_{X_{\mu}}}^2 ) = \left\{ f \in L_{\rho_{X_{\mu}}}^2: \sum_{\ell \geqslant 1} \frac{ \left< f,\phi_{\ell} \right>^2 }{ \sigma_{\ell}^{2r} } < \infty \right\}.
    \]
    Further, We have the range inclusions $ L_{K_0}^{r} ( L_{\rho_{X_{\mu}}}^2 ) \subseteq L_{K_0}^{r^\prime} ( L_{\rho_{X_{\mu}}}^2 ) $ if $ r > r^\prime $ and $ L_{K_0}^{r} ( L_{\rho_{X_{\mu}}}^2 ) \subseteq \mathcal{H}_{K_0} $ if $ r \geqslant \frac{1}{2} $.
     The above condition tells us that the Fourier coefficient $\left< f,\phi_{\ell} \right>^2$ of $ f_{\rho} $ in terms of $ \{ \phi_{\ell} \}_{\ell \geqslant 1} $ decays faster than the $2r-$th power of the singular value $\sigma_\ell$ of $L_K.$ Intuitively, larger $ r $ implies  higher regularities of $ f_{\rho} $.
    The regularity condition above can be equivalently formulated in terms of some classical concepts in approximation theory such as interpolation spaces \cite{smaleESTIMATINGAPPROXIMATIONERROR2003a}.

    In this paper, the performance of $ f_{\hat{D}} $ is measured by the $ L_{\rho_{X_{\mu}}}^2- $distance between $ f_{\hat{D}} $ and $ f_{\rho} $, i.e., $ \| f_{\hat{D}} - f_{\rho} \|_{\rho_{X_{\mu}}} $.
    We start with non-asymptotic error bounds for algorithm (\ref{eq:fDhat}), which can be utilized to derive explicit convergence rates with suitably chosen $ \lambda $ and $ N $.

    \begin{theorem}\label{theorem: indefinite kernel}
        Let the estimator $ f_{\hat{D}} $ be given by algorithm (\ref{eq:fDhat}) with an indefinite and continuous kernel $K$. Suppose that $ y $ and the kernel $ k $ satisfie the Assumption \ref{assum:1}, the regression function $f_\rho$ satisfies Assumption \ref{assum:regular} with $ r \geqslant \frac{1}{2} $, the  kernel $ K $ satisfies Assumption \ref{assum:2} and \ref{assum:holder}. Let $ h = \min \{ h_1,h_2 \} $ and $ h^\prime = \max \{ h_1,h_2 \} $. Then for any $ 0<\delta<1 $, $ \gamma>0 $, with probability at least $ 1-\delta- e^{-\gamma}  $, $ \left\|f_{\hat{D}}-f_{\rho}\right\|_{\rho_{X_{\mu}}} $ is bounded by
        \begin{equation}
        \begin{split}
            \label{eq:theorem1Bound}
            \begin{cases}
                c_1 ( \log \frac{12}{\delta} )^{2r+4}  ( 1 + \lambda^{-\frac{1}{4}} \mathcal{B}_{m,\lambda} )^{2r+4}  \left[\mathcal{B}_{m,\lambda} + \lambda^{\frac{r}{2}} + \lambda^{-\frac{3}{4}} N^{-\frac{h}{2}}(1+\sqrt{\log m + \gamma} )^{h^\prime} \right] \cdot \\
                \; \left[ 1 + \lambda^{-1} N^{-\frac{h}{2}} (1+\sqrt{\log m +\gamma} )^{h^\prime} \right] \left( 1 + \| g_{\rho} \|_{\rho_{X_{\mu}}} + \lambda^{\frac{r}{2}-\frac{1}{4}}  \right),  \qquad\qquad\qquad \text{if} \ \frac{1}{2} \leqslant r \leqslant \frac{3}{2}, \\
                c_{2} ( \log \frac{12}{\delta} )^{5}  ( 1 + \lambda^{-\frac{1}{4}} \mathcal{B}_{m,\lambda} )^{5} \left[ \mathcal{B}_{m,\lambda} + \lambda^{\frac{1}{4}} m^{-\frac{1}{2}} +\lambda^{\min \left\{1, \frac{r}{2}\right\}} + \lambda^{-\frac{3}{4}} N^{-\frac{h}{2}} (1+\sqrt{\log m + \gamma} )^{h^\prime} \right] \cdot \\
                \; \left[ 1 + \lambda^{-1} N^{-\frac{h}{2}} (1+\sqrt{\log m +\gamma} )^{h^\prime} \right] \left( 2 + \| g_{\rho} \|_{\rho_{X_{\mu}}} +\lambda^{ \min \left\{\frac{3}{4}, \frac{r}{2}-\frac{1}{4} \right\} } \right), \quad\qquad \text{if} \ r > \frac{3}{2},
            \end{cases}
        \end{split}
        \end{equation}
        provided that
        \[
            \kappa^{4} c(t) \log ^{2}(12 m / \delta) m^{-2} \leqslant \lambda \leqslant \kappa^{4} \text { with } 0 < t \leqslant 1/4,
        \]
        where $c_{k}(k=1,2)$ is a constant independent of $m$, $ N $, $ \delta $ or $\gamma$, $ c(t)=4(1+t / 3)^{2} t^{-4},$ and  $\mathcal{B}_{m, \lambda}$ is given by
        \begin{equation}
            \label{eq:Bmlambda}
            \mathcal{B}_{m, \lambda}=\frac{2 \kappa}{\sqrt{m}}\left\{\frac{\kappa}{\sqrt{m} \lambda^{1 / 4}}+\sqrt{\mathcal{N}\left(\lambda^{1 / 2}\right)}\right\}.
        \end{equation}
    \end{theorem}
Here the \emph{effective dimension} $ \mathcal{N}(\lambda) $ measures the complexity of $ \mathcal{H}_{K_{0}} $ and $ \mathcal{H}_{K_{1}} $ with respect to the marginal distribution $ \rho_{X_{\mu}} $, which is defined to be the trace of the operator $ (\lambda I + L_{K_{0}})^{-1} L_{K_{0}} $ and $ (\lambda I + L_{K_{1}})^{-1} L_{K_{1}} $, i.e.,
    \begin{equation}
        \label{eq:effedimen}
        \mathcal{N}(\lambda) = \operatorname{Tr} [ (\lambda I + L_{K_{0}})^{-1} L_{K_{0}} ] = \operatorname{Tr} [ (\lambda I + L_{K_{1}})^{-1} L_{K_{1}} ] = \sum_{\ell \geqslant 1} (\sigma_{\ell} + \lambda)^{-1} \sigma_{\ell}.
    \end{equation}
    To obtain explicit convergence rates, we propose a decaying condition on the singular values of $ L_{K} $ to qualify $ \mathcal{N}(\lambda) $ in the following assumption.

    \begin{assumption}
        \label{assum:capacity}
        There exist constants $ c_{\alpha} > 0 $ and $ \alpha >1 $ such that
        \begin{equation}
            \label{eq:polydecay}
            \sigma_{\ell} \leqslant c_{\alpha} \ell^{-\alpha} , \quad \forall \ell \geqslant 1.
        \end{equation}
    \end{assumption}

    Note that (\ref{eq:assum2}) in Assumption \ref{assum:2} indicates that $ L_{K} $ is an operator of trace class which implies $ \sum_{\ell \geqslant 1}\sigma_{\ell} \leqslant \kappa^2 $ and thus $ \sigma_{\ell} \leqslant \kappa^2  \ell^{-1} $ for all $ \ell \geqslant 1 $.
    This assumption above presents a stronger condition which asserts that the singular values $ \{ \sigma_{\ell} \}_{\ell \geqslant 1} $ of $ L_{K} $ converge even faster to zero.
    A direct calculation in \cite{caponnettoOptimalRatesRegularized2007} shows that
    \begin{equation}
        \label{eq:capacity}
        \mathcal{N}(\lambda) \leqslant\frac{\alpha c _{\alpha}}{\alpha -1} \lambda^{-\frac{1}{\alpha}}.
    \end{equation}
    (\ref{eq:capacity}) is also known as a \emph{capacity condition} or \emph{effective dimension assumption}, which is common in the non-parametric regression setting.
    It is equivalent to the classic entropy or covering number conditions (see \cite{steinwartSupportVectorMachines2008} for more details).
    With suitable $ \lambda $ and $ N $, we obtain the following convergence rates.

    \begin{corollary}\label{corollary1}
        Under the same assumptions of Theorem \ref{theorem: indefinite kernel} and Assumption \ref{assum:capacity}, let $ \lambda = \kappa^{4} m^{-\beta} $ with
        \begin{equation}
            \label{eq:beta}
            \beta =
            \begin{cases}
                \frac{2\alpha}{2 \alpha r + 1}, &\text{if } \frac{1}{2} \leqslant r \leqslant 2,\\
                \frac{2 \alpha}{ 4 \alpha +1 }, &\text{if } r > 2,
            \end{cases}
        \end{equation}
        and $ N = m^{\zeta} \log m $ with
        \begin{equation}
            \label{eq:zeta}
            \zeta=
                \begin{cases}
                    \frac{3 \alpha + 2 \alpha r}{ h \left( 2 \alpha r + 1 \right) }, & \text { if } \frac{1}{2}\leqslant r\leqslant 2, \\
                    \frac{7 \alpha }{h(4 \alpha + 1)}, & \text { if } r>2.
                \end{cases}
        \end{equation}
        Then for any $ 0<\delta<1 $, $ \gamma>0 $ and $ 0 < t \leqslant \frac{1}{4} $, when
        \[
            m \geqslant  \max \{ [4 c(t) \log^2 (12 / \delta)]^{\frac{1}{2-\beta}}, [  2^{12} e^{-4}(2 - \beta)^{-4} c^2(t)  ]^{\frac{1}{2-\beta}} \},
        \]
        with probability at least $ 1-\delta- e^{-\gamma}, $ we have
        \begin{equation}
            \label{eq:coro1bound}
            \left\|f_{\hat{D}}-f_{\rho}\right\|_{\rho_{X_{\mu}}} \leqslant
            \begin{cases}
                \widetilde{c}_{1} ( \log \frac{12}{\delta} )^{2r+4} ( 1 + \sqrt{1 + \gamma} )^{2h^\prime} (\log m)^{h^\prime-h} m^{- \frac{\alpha r}{2 \alpha r +1}},  & \text{if } \frac{1}{2} \leqslant  r \leqslant  \frac{3}{2}, \\
                \widetilde{c}_{2} ( \log \frac{12}{\delta} )^{5} ( 1 + \sqrt{1 + \gamma} )^{2h^\prime} (\log m)^{h^\prime-h} m^{- \frac{\alpha \min\{ r,2 \}}{2 \alpha \min\{ r,2 \} +1}},  & \text{if } r > \frac{3}{2},
            \end{cases}
        \end{equation}
        where $\widetilde{c}_{k}(k=1,2)$ is a constant independent of $m$, $ N $, $ \delta $ or $\gamma$.
    \end{corollary}


    When $ K $ is positive semi-definite, we can relax the restriction on $ m $ and remove the logarithmic term $ \log m $ of the learning rates in Corollary \ref{corollary2}. Recall that for a positive semi-definite kernel, there holds $ K_{(\cdot)} = K_{(\cdot)}^{*} $ and $ h_1 = h_2 = h $ in Assumption \ref{assum:holder}.
    Therefore, in this case we can refine the results in (\ref{eq:coro1bound}) with a learning rate of order $ \mathcal{O}( m^{- \frac{\alpha \min\{ r,2 \}}{2 \alpha \min\{ r,2 \} +1}} ) $ if $r\ge \frac12$, which is shown to be optimal in a minimax sense \cite{caponnettoOptimalRatesRegularized2007}.
    In addition, it can be observed from (\ref{eq:coro1bound}) that a larger $ r $ corresponding to the regression function with higher regularity results in a smaller $ N $ and better learning rates.
    However, the learning rate in (\ref{eq:coro1bound}) stops improving when the regularity index $r$ exceeds $2$, which is the saturation effect mentioned in Section \ref{section: introduction}. Moreover, the well established property of positive semi-definite kernel enables us to obtain the error bounds when the regularity condition (\ref{eq:regularity}) holds with $ 0 < r < \frac{1}{2} $, i.e., $ f_{\rho} \notin \mathcal{H}_{K_{0}} $.

    \begin{theorem}\label{theorem: positive kernel}
    Let the estimator $ f_{\hat{D}} $ be given by algorithm (\ref{eq:fDhat}) with a positive semi-definite kernel $K$. Suppose that $ y $ and the kernel $ k $ satisfie the Assumption \ref{assum:1}, the regression function $f_\rho$ satisfies Assumption \ref{assum:regular} with $ r >0 $, the kernel $ K $ satisfies Assumption \ref{assum:holder} with $h_1=h_2=h$. Then for any $ 0<\delta<1 $, $ \gamma>0 $, with probability at least $ 1-\delta- e^{-\gamma}  $, $ \left\|f_{\hat{D}}-f_{\rho}\right\|_{\rho_{X_{\mu}}} $ is bounded by
        \begin{equation}
            \label{eq:thm2}
            \begin{cases}
                d_{0} \mathcal{A}^{3}_{ m,N,\lambda } ( \log \frac{6}{\delta} )^{3}  \left( \lambda^{\frac{r}{2}} + \lambda^{-\frac{3}{4}} N^{-\frac{h}{2}} \left( 1 + \sqrt{\gamma+\log m} \right)^{h} \right), & \text {if } 0<r <\frac{1}{2}, \\
                d_{1} \mathcal{A}^{2 \max\{1,r\} }_{m,N,\lambda}  ( \log \frac{6}{\delta} )^{\max \{3,2r+1\} } \left[\mathcal{B}_{m, \lambda}+\lambda^{\frac{r}{2}} + \lambda^{-\frac{3}{4}} N^{-\frac{h}{2}} \left( 1 + \sqrt{\gamma+\log m} \right)^{h} \right], & \text {if } \frac{1}{2}\leqslant r \leqslant \frac{3}{2}, \\
                d_{2} \mathcal{A}^{2}_{m,N,\lambda}( \log \frac{8}{\delta} )^{3} \left[ \mathcal{B}_{m, \lambda}+\lambda^{1 / 4} m^{-1 / 2}+\lambda^{\min \left\{1, \frac{r}{2}\right\}} + \lambda^{-\frac{3}{4}} N^{-\frac{h}{2}} \left( 1 + \sqrt{\gamma+\log m} \right)^{h} \right], & \text {if } r>\frac{3}{2},
            \end{cases}
        \end{equation}
        where $ d_{k}= c_{k}^\prime + c _{r} + (12 \kappa^2 + 2) LM (2B_{k})^{\frac{h}{2}}(k=0,1,2) $ is a constant independent of $ m $, $ N $, $ \delta $ or $ \gamma $, $ \mathcal{A}_{ m,N,\lambda } $ is defined by
        \[
            \mathcal{A}_{ m,N,\lambda } = 1 + \kappa L \left( 1 + \sqrt{\gamma+\log m} \right)^{h} \frac{2^{\frac{h+2}{2}} B_{k}^{\frac{h}{2}} }{\lambda^{\frac{1}{2}}  N^{\frac{h}{2}} } + \lambda^{-\frac{1}{4}} \mathcal{B}_{m, \lambda}.
        \]
    \end{theorem}
    \begin{corollary}\label{corollary2}
        Under the same assumptions of Theorem \ref{theorem: positive kernel} and Assumption \ref{assum:capacity}. Let $ \lambda = m^{-\beta} $ with $ \beta $ given by (\ref{eq:beta}) and $ N = m^{\zeta} \log m $ with $ \zeta $ given by (\ref{eq:zeta}) in which the values for $ 0 <r < \frac{1}{2} $ is the same as those in the case of $ \frac{1}{2} \leqslant r \leqslant 2 $, then for any $ 0<\delta<1 $, $ \gamma>0 $, with probability at least $ 1-\delta- e^{-\gamma}  $, there holds
        \begin{equation}
            \left\|f_{\hat{D}}-f_{\rho}\right\|_{\rho_{X_{\mu}}} \leqslant
                    \begin{cases}
                        \widetilde{d}_{0} ( \log \frac{12}{\delta} )^{3} (1+\sqrt{1+\gamma})^{4h} m^{-\frac{\alpha r}{\alpha + 1}}, & \text{ if } 0 < r < \frac{1}{2}, \\
                        \widetilde{d}_{1} ( \log \frac{12}{\delta} )^{\max\{ 3,2r+1 \}  } (1+\sqrt{1+\gamma})^{h \max\{ 3,2r+1 \} } m^{-\frac{ \alpha r}{2 \alpha r+ 1 }}, & \text{ if } \frac{1}{2} \leqslant  r \leqslant \frac{3}{2}, \\
                        \widetilde{d}_{2} ( \log \frac{16}{\delta} )^{3} (1+\sqrt{1+\gamma})^{3h} m^{-\frac{ \alpha \min\{r,2\} }{2 \alpha \min\{r,2\} + 1}}, & \text{ if } r> \frac{3}{2},
                    \end{cases}
        \end{equation}
        provided that $ m \geqslant 3, $ where $ \widetilde{d}_{k}(k=0,1,2) $ is a constant independent of $ m $, $ N $, $ \delta $ or $ \gamma $, which will be given in the proof.
    \end{corollary}

    Our asymptotic convergence rates in Corollary \ref{corollary2} is optimal in the minimax sense when $ \frac12 \le r \leqslant  \frac{3}{2} $, however, when $ 0 <r < \frac{1}{2} $, the convergence rate in Corollary \ref{corollary2} is suboptimal, it is interesting to get better rates in the future. All these results will be proved in Section \ref{section: proof of main results}.

\section{Related work}\label{section: related work}

   In this section,  we compare our error analysis for algorithm (\ref{eq:fDhat}) with results in
the literature where the kernel $K$ is required to be positive semi-definite.

    As mentioned in Section \ref{section: introduction}, the KRR methods for DR (\ref{KRR}) based on the two-stage sampling and kernel mean embedding is investigated in \cite{szaboLearningTheoryDistribution2016}, where the kernel K is required to be positive semi-definite and bounded, then the integral operator $ L_{K} $ is a compact positive
operator of trace class.
    The analysis in \cite{szaboLearningTheoryDistribution2016} is based on the capacity assumption that the eigenvalues $ \sigma_{\ell} $ satisfy $ \sigma_{\ell} \simeq \ell^{-\alpha} $ with $ \alpha >1 $ and the regularity assumption for the regression function (i.e., (\ref{eq:regularity}) with $ 0 < r \leqslant 1 $).
    More specifically, Theorem 5 in \cite{szaboLearningTheoryDistribution2016} shows that when $ \frac{1}{2} < r \leqslant 1 $,
   $
        \| f_{\hat{D}}^{K} - f_{\rho} \|_{\rho_{X_{\mu}}}^2 \leqslant \mathcal{O}(m^{-\frac{2\alpha r}{ 2\alpha r + 1 }} )
    $ by taking $ N \geqslant m^{\frac{\alpha + 2\alpha r}{ h( 2\alpha r + 1 ) }} \log m $.
    The convergence rate is optimal in the minimax sense and matches our results for the case $ \frac{1}{2} \leqslant  r \leqslant \frac{3}{2} $ in Corollary \ref{corollary2} with less restriction on the second stage sample size $N$ than our results.
    When $ 0 < r \leqslant \frac{1}{2} $, Theorem 9  in \cite{szaboLearningTheoryDistribution2016} declares that
    $ \| f_{\hat{D}}^{K} - f_{\rho} \|_{\rho_{X_{\mu}}}^2 \leqslant \mathcal{O}(m^{-\frac{2r}{r+2}}) $ with $ N \geqslant m^{\frac{2(r+1)}{h(r+2)}} \log m $.
    This result is derived without the capacity assumption which corresponds to the capacity independent case, that is, $ \alpha = 1 $.
    To make comparison, the learning rate for $ 0 < r \leqslant \frac{1}{2} $ in Corollary
    \ref{corollary2} with $ \alpha = 1 $  is of the for $ \left\| f_{\hat{D}} - f_{\rho} \right\|_{\rho_{X_{\mu}}}^2 \leqslant \mathcal{O}(m^{-r}) $ which is slightly better than $\mathcal{O}(m^{-\frac{2r}{r+2}})$.
    Moreover, it is worthy noting that our learning rate with $ r = \frac{1}{2} $ in Corollary \ref{corollary2} is optimal in a minimax sense, but the convergence rate is suboptimal in \cite{szaboLearningTheoryDistribution2016}. The suboptimal rate in the case $r =\frac12$ was improved to the optimal one in
    \cite{fangOptimalLearningRates2020} via a novel second order decomposition for operators which is also adopted in our analysis. And a logarithmic term $ \log m $ in the restriction on the second stage sample size $N$ in \cite{szaboLearningTheoryDistribution2016} is removed  in
    \cite{fangOptimalLearningRates2020}.
    When $ 1< r \leqslant 2 $, our learning rate in Corollary \ref{corollary2} still remains minimax optimal while that in \cite{szaboLearningTheoryDistribution2016,fangOptimalLearningRates2020} ceases to improve. Thus one can expect faster convergence rates by applying coefficient-based regularization when the regression function has higher regularities.

    Beyond the framework of the KRR method for DR, the work \cite{yuRobustKernelbasedDistribution2021} takes robustness into consideration and investigates a robust kernel-based DR, where the least square loss is substituted with
    a robust loss function $ l_{\sigma}:\mathbb{R} \to \mathbb{R} $ given by $ l_{\sigma}(u) = \sigma^2 V(\frac{u^2}{\sigma^2}) $ where $ V:\mathbb{R}_{+} \to \mathbb{R} $ is a windowing function and $ \sigma >0 $ is a scaling parameter, and kernel $K$ is required to be positive semi-definite.
    Under some mild conditions on the windowing function $ V $, Corollary 1 in \cite{yuRobustKernelbasedDistribution2021} describes explicit learning rates for the $ L^2- $distance between the estimator $ f_{\hat{D}}^{\sigma} $ and $ f_{\rho} $, by choosing some proper scaling parameter $\sigma$ and the second stage sample size $N$.
    The convergence rates in \cite{yuRobustKernelbasedDistribution2021} match our results in Corollary \ref{corollary2} when $ 0 < r \leqslant 1. $
    However, the learning rate also suffers from the saturation effect, and it stops improving when the regularity index $ r $ is larger than $1.$
    Compared with the robust kernel-based DR schemes, our scheme (\ref{eq:fDhat}) can deal with the indefinite kernel and further promote the saturation level to $ r=2 $, but we need more sample size in the second stage.

    Recently, algorithm (\ref{eq:fDhat}) with positive semi-definite kernel is studied in \cite{dongDistributedLearningDistribution2021}, capacity independent optimal learning rates (i.e., the case $ \alpha = 1 $) are established for $ \frac{1}{2} \leqslant r \leqslant 2$ when the second stage sample size satisfies $ N = m^{\frac{7}{h(2r + 1)}} $ which is  larger than our bound $ N = m^{\frac{3+2r}{h(2r + 1)}}.$
    Besides the less restriction on the second stage sample size and more general regularity condition (\ref{assum:regular}) with $r>0,$ our results can be extended to more general setting with indefinite kernels and further enrich the analysis for DR.

    The saturation effect emerging from KRR can also be overcome by stochastic gradient descent (SGD) scheme.
    Recently, DR associated with positive semi-definite kernel $K$ is studied via SGD with mini-batching in \cite{mueckeStochasticGradientDescent2021}, where the SGD recursion is given by
    \[
        f_{\hat{D},t+1} = f_{\hat{D},t} - \eta \frac{1}{b} \sum_{i = b(t-1)+1}^{bt} ( f_{\hat{D},t}( \mu_{\hat{x}_{j_{i}}} ) - y_{j_{i}}  ) K_{\mu_{\hat{x}_{j_{i}}}},
    \]
    here $ \eta >0 $ is the step size, $b$ is the batch size, and $ j_{1}, \cdots ,j_{bT} $ are identically and independently distributed from the uniform distribution on $\{1,\cdots,m\}$.
    The performance of a tail-averaging estimator $\overline{f}_{\hat{D},T}$ is measured by the excess risk $ \mathbb{E}_{\hat{D}| D} [ \| S_{K} \overline{f}_{\hat{D},T} - f_{\rho} \|_{\rho_{X_{\mu}}} ] $ in \cite{mueckeStochasticGradientDescent2021}, here $\overline{f}_{\hat{D},T}$ is defined as
    \[
        \overline{f}_{\hat{D},T} = \frac{2}{T} \sum_{t=[T / 2]+1}^{T} f_{\hat{D},t},
    \]
    where $ T \in \mathbb{N} $ is the number of iterations.
    When $ r \geqslant \frac{1}{2} $, Corollary 3.5 in \cite{mueckeStochasticGradientDescent2021} shows the rates are minimax optimal $ \mathcal{O}(m^{-\frac{\alpha r}{ 2 \alpha r + 1}} ) $ if $ N \gtrsim m^{\frac{\alpha + 2\alpha r}{ h( 2\alpha r + 1 ) }} \log^{\frac{2}{h}}  m $, $b=1$ and $T=\mathcal{O}(m) $ while the first stage sample size $m$ depends on the confidence level $ \delta $.
    By comparison for the case of $ \frac{1}{2} \leqslant r \leqslant 2 $, our convergence result in Corollary \ref{corollary2} matches that in \cite{mueckeStochasticGradientDescent2021} without the dependence of $ m $ on the confidence level $\delta $, though we need more samples in the second stage sampling in our analysis. When $ 0 < r < \frac{1}{2} $, convergence rates in \cite{mueckeStochasticGradientDescent2021} are derived in two cases: easy problems ($ 2r+\alpha^{-1} >1 $) and hard problems ($ 2r + \alpha^{-1} \leqslant 1 $).
    The former is presented in Corollary 3.6 of \cite{mueckeStochasticGradientDescent2021} which requires $ N \gtrsim m^{\frac{1 + 2\alpha }{ h( 2\alpha r + 1 ) }} \log^{\frac{2}{h}} m $ to achieve minimax optimal rates by multi-pass SGD ($b=\sqrt{m}$ and $T=\mathcal{O}(m^{\frac{1}{2r+\alpha^{-1}}})$) or gradient descent methods ($b=m$ and $T=\mathcal{O}(m^{\frac{1}{2r+\alpha^{-1}}})$), and the latter is presented in Corollary 3.7 where the rate obtained is $ \mathcal{O}(m^{-r} \log^{Cr} m ) $ if $ N \gtrsim \mathcal{O} ( m^{\frac{3-2r}{h}} \log^{-\frac{C(3-2r)}{h}} m ) $ with $T=\mathcal{O} ((m/\log^C m)^{2r+\alpha^{-1}+1})$ and some constant $ C >1 $. Hence, for the case of $ 0 < r < \frac{1}{2} $, our convergence rates are suboptimal and more samples in the second stage are needed in our case.

    At the end of this section, we would like to emphasize that the analysis developed for DR in the literature is based on positive semi-definite kernels.
    To the best of our knowledge, there are no theoretical analysis and explicit learning rates for DR with indefinite kernels in the literature.
    When $ K $ is indefinite, the operators involved in the existing literature are not self-adjoint and positive, which leads to difficulties in our analysis.
    It is worthy noting that the lower bound of the second stage sample size $ N $ in our results is slightly larger than those in existing work whether we apply positive definite kernel or not.
    It would be interesting to establish error bounds in conditional expectation by the analysis developed in \cite{fangOptimalLearningRates2020} to relax the restriction on the second second stage sample size $ N .$
    In addition, our analysis in this paper may be utilized to establish theoretical analysis for other DR schemes with indefinite kernels such as the two-stage SGD in a DR setting with indefinite kernels.

\section{Preliminary Results and Error Decomposition}
\label{section: preliminary results and error decomposition}

    In this section, we present the explicit operator expression for the estimator (\ref{eq:fDhat}) and some preliminary results. Then we will establish the error decomposition of coefficient-based regularization for distribution regression with indefinite kernels and positive semi-definite kernels respectively.

\subsection{Preliminary Results}\label{subsection: preliminary results}

    Suppose that Assumption \ref{assum:2} is satisfied in the rest of this paper, then we can obtain two well-defined positive semi-definite kernels $ K_0 $ and $ K_1 $ denoted by (\ref{eq:K0K1}).
    The corresponding integral operators on $ L_{\rho_{X_{\mu}}}^2 $ can be expressed as
    \[
        L_{K_{0}} = \sum_{\ell \geqslant 1} \sigma_{\ell} \phi_{\ell} \otimes \phi_{\ell}, \quad L_{K_{1}} = \sum_{\ell \geqslant 1} \sigma_{\ell} \psi_{\ell} \otimes \psi_{\ell},
    \]
    where $ \{\phi_{\ell}\}_{\ell\geqslant 1} $ and $ \{\psi_{\ell}\}_{\ell\geqslant 1} $ form two orthonormal systems of $ L_{\rho_{X_{\mu}}}^2 $.
    We can rewrite $ L_{K} $ and $ L_{K}^{*} $ in terms of $ \{\phi_{\ell}\}_{\ell\geqslant 1} $ and $ \{\psi_{\ell}\}_{\ell\geqslant 1} $ as
    \[
    L_{K} = \sum_{\ell \geqslant 1} \sigma_{\ell} \phi_{\ell} \otimes \psi_{\ell}, \quad L_{K}^{*} = \sum_{\ell \geqslant 1} \sigma_{\ell} \psi_{\ell} \otimes \phi_{\ell}.
    \]
    The following lemma from \cite{guoOptimalRatesCoefficientbased2019} characterizes the properties of these operators above, which plays an important role in our analysis.
    \begin{lemma}\label{lem:properties of LK}
        Under Assumption \ref{assum:2}, the following statements hold.
        \begin{enumerate}
            \item $ \left\{ \sqrt{\sigma_{\ell}} \phi_{\ell}: \sigma_{\ell}>0 \right\} $ is an orthonormal basis of $\mathcal{H}_{K_{0}}$ and $\left\{\sqrt{\sigma_{\ell}} \psi_{\ell}: \sigma_{\ell}>0\right\}$ is an orthonormal basis of $\mathcal{H}_{K_{1}}$.
            \item $ L_{K} \in \mathscr{B}\left(\mathcal{H}_{K_{1}}, \mathcal{H}_{K_{0}}\right) $ and $ L_{K}^{*} \in \mathscr{B}\left(\mathcal{H}_{K_{0}}, \mathcal{H}_{K_{1}}\right) $ with $ \left\|L_{K}\right\|=\left\|L_{K}^{*}\right\| \leq \kappa^{2}$.
            \item $ L_{K_{0}} \in \mathscr{B}\left(\mathcal{H}_{K_{0}}\right) $ and $ L_{K_{1}} \in \mathscr{B}\left(\mathcal{H}_{K_{1}}\right) $ with $\left\|L_{K_{0}}\right\|=\left\|L_{K_{1}}^{*}\right\| \leq \kappa^{2}$.
            \item There exists a linear isometry $ U \in \mathscr{B}\left(\mathcal{H}_{K_{1}}, \mathcal{H}_{K_{0}}\right) $ such that $ \phi_{\ell}=U \psi_{\ell} $ for $ \sigma_{\ell} \geq 0 $. For given $ \mu_{x} \in X_{\mu} $, there hold $ K(\cdot, \mu_{x}) = U K_{1}(\cdot, \mu_{x}) \in \mathcal{H}_{K_{0}} $ and $ K(\mu_{x}, \cdot) = U^{*} K_{0}(\mu_{x}, \cdot) \in \mathcal{H}_{K_{1}}$, where $U^{*} \in \mathscr{B}\left(\mathcal{H}_{K_{0}}, \mathcal{H}_{K_{1}}\right)$ is the adjoint operator of $U$.
        \end{enumerate}
    \end{lemma}
    A direct result of this lemma is
    \begin{equation}
        L_{K}=U L_{K_{1}}=L_{K_{0}} U, \quad  L_{K}^{*}=U^{*} L_{K_{0}}=L_{K_{1}} U^{*}.
    \end{equation}
    Let us introduce two sampling operators for the two-stage sampling process.
    Based on the input data $ \mathbf{x}=\left\{\mu_{x_{1}}, \cdots, \mu_{x_{m}} \right\} $, the sampling operator $S_{q}: \mathcal{H}_{K_{q}} \rightarrow \mathbb{R}^{m}$ corresponding to the first-stage sampling is defined as
    \[
        S_{q} f=\left(f\left(\mu_{x_{1}}\right), \cdots, f\left(\mu_{x_{m}}\right)\right), \quad \forall f \in \mathcal{H}_{K_{q}},
    \]
    for the index $ q \in \{ 0, 1\} $.
    Its scaled adjoint operator $S_{q}^{*}: \mathbb{R}^{m} \rightarrow \mathcal{H}_{K_{q}}$ is given by
    \[
        S_{q}^{*} \alpha = \frac{1}{m} \sum_{i=1}^{m} \alpha_{i} K_{q}\left(\mu_{x_{i}}, \cdot\right), \quad \forall \alpha=\left(\alpha_{1}, \cdots, \alpha_{m}\right) \in \mathbb{R}^{m}.
    \]
    The sampling operator $\hat{S}_{q}: \mathcal{H}_{K_{q}} \rightarrow \mathbb{R}^{m}$ corresponding to the second-stage sampling can be defined in a similar way as
    \[
        \hat{S}_{q} f=\left(f\left(\mu_{\hat{x}_{1}}\right), \cdots, f\left(\mu_{\hat{x}_{m}}\right)\right), \quad \forall f \in \mathcal{H}_{K_{q}},
    \]
    associated with the input data $\hat{\mathbf{x}}=\left\{ \mu_{\hat{x}_{1}}, \cdots, \mu_{\hat{x}_{m}} \right\}$.
    And its scaled adjoint operator $\hat{S}_{q}^{*}: \mathbb{R}^{m} \rightarrow \mathcal{H}_{K_{q}}$ is given by
    \[
        \hat{S}_{q}^{*} \alpha = \frac{1}{m} \sum_{i=1}^{m} \alpha_{i} K_{q}\left(\mu_{\hat{x}_{i}}, \cdot\right), \quad \forall \alpha=\left(\alpha_{1}, \cdots, \alpha_{m}\right) \in \mathbb{R}^{m}.
    \]

    We also present some empirical integral operators which can be derived by these sample operators above.
    In the setting of one-stage sampling, define
    \begin{equation}
        \begin{aligned}
            &T_{0}^{\mathbf{x}} = \frac{1}{m} \sum_{i=1}^{m} K_{0}(\mu_{x_{i}},\cdot) \otimes K_{0}(\mu_{x_{i}},\cdot), \quad T_{1}^{\mathbf{x}} = \frac{1}{m} \sum_{i=1}^{m} K_{1}(\mu_{x_{i}},\cdot) \otimes K_{1}(\mu_{x_{i}},\cdot), \\
            & T^{\mathbf{x}} = \frac{1}{m} \sum_{i=1}^{m} K(\cdot,\mu_{x_{i}}) \otimes K_{1}(\mu_{x_{i}}, \cdot), \; \, \quad T_{*}^{\mathbf{x}} = \frac{1}{m} \sum_{n=1}^{m} K(\mu_{x_{i}}, \cdot) \otimes K_{0}(\mu_{x_{i}}, \cdot).
        \end{aligned}
    \end{equation}
    Recall that $ U: \mathcal{H}_{K_{1}} \rightarrow \mathcal{H}_{K_{0}} $ is a linear isometry and $ \mathbb{K}_{m}=\left[K\left(\mu_{x_{i}}, \mu_{x_{j}} \right)\right]_{i, j=1}^{m} $ denotes the kernel matrix evaluated on $\mathbf{x}$.
    Then
    \begin{equation}
        \begin{aligned}
            &T^{\mathbf{x}}=U T_{1}^{\mathbf{x}} = U S_{1}^{*} S_{1}, \quad T_{*}^{\mathbf{x}}=U^{*} T_{0}^{\mathbf{x}}=  U^{*} S_{0}^{*} S_{0},\quad \mathbb{K}_{m} = mS_{0} U S_{1}^{*}. \\
        \end{aligned}
    \end{equation}
    In addition, $ T_{0}^{\mathbf{x}}:\mathcal{H}_{K_0} \to \mathcal{H}_{K_0} $ and $ T_{1}^{\mathbf{x}}:\mathcal{H}_{K_{1}} \to \mathcal{H}_{K_{1}} $ converge to their data-free limits $ L_{K_0} $ and $ L_{K_1} $. The related result found in \cite{linDistributedLearningRegularized2017} characterizes the similarities between them, which is presented as below.
    \begin{lemma}\label{lem: T0X-LK0}
        Let $ D $ be a sample drawn independently according to $\rho$. Then for any $ 0<\delta <1 $, with probability at least $1-\delta$, there hold
        \[
            \begin{aligned}
                \left\|T_{0}^{\mathbf{x}}-L_{K_{0}}\right\| & \leqslant \frac{4 \kappa^{2}}{\sqrt{m}} \log \frac{2}{\delta}, \\
                \left\|T_{1}^{\mathbf{x}}-L_{K_{1}}\right\| & \leqslant \frac{4 \kappa^{2}}{\sqrt{m}} \log \frac{2}{\delta},
            \end{aligned}
        \]
        and
        \[
            \begin{aligned}
                &\left\|\left(\sqrt{\lambda} I+L_{K_{0}}\right)^{-\frac{1}{2}}\left(T_{0}^{\mathbf{x}}-L_{K_{0}}\right)\right\| \leqslant \mathcal{B}_{m, \lambda} \log \frac{2}{\delta}, \\
                &\left\|\left(\sqrt{\lambda} I+L_{K_{1}}\right)^{-\frac{1}{2}}\left(T_{1}^{\mathbf{x}}-L_{K_{1}}\right)\right\| \leqslant \mathcal{B}_{m, \lambda} \log \frac{2}{\delta}.
            \end{aligned}
        \]
    \end{lemma}
    \noindent
    Our analysis is based on the results above.
    Note that $ T^{\mathbf{x}} $ and $ T_{*}^{\mathbf{x}} $ can be also considered as the finite sample estimators of $ L_{K} $ and $ L_{K}^{*} $ while $ T_{*}^{\mathbf{x}} $ is not the adjoint operator of $ T^{\mathbf{x}} $ in general.

    In the setting of two-stage sampling, we can define similar empirical operators
    \begin{equation}
        \begin{aligned}
            &T_{0}^{\hat{\mathbf{x}}} = \frac{1}{m} \sum_{i=1}^{m} K_{0}(\mu_{\hat{x}_{i}},\cdot) \otimes K_{0}(\mu_{\hat{x}_{i}},\cdot), \quad T_{1}^{\hat{\mathbf{x}}} = \frac{1}{m} \sum_{i=1}^{m} K_{1}(\mu_{\hat{x}_{i}},\cdot) \otimes K_{1}(\mu_{\hat{x}_{i}},\cdot), \\
            & T^{\hat{\mathbf{x}}} = \frac{1}{m} \sum_{i=1}^{m} K(\cdot,\mu_{\hat{x}_{i}}) \otimes K_{1}(\mu_{\hat{x}_{i}}, \cdot), \quad T_{*}^{\hat{\mathbf{x}}} = \frac{1}{m} \sum_{n=1}^{m} K(\mu_{\hat{x}_{i}}, \cdot) \otimes K_{0}(\mu_{\hat{x}_{i}}, \cdot),
        \end{aligned}
    \end{equation}
    with the fact that
    \begin{equation}
        \label{eq:factsAboutThat}
            T^{\hat{\mathbf{x}}} = U T_{1}^{\hat{\mathbf{x}}} =  U \hat{S}_{1}^{*} \hat{S}_{1}, \quad T_{*}^{\hat{\mathbf{x}}}=U^{*} T_{0}^{\hat{\mathbf{x}}} =  U \hat{S}_{0}^{*} \hat{S}_{0}, \quad \hat{\mathbb{K}}_{m} = m \hat{S}_{0} U \hat{S}_{1}^{*}.
    \end{equation}
    Considering that  $ T_{*}^{\hat{\mathbf{x}}} $ is also not the adjoint operator of $ T^{\hat{\mathbf{x}}} $ in general, we need to ensure the reversibility of $ \lambda I+T^{\hat{\mathbf{x}}} T_{*}^{\hat{\mathbf{x}}} $ and $ \lambda I+T^{\mathbf{x}} T_{*}^{\mathbf{x}} $.
    \begin{lemma}
        \label{lem:reversibility}
        For $ \lambda >0 $, the operators $ \lambda I+T^{\hat{\mathbf{x}}} T_{*}^{\hat{\mathbf{x}}} $ and $ \lambda I+T^{\mathbf{x}} T_{*}^{\mathbf{x}} $ are invertible on $ \mathcal{H}_{K_0} $,
        \[
            \left\| \left( \lambda I + T^{\hat{\mathbf{x}}} T_{*}^{\hat{\mathbf{x}}} \right)^{-1}  \right\| \leqslant \frac{1}{\lambda} \left( 1 + \frac{\kappa^2}{\sqrt{\lambda} } \right),
        \]
        and
        \[
            \left\| \left( \lambda I+T^{\mathbf{x}} T_{*}^{\mathbf{x}}\right)^{-1}  \right\| \leqslant \frac{1}{\lambda} \left( 1 + \frac{\kappa^2}{\sqrt{\lambda} } \right).
        \]
    \end{lemma}
    The latter in the lemma above is from \cite{guoOptimalRatesCoefficientbased2019} and the former can be proved similarly, which will be left to Appendix.
    With the reversibility of $ \lambda I + T^{\hat{\mathbf{x}}} T_{*}^{\hat{\mathbf{x}}}  $ valid, we can use these notations above to provide an explicit operator expression of the estimator $ f_{\hat{D}} $.
    \begin{proposition}
        \label{prop:fDhatop}
        The estimator $ f_{\hat{D}} $ of coefficient-based regularization for distribution regression (\ref{eq:fDhat}) can be further expressed as
        \begin{equation}
            f_{\hat{D}} = \left(\lambda I+T^{\hat{\mathbf{x}}} T_{*}^{\hat{\mathbf{x}}}\right)^{-1} T^{\hat{\mathbf{x}}} U^{*} \hat{S}_{0}^{*} \mathbf{y},
        \end{equation}
        where $\mathbf{y}=\left(y_{1}, \cdots, y_{m}\right)^{T}$ is a vector in $\mathbb{R}^{m}$ composed of the output data of sample $\hat{D}$.
    \end{proposition}

    We present a direct result of Assumption \ref{assum:holder} in the end, which will be utilized to deal with approximations of integral operators and sample operators.
    Combining the facts of $ K(\mu_{x}, \cdot) = U^{*} K_0(\mu_{x},\cdot) $ and $ K(\cdot, \mu_{x}) = U K_1(\cdot,\mu_{x}) $ in Lemma \ref{lem:properties of LK} with Assumption \ref{assum:holder}, we have $ \forall (\mu_{a},\mu_{b}) \in X_{\mu} \times X_{\mu} $,
    \[
        \| K_1(\cdot, \mu_{a}) - K_1(\cdot, \mu_{b}) \|_{K_1} = \| K(\cdot, \mu_{a}) - K(\cdot, \mu_{b}) \|_{K_0} \leqslant L \| \mu_{a} - \mu_{b} \|_{\mathcal{H}_{k}}^{h_1},
    \]
    \[
        \| K_{0}(\mu_{a}, \cdot) - K_{0}(\mu_{b}, \cdot) \|_{K_0} = \| K(\mu_{a}, \cdot) - K(\mu_{b}, \cdot) \|_{K_1} \leqslant L \| \mu_{a} - \mu_{b} \|_{\mathcal{H}_{k}}^{h_{2}}.
    \]
    Further, we obtain the following lemma with the result in Section A.1.11 in \cite{szaboTwostageSampledLearning2015} which is crucial for our analysis.
    \begin{lemma}
        \label{lem:TxhatMinusTx}
        Assume Assumption \ref{assum:1}, \ref{assum:2} and \ref{assum:holder} are satisfied. Then with probability at least $ 1 - m e^{-\theta} $ there hold
        \begin{equation}
                \left\| T_0^{\hat{\mathbf{x}}} - T_0^{\mathbf{x}}  \right\| \leqslant  \kappa L (1+\sqrt{\theta})^{h_{2}} \frac{2^{\frac{h_{2}+2}{2}} B_{k}^{\frac{h_{2}}{2}} }{N^{\frac{h_2}{2}} } , \quad \left\| T_1^{\hat{\mathbf{x}}} - T_1^{\mathbf{x}}  \right\| \leqslant  \kappa L (1+\sqrt{\theta})^{h_1} \frac{2^{\frac{h_{1}+2}{2}} B_{k}^{\frac{h_1}{2}} }{N^{\frac{h_1}{2}}},
        \end{equation}
        and
        \begin{equation}
            \left\| \hat{S}_{0}^{*} \mathbf{y} -  S_{0}^{*} \mathbf{y} \right\|_{K_0} \leqslant L M \frac{(1+\sqrt{\theta} )^{h_2} (2B_{k})^{\frac{h_2}{2}} }{N^{\frac{h_2}{2}} }.
        \end{equation}
    \end{lemma}

\subsection{Error Decompositions and Key Results}
    We conduct our error analysis by the following error decomposition
    \begin{equation}
        \label{eq:DecomfDhatminusfrho}
        \left\| f_{\hat{D}} - f_{\rho} \right\|_{\rho_{X_{\mu}}} \leqslant \| f_{\hat{D}}- f_{D} \|_{\rho_{X_{\mu}}}  + \left\| f_{D} - f_{\lambda} \right\|_{\rho_{X_{\mu}}} + \left\| f_{\lambda} - f_{\rho} \right\|_{\rho_{X_{\mu}}}.
    \end{equation}
    where $ f_{D} $  is the minimizer of the coefficient-based regularization on the first stage sample $ D = \{ (\mu_{x_{i}}, y_{i} ) \}_{i=1}^{m} $, i.e.,
    \begin{equation}
        f_{D} = \arg \min_{f \in \mathcal{H}_{K,D}} \left\{ \frac{1}{m} \sum_{i=1}^{m} \left( f(\mu_{x_{i}}) - y_{i} \right)^2 + \lambda m \| \alpha \|_{2}^2 \right\},
    \end{equation}
    with
    \[
        \mathcal{H}_{K,D} = \left\{ f = \sum_{i=1}^{m} \alpha_{i} K(\cdot,\mu_{x_{i}}), \alpha_{1}, \cdots ,\alpha_{m} \in \mathbb{R} \right\}.
    \]
    Similar to Proposition \ref{prop:fDhatop}, $f_{D}$ has the following explicit form
    \begin{equation}
        \label{eq:fDRaw}
        f_{D}=\left(\lambda I+T^{\mathbf{x}} T_{*}^{\mathbf{x}}\right)^{-1} T^{\mathbf{x}} U^{*} S_{0}^{*} \mathbf{y}.
    \end{equation}
    It plays an important role as an intermediate function in our error analysis.
    Another intermediate function $ f_{\lambda} $ is the data-free minimizer of regularized least squares regression with a positive semi-definite kernel $ \widetilde{K}(\mu_{x},\mu_{x^\prime}) = \mathbb{E}_{\mu_{y}} [K( \mu_{x}, \mu_{y} ) K(\mu_{x^\prime},\mu_{y})] $, i.e.,
    \begin{equation}\label{eqn: flambda}
        f_{\lambda} = \arg\min_{ f \in \mathcal{H}_{\widetilde{K}}} \left\{ \| f - f_{\rho} \|_{\rho}^2 + \lambda \| f \|_{\widetilde{K}}^2 \right\}.
    \end{equation}
    Then we can obtain the operator representation of $ f_{\lambda} $ in \cite{smaleLearningTheoryEstimates2007} as
    \begin{equation}
        \label{eq:flambda}
        \begin{aligned}
            f_{\lambda} = \left( \lambda I + L_{\widetilde{K}} \right)^{-1} L_{\widetilde{K}} f_{\rho} &= \left( \lambda I + L_{K}L_{K}^{*} \right)^{-1} L_{K} L_{K}^{*} f_{\rho} \\
            &= \left( \lambda I + L_{K_{0}}^{2} \right)^{-1} L_{K_{0}}^{2} f_{\rho}.
        \end{aligned}
    \end{equation}
    where it follows from $ L_{\widetilde{K}} = L_{K} L_{K}^{*} = L_{K_{0}} U U^{*} L_{K_{0}} = L_{K_{0}}^{2} $.
    Note that $ f_{\lambda} $ is a population version of $ f_{D} $. From this perspective, the concerned algorithm in this paper can be viewed as a variant of kernel-based regularized distribution regression, which applies a positive semi-definite kernel $ \widetilde{K}_{D}(\mu_{x},\mu_{x^\prime}) = \frac{1}{m} \sum_{i=1}^{m} K(\mu_{x},\mu_{x_{i}}) K(\mu_{x^\prime},\mu_{x_{j}}) $ to provide an effective way to improve the saturation effect.

    The third term of (\ref{eq:DecomfDhatminusfrho}) is independent of data and it can be easily estimated from the regularity condition (\ref{eq:regularity}) and (\ref{eq:flambda}).
    \begin{lemma}
        \label{lem:approxerror}
        Under Assumption \ref{assum:regular}, there holds
        \begin{align}
            &\| f_{\lambda}-f_{\rho} \|_{\rho_{X_{\mu}}} \leqslant  c _{r} \lambda^{\min \{ 1,\frac{r}{2} \}}, \quad \, \text{ with }  r >0, \\
            &\| f_{\lambda}-f_{\rho} \|_{K_0} \; \, \leqslant  c _{r}^\prime \lambda^{\min\{ 1,\frac{2r-1}{4} \} }, \text{ with }  r \geqslant \frac{1}{2},
        \end{align}
        where $ c _{r} = \max \{ 1, \kappa^{2r-4} \} \| g_{\rho} \|_{\rho_{X_{\mu}}}  $ and $ c_{r}^\prime = \max \{ 1, \kappa^{2r-5} \} \| g_{\rho} \|_{\rho_{X_{\mu}}} $.
    \end{lemma}
    The proof is similar to that of Theorem 4 in \cite{smaleLearningTheoryEstimates2007}.

    Now we focus on the second term $ \| f_{D} - f_{\lambda} \|_{\rho_{X_{\mu}}} $ and the first term   $ \| f_{\hat{D}}- f_{D} \|_{\rho_{X_{\mu}}} $ of (\ref{eq:DecomfDhatminusfrho}).  The second term $ \| f_{D} - f_{\lambda} \|_{\rho_{X_{\mu}}} $ represents the error induced by one-stage coefficient-based regularization scheme and it can be bounded by utilizing the second order decomposition of inverse operator differences \cite{guoLearningTheoryDistributed2017,linDistributedLearningRegularized2017}. The first term $ \| f_{\hat{D}}- f_{D} \|_{\rho_{X_{\mu}}} $ reflects the error incurred by the second-stage sampling and it can be bounded by the H\"{o}lder continuity of the kernel $ K $, which leads to the dependence of $ N $ on $ h $.

    From the expressions of $ f_{D} $ and $ f_{\lambda} $ given by (\ref{eq:fDRaw}) and (\ref{eq:flambda}), we need to handle the difference between $ \left(\lambda I+T^{\mathbf{x}} T_{*}^{\mathbf{x}}\right)^{-1} $ and $ \left( \lambda I + L_{K}L_{K}^{*} \right)^{-1} $ to bound $ \left\| f_{D} - f_{\lambda} \right\|_{\rho_{X_{\mu}}} $. Let  $A$ and $B$ be invertible operators on a Banach space, then from the second order decomposition of operator difference developed in \cite{linDistributedLearningRegularized2017}, we have
    \begin{equation}
        \label{Second order}
        \begin{aligned}
            A^{-1}-B^{-1} &= B^{-1}(B-A)A^{-1}(B-A)B^{-1}+B^{-1}(B-A)B^{-1} \\
            &= B^{-1}(B-A)B^{-1}(B-A)A^{-1}+B^{-1}(B-A)B^{-1}.
        \end{aligned}
    \end{equation}
    Then if we take $A=\lambda I+T^{\mathbf{x}} T_{*}^{\mathbf{x}}$ and $B=  \lambda I + L_{K}L_{K}^{*} $ in (\ref{Second order}), we need to estimate the difference
    \[
        L_{K}L_{K}^{*} - T^{\mathbf{x}} T_{*}^{\mathbf{x}} = (L_{K} - T^{\mathbf{x}} ) L_{K}^{*} + L_{K}(L_{K}^{*} - T_{*}^{\mathbf{x}} ) - (L_{K} - T^{\mathbf{x}} )(L_{K}^{*} - T_{*}^{\mathbf{x}} ).
    \]
    Since the empirical operators and their data-free limits cannot commute with each other, we may not obtain satisfactory bounds with this product term appearing in our analysis.
    In view of this, we define $ \hat{T}_{0}^{\mathbf{x}} = T^{\mathbf{x}} U^{*} = \frac{1}{m} \sum_{i=1}^{m} K(\cdot,\mu_{x_{i}}) \otimes K(\cdot,\mu_{x_{i}})$, which is also positive semi-definite on $ \mathcal{H}_{K_{0}} $ and can approximate $ L_{K_{0}} $ well.
    Thus, $ f_{D} $ can be expressed in terms of $ \hat{T}_{0}^{\mathbf{x}} $ as
    \begin{equation}
        \label{eq:fDoperator}
        f_{D}=\left(\lambda I+  \hat{T}_{0}^{\mathbf{x}} T_{0}^{\mathbf{x}} \right)^{-1} \hat{T}_{0}^{\mathbf{x}} S_{0}^{*} \mathbf{y},
    \end{equation}

     We have the following decomposition to bound $ \| f_{D}-f_{\lambda} \|_{\rho_{X_{\mu}}} $ in spirit of that in \cite{guoOptimalRatesCoefficientbased2019}.
    \begin{lemma}
        \label{lem:fDMinusflambda}
        If Assumption \ref{assum:regular} is satisfied with some $ r>0 $ and $ g_{\rho} \in L_{\rho_{X_{\mu}}}^{2} $, let
        \begin{align*}
          r_1 = \max \left\{ 0, r-\frac{1}{2} \right\}, \quad r_2 = \min \left\{ 1, \frac{r}{2}+\frac{3}{4} \right\},
        \end{align*}
         then
        \begin{equation}
            \| f_{D}-f_{\lambda} \|_{\rho_{X_{\mu}}} \leqslant \mathcal{S}_{1}(D,\lambda) \mathcal{S}_{2}^{2} (D,\lambda) \mathcal{S}_{3}(D,\lambda) + \mathcal{S}_{2}(D,\lambda) \mathcal{S}_{4}(D,\lambda) \| g_{\rho} \|_{\rho_{X_{\mu}}} \lambda^{r_{2}},
        \end{equation}
        where
        \[
            \begin{aligned}
            & \mathcal{S}_{1} (D,\lambda) = \left\| (\sqrt{\lambda}I + L_{K_{0}} )^{-\frac{1}{2}}  \left( S_{0}^{*} \mathbf{y} - T_{0}^{\mathbf{x}} f_{\lambda} \right) \right\|_{K_{0}},	\\
                & \mathcal{S}_{2} (D,\lambda) = \left\| (\sqrt{\lambda} I + L_{K_{0}})^{\frac{1}{2}} (\sqrt{\lambda} I + \hat{T}_{0}^{\mathbf{x}} )^{-\frac{1}{2}} \right\|,	\\
                & \mathcal{S}_{3} (D, \lambda) = \left\| (\sqrt{\lambda} I + \hat{T}_{0}^{\mathbf{x}} )^{\frac{1}{2}} \left(\lambda I+  \hat{T}_{0}^{\mathbf{x}} T_{0}^{\mathbf{x}} \right)^{-1}  \hat{T}_{0}^{\mathbf{x}} (\sqrt{\lambda} I + \hat{T}_{0}^{\mathbf{x}} )^{\frac{1}{2}}  \right\|,	\\
                & \mathcal{S}_{4}(D,\lambda) = \left\| (\sqrt{\lambda} I + \hat{T}_{0}^{\mathbf{x}} )^{\frac{1}{2}}  \left( \lambda I + \hat{T}_{0}^{\mathbf{x}} T_{0}^{\mathbf{x}} \right) ^{-1} L_{K_{0}}^{r_{1}}  \right\|.
            \end{aligned}
        \]
           \end{lemma}
    The decomposition above is a variant of Proposition 5.1 in \cite{guoOptimalRatesCoefficientbased2019} in which the only difference is the input space.
    It allows us to estimate the error term $ \| f_{D}-f_{\lambda} \|_{\rho_{X_{\mu}}} $ by bounding $ \mathcal{S}_{i}(D,\lambda) $  $ (i = 1,2,3,4) $.
    Estimating $ \mathcal{S}_{3}(D,\lambda) $ and $ \mathcal{S}_{4}(D,\lambda) $ is more complicated and we will discuss them respectively when $ K $ is positive semi-definite and indefinite.
    If $ K $ is positive semi-definite, $ U $ is identity operator and $ \hat{T}^{\mathbf{x}}_{0} = T^{\mathbf{x}}_{0} $. Define
    \begin{equation}
        \label{eq:36}
        \begin{aligned}
            \mathcal{S}^{\prime}_{3}(D,\lambda) &= \left\| (\lambda I + \hat{T}^{\mathbf{x}}_{0}\hat{T}^{\mathbf{x}}_{0} )^{-1}\hat{T}^{\mathbf{x}}_{0}(\sqrt{\lambda} I + \hat{T}^{\mathbf{x}}_{0} ) \right\|,	\\
            \mathcal{S}_{4}^{^\prime}(D,\lambda) &= \left\| (\sqrt{\lambda} I + \hat{T}^{\mathbf{x}}_{0} )^{\frac{1}{2}} (\lambda I + \hat{T}^{\mathbf{x}}_{0}\hat{T}^{\mathbf{x}}_{0} )^{-1} L_{K_{0}}^{r_{1}}   \right\|.
        \end{aligned}
    \end{equation}
    Then $ \mathcal{S}^{\prime}_{3}(D,\lambda) $ is bounded by
    \begin{equation}
        \label{eq:S3'bound}
        \mathcal{S}^{\prime}_{3}(D,\lambda) \leqslant \left\| (\lambda I + \hat{T}^{\mathbf{x}}_{0}\hat{T}^{\mathbf{x}}_{0} )^{-1} \sqrt{\lambda} \hat{T}^{\mathbf{x}}_{0} \right\| + \left\| (\lambda I + \hat{T}^{\mathbf{x}}_{0}\hat{T}^{\mathbf{x}}_{0} )^{-1} \hat{T}^{\mathbf{x}}_{0} \hat{T}^{\mathbf{x}}_{0}\right\| \leqslant 2.
    \end{equation}
    When the kernel is indefinite, we can bound $ \mathcal{S}_{3}(D,\lambda) $ and $ \mathcal{S}_{4}(D,\lambda) $ as
    \begin{equation}
        \label{eq:S3S4}
        \begin{aligned}
        \mathcal{S}_{3}(D,\lambda) & \leqslant \mathcal{S}_{5}(D,\lambda) \mathcal{S}_{3}^{^\prime}(D,\lambda),	\\
        \mathcal{S}_{4}(D,\lambda) & \leqslant \mathcal{S}_{5}(D,\lambda) \mathcal{S}_{4}^{^\prime}(D,\lambda),
        \end{aligned}
    \end{equation}
    where $ \mathcal{S}_{3}^\prime(D,\lambda) $ and $ \mathcal{S}_{4}^\prime(D,\lambda) $ are defined by (\ref{eq:36}) and
    \begin{equation}
        \label{eq:S5}
        \mathcal{S}_{5}(D,\lambda) = \left\| (\sqrt{\lambda} I + \hat{T}_{0}^{\mathbf{x}}  )^{\frac{1}{2}} (\lambda I + \hat{T}_{0}^{\mathbf{x}} T_{0}^{\mathbf{x}})^{-1} (\lambda I + \hat{T}_{0}^{\mathbf{x}} \hat{T}_{0}^{\mathbf{x}} ) (\sqrt{\lambda} I + \hat{T}_{0}^{\mathbf{x}}  )^{-\frac{1}{2}}    \right\|.
    \end{equation}
    Hence, we can derive the decomposition for $ \| f_{D}-f_{\lambda} \|_{\rho_{X_{\mu}}} $ with an indefinite kernel $ K $.

    \begin{proposition}[indefinite kernel case]\label{prop:approximation error of indefinite}
        If Assumption \ref{assum:regular} is satisfied with some $ r\geqslant \frac{1}{2} $ and $ g_{\rho} \in L_{\rho_{X_{\mu}}}^{2} $, then
        \[
            \| f_{D}-f_{\lambda} \|_{\rho_{X_{\mu}}} \leqslant \mathcal{S}_{5}(D,\lambda) \Xi_1(D,\lambda),
        \]
        where $ \Xi_1(D,\lambda) = 2 \mathcal{S}_{1}(D,\lambda) \mathcal{S}_{2}^{2} (D,\lambda) + \mathcal{S}_{2}(D,\lambda) \mathcal{S}_{4}^\prime(D,\lambda) \| g_{\rho} \|_{\rho_{X_{\mu}}} \lambda $.
    \end{proposition}

    For the term $ \| f_{\hat{D}}- f_{D} \|_{\rho_{X_{\mu}}} $ in (\ref{eq:DecomfDhatminusfrho}), we rewrite $ f_{\hat{D}} $ as
        \begin{equation}
            \label{eq:fDhatOperRepre}
            f_{\hat{D}} = \left( \lambda I + \hat{T}_{0}^{\hat{\mathbf{x}}} T_{0}^{\hat{\mathbf{x}}}  \right)^{-1} \hat{T}_{0}^{\hat{\mathbf{x}}} \hat{S}_{0}^{*}\mathbf{y},
        \end{equation}
    where $ \hat{T}_{0}^{\hat{\mathbf{x}}} = T^{\hat{\mathbf{x}}} U^{*} = \frac{1}{m} \sum_{i=1}^{m} K(\cdot, \mu_{\hat{x}_{i}}) \otimes K(\cdot, \mu_{\hat{x}_{i}}) $.
    Then we can combine the representations of $ f_{\hat{D}} $ and $ f_{D} $ in (\ref{eq:fDoperator}), (\ref{eq:fDhatOperRepre}) with some estimates to obtain the following bound.

    \begin{proposition}[indefinite kernel case]\label{prop:difference of fD and fDhat indefinite}
        \label{prop:indefinitesecond}
        Suppose that Assumption \ref{assum:1} and \ref{assum:2}  are satisfied, then we have
        \begin{equation}
		    \label{eq:fDhatminusfDbound}
            \left\| f_{\hat{D}} - f_{D} \right\|_{\rho_{X_{\mu}}} \leqslant 4 \kappa^{3} (M+\kappa) \lambda^{-\frac{3}{4}} \mathcal{S}_{2}(D,\lambda) \mathcal{S}_{5}(D,\lambda) \mathcal{S}_{6}(\hat{D},\lambda)  \mathcal{S}_{7}(\hat{D},\lambda),
        \end{equation}
        where $ \mathcal{S}_{6}(\hat{D},\lambda) = \lambda^{-1} \left\| T_{0}^{\mathbf{x}} - T_{0}^{\hat{\mathbf{x}}} \right\| + \lambda^{-1}  \left\| \hat{T}_{0}^{\mathbf{x}}  - \hat{T}_{0}^{\hat{\mathbf{x}}} \right\| + 1 $ and
        \[
            \mathcal{S}_{7}(\hat{D},\lambda) =  \left(  \left\| \hat{S}_{0}^{*} \mathbf{y} -  S_0^{*} \mathbf{y} \right\|_{K_0} +  \|  \hat{T}_{0}^{\hat{\mathbf{x}}} - \hat{T}_{0}^{\mathbf{x}} \| \right) +  \| f_{D} \|_{K_0} \left( \left\| T_{0}^{\mathbf{x}} - T_{0}^{\hat{\mathbf{x}}} \right\| +  \left\| \hat{T}_{0}^{\mathbf{x}} - \hat{T}_{0}^{\hat{\mathbf{x}}} \right\| \right).
        \]
    \end{proposition}
    We can immediately obtain a general bound for the total error $ \| f_{\hat{D}} - f_{\rho} \|_{\rho_{X_{\mu}}} $ by combining Proposition \ref{prop:approximation error of indefinite},  \ref{prop:difference of fD and fDhat indefinite} with Lemma \ref{lem:approxerror}, which is essential for the derivation of our learning rates.
    \begin{proposition}[indefinite kernel case]
        \label{prop: general error of indefinite}
        Suppose that Assumption \ref{assum:1}  and \ref{assum:2}  are satisfied and Assumption \ref{assum:regular} holds with some $ r\geqslant \frac{1}{2} $ and $ g_{\rho} \in L_{\rho_{X_{\mu}}}^{2} $, then we have
        \begin{equation*}
            \left\| f_{\hat{D}} - f_{\rho} \right\|_{\rho_{X_{\mu}}} \leqslant 4 \kappa^{3} (M+\kappa) \mathcal{S}_{2}(D,\lambda) \mathcal{S}_{5}(D,\lambda) \mathcal{S}_{6}(\hat{D},\lambda)  \mathcal{S}_{7}(\hat{D},\lambda) + \mathcal{S}_{5}(D,\lambda) \Xi_1(D,\lambda) + c _{r} \lambda^{\min \{ 1,\frac{r}{2} \}}.
        \end{equation*}
    \end{proposition}

    When $ K  $ is positive semi-definite, we can make a different decomposition for the term $ \| f_{\hat{D}} - f_{D} \|_{\rho_{X_{\mu}}} $ as follows to refine the result.
    \begin{proposition}[positive semi-definite kernel case]\label{prop:difference of fD and fDhat positive}
        \label{prop:positivesecond}
        Suppose that Assumption \ref{assum:1}  and \ref{assum:2}  are satisfied, then we have
        \begin{equation}
            \left\| f_{\hat{D}} - f_{D} \right\|_{\rho_{X_{\mu}}} \leqslant 2 \lambda^{-\frac{1}{4}} S_2(\hat{D},\lambda) \left\|  \hat{S}_{0}^{*} \mathbf{y} -  S_{0}^{*} \mathbf{y}  \right\|_{K_0} + 6 \kappa M \lambda^{-\frac{3}{4}}  S_2(\hat{D},\lambda) \left\| \hat{T}_{0}^{\hat{\mathbf{x}}} - \hat{T}_{0}^{\mathbf{x}}  \right\|,
        \end{equation}
        where
        \begin{equation*}
            S_2(\hat{D},\lambda) = \left\| \left( \sqrt{\lambda} I + L_{K_0} \right)^{\frac{1}{2}} \left( \sqrt{\lambda}I+ \hat{T}_{0}^{\hat{\mathbf{x}}} \right)^{-\frac{1}{2}} \right\|.
        \end{equation*}
    \end{proposition}
    Further, we can also bound $  \| f_{\hat{D}} - f_{\rho} \|_{\rho_{X_{\mu}}} $ with positive semi-definite kernels by combining Proposition \ref{prop:difference of fD and fDhat positive}, Lemma \ref{lem:fDMinusflambda} and Lemma \ref{lem:approxerror}.
    \begin{proposition}[positive semi-definite kernel case]\label{prop:general error of positive}
        Suppose that Assumption \ref{assum:1}  and \ref{assum:2}  are satisfied and Assumption \ref{assum:regular} holds with some $ r > 0 $ and $ g_{\rho} \in L_{\rho_{X_{\mu}}}^{2} $, then we have
        \begin{equation}
            \begin{aligned}
                \left\| f_{\hat{D}} - f_{\rho} \right\|_{\rho_{X_{\mu}}} \leqslant 2 &\lambda^{-\frac{1}{4}} S_2(\hat{D},\lambda) \left\|   \hat{S}_{0}^{*} \mathbf{y} -  S_{0}^{*} \mathbf{y}  \right\|_{K_0} + 6 \kappa M \lambda^{-\frac{3}{4}}  S_2(\hat{D},\lambda) \left\| \hat{T}_{0}^{\hat{\mathbf{x}}} - \hat{T}_{0}^{\mathbf{x}}  \right\| \\
                &+ 2 \mathcal{S}_{1}(D,\lambda) \mathcal{S}_{2}^{2} (D,\lambda) + \mathcal{S}_{2}(D,\lambda) \mathcal{S}_{4}^\prime(D,\lambda) \| g_{\rho} \|_{\rho_{X_{\mu}}} \lambda^{r_2} + c _{r} \lambda^{\min \{ 1,\frac{r}{2} \}}.
            \end{aligned}
        \end{equation}
    \end{proposition}

\section{Proof of Main Results}\label{section: proof of main results}
    We prove our finite sample bounds and learning rates of algorithm (\ref{eq:fDhat}) in this section.
    In the following, we need estimate those terms in Proposition \ref{prop: general error of indefinite} and \ref{prop:general error of positive} to derive Theorem \ref{theorem: indefinite kernel} and \ref{theorem: positive kernel} for different regularity indexes $ r $.
    Afterwards, we present the proof of Corollary \ref{corollary1} and \ref{corollary2} which present explicit learning rates with proper $ \lambda $ and $ N $.
    To this end, first we need the bounds for $ \mathcal{S}_{i}(D,\lambda)(i=1,2,4,5) $ which are crucial in our error analysis and can be found in \cite{guoOptimalRatesCoefficientbased2019}.
    \begin{lemma}
        \label{lem:S1}
        Suppose that Assumption \ref{assum:1} and \ref{assum:2} are satisfied. Then for any $ 0<\delta<1 , $ with probability at least $ 1-\delta, $ there holds
        \begin{equation}
            \label{eq:S1}
            \mathcal{S}_{1}(D,\lambda) \leqslant c _{\lambda} \mathcal{B}_{m,\lambda} \log \frac{4}{\delta} + \| f_{\rho} - f_{\lambda} \|_{\rho_{X_{\mu}}},
        \end{equation}
        where $c _{\lambda} = \left(  M + \| f_{\lambda} \|_{\infty} \right) \kappa ^{-1}$.
    \end{lemma}

    Probabilistic bound of $ \mathcal{S}_{1}(D,\lambda) $ in \cite{guoOptimalRatesCoefficientbased2019} is based on the moment assumption concerning the output $ y $. For the sake of completeness, we give the proof in Appendix for bounded $ y $, which leads to a different $ c _{\lambda} $.

    \begin{lemma}
        \label{lem:S2}
        For any $0 < \delta < 1$, with probability at least $ 1-\delta $, there holds
        \begin{equation}
            \label{eq:S2}
            \mathcal{S}_{2}(D,\lambda) \leq 1 + \lambda^{-\frac{1}{4}} \mathcal{B}_{m,\lambda} \log \frac{2}{\delta}.
        \end{equation}
    \end{lemma}

    \begin{lemma}\label{lem:S4'}
        Under Assumption \ref{assum:holder} with $ r>0 $, if $ 0<r<\frac{1}{2} $, then
        \[
            \mathcal{S}^{\prime}_{4}(D,\lambda) \leqslant 2 \lambda^{-\frac{3}{4}}.
        \]
        If $ r\geqslant \frac{1}{2} $, then
        \[
            \mathcal{S}^{\prime}_{4}(D,\lambda) \leqslant
            \begin{cases}
                2^{r+1} \left\| (\sqrt{\lambda} I + \hat{T}^{\mathbf{x}}_{0}  )^{-1} (\sqrt{\lambda} I + L_{K_{0}} ) \right\|^{r-\frac{1}{2}} \lambda^{\frac{r}{2}-1}, & \text{if } \frac{1}{2} \leqslant  r \leqslant  \frac{3}{2},	\\
                (2r-1) \kappa^{2r-3} \left[ \| L_{K_{1}} - T_{1}^{\mathbf{x}}  \| \lambda^{-\frac{3}{4}} + \lambda^{\min \{0,\frac{r}{2}- 1 \}}   \right], & \text{if } r>\frac{3}{2}.
            \end{cases}
        \]
    \end{lemma}

    \begin{lemma}
        \label{lem:S5}
        Let $\delta \in(0,1) $ and
        \begin{equation}
            \label{eq:Theta_1}
            \Theta_{1}(D,\lambda) =  \left\| \left( \lambda I + \hat{T}_{0}^{\mathbf{x}} \hat{T}_{0}^{\mathbf{x}}  \right)^{-\frac{3}{4}} \hat{T}_{0}^{\mathbf{x}}  (\hat{T}_{0}^{\mathbf{x}} - T_{0}^{\mathbf{x}} ) \left( \lambda I + \hat{T}_{0}^{\mathbf{x}} \hat{T}_{0}^{\mathbf{x}}  \right)^{-\frac{1}{4}} \right\|.
        \end{equation}
        If $ m \geqslant 2 $ and
        \[
            c(t) \kappa^{4}\left(\frac{\log \frac{2 m}{\delta}}{m}\right)^{2} \leqslant \lambda \leqslant \kappa^{4} \text { with } 0 < t \leqslant \frac{1}{4}
        \]
        where $c(t)=4(1+t / 3)^{2} t^{-4}$, then $\Theta_{1}(D, \lambda) \leqslant 2 \sqrt{2} t(1-t)^{-1}<1$ holds with probability at least $1-\delta$. Further,
        \begin{equation}
            \label{eq:S5bound}
            \mathcal{S}_{5} (D,\lambda) \leqslant 1 + 2 \lambda^{-\frac{1}{4}} \mathcal{S}_{2} (D,\lambda) \Theta_{2} (D,\lambda) + 2 \left[ 1 - \Theta_{1}(D,\lambda) \right]^{-1} \left[ \lambda^{-\frac{1}{4}} \mathcal{S}_{2}(D,\lambda) \Theta_{2}(D,\lambda)  \right]^2,
        \end{equation}
        where
        \[
            \Theta_{2}(D,\lambda) = \left\| \left( \sqrt{\lambda}I + L_{K_{1}} \right)^{-\frac{1}{2}} (T_{1}^{\mathbf{x}} - L_{K_{1}} ) \right\| + \left\| \left( \sqrt{\lambda}I + L_{K_{0}} \right)^{-\frac{1}{2}} (T_{0}^{\mathbf{x}} - L_{K_{0}} ) \right\|.
        \]
    \end{lemma}

    $ \mathcal{S}_{2}(D,\lambda) $ and $ \mathcal{S}_{5}(D,\lambda) $ are estimated through the second order decomposition on the difference of operator inverses.
    All of these proofs of Lemma \ref{lem:S2}-\ref{lem:S5} can be found in \cite{guoOptimalRatesCoefficientbased2019}.
    We now turn to the proof of Theorem \ref{theorem: indefinite kernel} based on the results above.

\subsection{Proof of Theorem \ref{theorem: indefinite kernel} }

    In this case we assume Assumption \ref{assum:regular} holds with some $ r \geqslant \frac{1}{2} $ to guarantee the uniform bound of $ \| f_{D} \|_{K_0} $.
    Based on (\ref{eq:flambda}), there holds
    \[
        \begin{aligned}
            \| f_{\lambda} \|_{\infty} & \leqslant \sup_{\mu_{x} \in X_{\mu}} \sqrt{K_0(\mu_{x},\mu_{x})} \| f_{\lambda} \|_{K_{0}} = \kappa \| (\lambda I + L_{K_{0}}^2)^{-1} L_{K_{0}}^{2+r} g_{\rho} \|_{K_{0}} \\
            & = \kappa\left\|\left(\lambda I+L_{K_{0}}^{2}\right)^{-1} L_{K_{0}}^{\frac{3}{2}+r} L_{K_{0}}^{\frac{1}{2}} g_{\rho}\right\|_{K_{0}} \leqslant \kappa^{ 2r}\left\|g_{\rho}\right\|_{\rho_{X_{\mu}}},
        \end{aligned}
    \]
    where the last inequality follows from
    \begin{equation}
        \label{eq:NormflambdaLargeR}
        \left\|\left(\lambda I+L_{K_{0}}^{2}\right)^{-1} L_{K_{0}}^{\frac{3}{2}+r}\right\| \leq\left\|\left(\lambda I+L_{K_{0}}^{2}\right)^{-1} L_{K_{0}}^{2}\right\|\left\|L_{K_{0}}^{r-\frac{1}{2}}\right\| \leqslant \kappa^{2r-1}.
    \end{equation}
    We substitute these estimates into bound (\ref{eq:S1}),
    then for any $ 0<\delta<1 $, there exists a subset $D_{\delta, 1}^{m}$ of $D^{m}$ with measure at most $\delta$ such that
    \[
        \mathcal{S}_{1}(D, \lambda) \leqslant c_{\rho} \kappa^{2r-1} \mathcal{B}_{m, \lambda} \log \frac{4}{\delta} +c_{r} \lambda^{\min \left\{1, \frac{r}{2}\right\}}, \quad \forall \mathbf{z} \in Z^{m} \backslash Z_{\delta, 1}^{m},
    \]
    where $ c_{\rho} = M + \| g_{\rho} \|_{\rho_{X_{\mu}}} $.

    Lemma \ref{lem:S2} implies that there exists another subset $ D_{\delta, 2}^{^\prime m} $ of $ D^{m} $ with measure at most $ \frac{\delta}{2} $ such that
    \[
        S_{2}^{2} (D,\lambda) \leqslant \left\| \left( \sqrt{\lambda} I + L_{K_{0}} \right) \left( \sqrt{\lambda} I + \hat{T}_{0}^{\mathbf{x}}   \right)^{-1}  \right\| \leqslant \left[ 1 + \lambda^{-\frac{1}{4}} \mathcal{B}_{m,\lambda} \log \frac{4}{\delta} \right]^2.
    \]
    In the proof of Lemma \ref{lem:S5}, there exists a subset $ D_{\delta,2}^{m}  $ with measure at most $\frac{\delta}{2}$ such that for $ D \in D^{m} \backslash  D_{\delta, 2}^{ m}$,
    \[
        \mathcal{S}_{2}(D, \lambda)=\left\|\left(\sqrt{\lambda} I+L_{K_{0}}\right)^{\frac{1}{2}}\left(\sqrt{\lambda} I+\hat{T}_{0}^{\mathrm{x}}\right)^{-\frac{1}{2}}\right\| \leq(1-t)^{-\frac{1}{2}}
    \]
    provided that $c(t) \kappa^{4}\left(\frac{\log \frac{4 m}{\delta}}{m}\right)^{2} \leqslant \lambda \leqslant \kappa^{4}$ with $0<t \leqslant \frac{1}{4} $.

    By Lemma \ref{lem: T0X-LK0} there exists a third subset $ D_{\delta,3}^{m}  $ of $ D^{m} $ with measure at most $ \frac{\delta}{2} $ such that
    \[
        \left\|T_{1}^{\mathbf{x}}-L_{K_{1}}\right\|  \leqslant \frac{4 \kappa^{2}}{\sqrt{m}} \log \frac{4}{\delta},\quad \forall  D \in D^{m} \backslash D_{\delta,3}^{m} .
    \]
    Further, $ \mathcal{S}_{4}^\prime(D,\lambda) $ can be bounded by Lemma \ref{lem:S4'}. Specifically,
    \[
        \mathcal{S}_{4}^\prime(D, \lambda) \leqslant
        \begin{cases}
            2^{r+1}\left[1+\lambda^{-\frac{1}{4}} \mathcal{B}_{m, \lambda} \log \frac{4}{\delta} \right]^{2 r-1} \lambda^{\frac{r}{2}-1}, & \text { if } \frac{1}{2} \leqslant r \leqslant \frac{3}{2} \text { and } D \in D^{m} \backslash D_{\delta, 2}^{^\prime m}, \\
            (8 r-4) \kappa^{2 r-1}\left[m^{-\frac{1}{2}} \lambda^{-\frac{3}{4}} \log \frac{4}{\delta}+\lambda^{\min \left\{0, \frac{r}{2}-1\right\}}\right], & \text { if } r > \frac{3}{2} \text { and } D \in D^{m} \backslash D_{\delta, 3}^{m}.
        \end{cases}
    \]

    Next we need to estimate $ \Theta_2(D,\lambda) $ in Lemma \ref{lem:S5} to bound $ \mathcal{S}_{5}(D,\lambda) $.
    Lemma \ref{lem: T0X-LK0} guarantees the existence of two subsets $ D_{\delta,4}^{m}  $ and $ D_{\delta,5}^{m} $ of $ D^{m}  $ with each measure at most $ \frac{\delta}{2} $ such that
    \[
        \begin{aligned}
            &\left\|\left(\sqrt{\lambda} I+L_{K_{0}}\right)^{-\frac{1}{2}}\left(T_{0}^{\mathbf{x}}-L_{K_{0}}\right)\right\| \leqslant \mathcal{B}_{m, \lambda} \log \frac{4}{\delta} , \quad \forall D \in D^{m} \backslash D_{\delta,4}^{m}, \\
            &\left\|\left(\sqrt{\lambda} I+L_{K_{1}}\right)^{-\frac{1}{2}}\left(T_{1}^{\mathbf{x}}-L_{K_{1}}\right)\right\| \leqslant \mathcal{B}_{m, \lambda} \log \frac{4}{\delta} , \quad \forall D \in D^{m} \backslash D_{\delta,5}^{m}.
        \end{aligned}
    \]
    which implies
    \[
        \Theta_2(D,\lambda) \leqslant 2 \mathcal{B}_{m, \lambda} \log \frac{4}{\delta}, \quad \forall D \in D^{m} \backslash (D_{\delta,4}^{m} \cup D_{\delta,5}^{m}).
    \]
    Moreover, Lemma \ref{lem:S5} ensures that when $c(t) \kappa^{4}\left(\frac{\log \frac{4 m}{\delta}}{m}\right)^{2} \leqslant \lambda \leqslant \kappa^{4}$ with $0<t \leqslant \frac{1}{4}$,
    \[
        \Theta_{1}(D, \lambda) \leqslant 2 \sqrt{2} t(1-t)^{-1}<1, \quad \forall D \in D^{m} \backslash Z_{\delta, 2}^{ m}.
    \]
    Substituting these bounds of $\Theta_{1}(D, \lambda), \Theta_{2}(D, \lambda)$ and $\mathcal{S}_{2}(D, \lambda)$ into (\ref{eq:S5bound}), we obtain that with the choice of $\lambda$,
    \[
        \mathcal{S}_{5}(D, \lambda) \leqslant c^{\prime}(t)( \log \frac{4}{\delta} )^{2} \left( 1+\lambda^{-\frac{1}{4}} \mathcal{B}_{m, \lambda}\right)^{2}, \quad \forall D \in D^{m} \backslash\left(D_{\delta, 2}^{m} \cup D_{\delta, 4}^{m} \cup D_{\delta, 5}^{m}\right)
    \]
    where $c^{\prime}(t)=1+4(1-t)^{-\frac{1}{2}}+8[1-(2 \sqrt{2}+1) t]^{-1}$.
    In addition, we can bound $\Xi_1(D,\lambda)  $ in Proposition \ref{prop:approximation error of indefinite} by combining the above estimates together.
    When $ \frac{1}{2} \leqslant r \leqslant \frac{3}{2} $, there holds
    \[
        \begin{aligned}
            \Xi_1(D,\lambda)& \leqslant 2  \left( c_{\rho} \kappa^{2r-1} \mathcal{B}_{m, \lambda} \log \frac{4}{\delta}+c_{r} \lambda^{\frac{r}{2}} \right) \left( 1-t \right)^{-1} + 2^{r+1}\cdot   \\
            & \qquad \left(1+\lambda^{-\frac{1}{4}} \mathcal{B}_{m, \lambda} \log \frac{4}{\delta}\right)^{2 r-1} \left( 1-t \right)^{-\frac{1}{2}}\left\|g_{\rho}\right\|_{\rho_{X}} \lambda^{\frac{r}{2}} \\
            &\leqslant  c_{1}^\prime (1-t)^{-1}  \left( 1+\lambda^{-\frac{1}{4}} \mathcal{B}_{m, \lambda} \right)^{2 r-1} ( \log \frac{4}{\delta} )^{2r} ( \mathcal{B}_{m,\lambda} + \lambda^{\frac{r}{2}}  ),
        \end{aligned}
    \]
    for $D \in D^{m} \backslash\left( D_{\delta, 1}^{m} \cup D_{\delta, 2}^{m} \cup D_{\delta, 2}^{^\prime m} \right)$ where $ c_{1}^{\prime}=2 c_{\rho} \kappa^{2r-1}+2^{r+1} \left\| g_{\rho} \right\|_{\rho_{X_{\mu}}} + 2 c_{r} $.

    \noindent When $r>\frac{3}{2}$, there holds
    \[
        \begin{aligned}
            \Xi_1(D,\lambda) &\leqslant 2 \left( c_{\rho} \kappa^{2r-1} \mathcal{B}_{m, \lambda} \log \frac{4}{\delta}+c_{r} \lambda^{\min \left\{1, \frac{r}{2}\right\}} \right) \left( 1-t \right)^{-1} +(8 r-4) \kappa^{2 r-1} \cdot \\
            & \qquad \left( \lambda^{\frac{1}{4}} m^{-\frac{1}{2}}  \log \frac{4}{\delta} +\lambda^{\min \left\{1, \frac{r}{2}\right\}} \right) \left( 1-t \right)^{-\frac{1}{2}}  \left\|g_{\rho}\right\|_{\rho_{X}} \\
            & \leqslant c_{2}^{\prime}\left( 1-t \right)^{-1}\log \frac{4}{\delta} \left( \mathcal{B}_{m, \lambda} +\lambda^{\frac{1}{4}} m^{-\frac{1}{2}} +\lambda^{\min \left\{1, \frac{r}{2}\right\}} \right),
        \end{aligned}
    \]
    for $ D \in D^{m} \backslash\left(D_{\delta, 1}^{m} \cup D_{\delta, 2}^{m} \cup D_{\delta, 3}^{m}\right)$ where $ c_{2}^{\prime}=2 c_{\rho} \kappa^{2r-1}+(8 r-4) \kappa^{2 r-1}\left\|g_{\rho}\right\|_{\rho_{X}}+ 2 c_{r} $.

    To derive the explicit result of $ \left\| f_{\hat{D}} - f_{D} \right\|_{\rho_{X_{\mu}}} $ in Proposition \ref{prop:indefinitesecond}, we need to bound $ \| f_{D} \|_{K_0} $ first.
    \begin{proposition}
        \label{prop:boundK0}
        Suppose that Assumption \ref{assum:2} is satisfied with $ \kappa \geqslant 1 $ and Assumption \ref{assum:regular} holds with some $ r \geqslant \frac{1}{2} $, then we have
        \begin{equation}
            \label{eq:fDK0bound}
                \| f_{D} \|_{K_0} \leqslant  \mathcal{S}_{5}(D,\lambda) \Xi_{2}(D,\lambda)+ c_{r}^\prime \lambda^{\min \{ 1,\frac{2r-1}{4} \} }+ \kappa^{2r-1} \| g_{\rho} \|_{\rho_{X_{\mu}}},
        \end{equation}
        where $ \Xi_{2}(D,\lambda) = 2 \lambda^{-\frac{1}{4}}  \mathcal{S}_{1}(D,\lambda) \mathcal{S}_{2}(D,\lambda) + \lambda^{\frac{3}{4}} \left\| g_{\rho} \right\|_{\rho_{X_{\mu}}} \mathcal{S}^\prime_{4}(D,\lambda) $.
    \end{proposition}
    \noindent
    We can bound the term $ \Xi_{2}(D,\lambda) $ in Proposition \ref{prop:boundK0} in a similar way.
    When $ \frac{1}{2} \leqslant r \leqslant \frac{3}{2} $, for $D \in D^{m} \backslash\left( D_{\delta, 1}^{m} \cup D_{\delta, 2}^{m} \cup D_{\delta, 2}^{^\prime m} \right)$, there holds
    \begin{equation*}
        \Xi_{2}(D,\lambda) \leqslant c_{1}^\prime (1-t)^{-\frac{1}{2}} \left( 1+\lambda^{-\frac{1}{4}} \mathcal{B}_{m, \lambda} \right)^{2 r-1} \left(  \log \frac{4}{\delta}  \right) ^{2r} \left(  \lambda^{-\frac{1}{4}}  \mathcal{B}_{m,\lambda} + \lambda^{\frac{r}{2}-\frac{1}{4}}  \right).
    \end{equation*}

    \noindent When $r>\frac{3}{2}$, for $D \in Z^{m} \backslash\left(D_{\delta, 1}^{m} \cup D_{\delta, 2}^{m} \cup D_{\delta, 3}^{m}\right)$, there holds
    \begin{equation*}
        \Xi_{2}(D,\lambda) \leqslant c_{2}^{\prime} ( 1-t )^{-\frac{1}{2}}\log \frac{4}{\delta} \left( \lambda^{-\frac{1}{4}} \mathcal{B}_{m, \lambda} + m^{-\frac{1}{2}} +\lambda^{\min \left\{\frac{3}{4}, \frac{r}{2}-\frac{1}{4} \right\}} \right) .
    \end{equation*}

    Recall that $ h = \min \{ h_1,h_2 \} $ and $ h^\prime = \max \{ h_1,h_2 \} $ in Theorem \ref{theorem: indefinite kernel}.
    Then we can derive a probabilistic bound of $ \left\| f_{\hat{D}} - f_{D} \right\|_{\rho_{X_{\mu}}} $ by substituting the bounds of $ \mathcal{S}_{2}(D,\lambda) $, $ \mathcal{S}_{5}(D,\lambda) $, $ \| f_{D} \|_{K_0} $ and lemma \ref{lem:TxhatMinusTx} into Proposition \ref{prop:indefinitesecond}.
    When $ \frac{1}{2} \leqslant r \leqslant \frac{3}{2} $, there holds with probability at least $ 1-3\delta - m e^{-\theta} $
    \[
        \begin{aligned}
            \left\| f_{\hat{D}} - f_{D} \right\|_{\rho_{X_{\mu}}} \leqslant  c_1^\prime& (t)( \log \frac{4}{\delta} )^{2r+4} ( 1 + \lambda^{-\frac{1}{4}} \mathcal{B}_{m,\lambda} )^{2r+4} \lambda^{-\frac{3}{4}} N^{-\frac{h}{2}}  (1+\sqrt{\theta} )^{h^\prime} \cdot \\
            &\left[ 1 + \lambda^{-1} N^{-\frac{h}{2}} (1+\sqrt{\theta} )^{h^\prime} \right] \left( 1 + \| g_{\rho} \|_{\rho_{X_{\mu}}} + \lambda^{\frac{r}{2}-\frac{1}{4}}  \right),
        \end{aligned}
    \]
    with $ c_1^\prime(t) = 2^{h^\prime+4} \kappa^{2r+3} L (\kappa + M)^2 (c_{1}^{\prime} + c _{r}^\prime) ( B_{k}^{\frac{h_1}{2}} + B_{k}^{\frac{h_{2}}{2}} ) [ 1 +  L  ( B_{k}^{\frac{h_1}{2}} + B_{k}^{\frac{h_{2}}{2}} ) ] [c^{ \prime}(t)]^2 (1-t)^{-1}  $.

    \noindent When $r>\frac{3}{2}$, there holds with probability at least $ 1-3\delta - m e^{-\theta}  $
    \[
        \begin{aligned}
            \left\| f_{\hat{D}} - f_{D} \right\|_{\rho_{X_{\mu}}} \leqslant  c_{2}^\prime&(t) ( \log \frac{4}{\delta} )^{5} ( 1 + \lambda^{-\frac{1}{4}} \mathcal{B}_{m,\lambda} )^{5} \lambda^{-\frac{3}{4}} N^{-\frac{h}{2}}  (1+\sqrt{\theta} )^{h^\prime} \cdot \\
            & \left[ 1 + \lambda^{-1} N^{-\frac{h}{2}} (1+\sqrt{\theta} )^{h^\prime} \right] \left( 2 + \| g_{\rho} \|_{\rho_{X_{\mu}}} +\lambda^{ \min \left\{\frac{3}{4}, \frac{r}{2}-\frac{1}{4} \right\} } \right).
        \end{aligned}
    \]
    with $ c_{2}^\prime(t) = 2^{h^\prime+4} \kappa^{2r+3} L (\kappa + M)^2 (c_{2}^{\prime} + c _{r}^\prime) ( B_{k}^{\frac{h_1}{2}} + B_{k}^{\frac{h_{2}}{2}} ) [ 1 +  L  ( B_{k}^{\frac{h_1}{2}} + B_{k}^{\frac{h_{2}}{2}} ) ] [c^{ \prime}(t)]^2 (1-t)^{-1} $.

    Finally we can combine the bounds of $ \left\| f_{\hat{D}} - f_{D} \right\|_{\rho_{X_{\mu}}} $, $ \Xi_1(D,\lambda) $ and $ \mathcal{S}_{5}(D,\lambda) $ together to verify the results in Theorem \ref{theorem: indefinite kernel} by scaling $ 3\delta $ to $ \delta $ and $ me^{-\theta} $ to $ e^{-\gamma} $.
    When $ \frac{1}{2} \leqslant r \leqslant \frac{3}{2} $, there holds with probability at least $ 1 - \delta - e^{-\gamma} $
    \[
        \begin{aligned}
            \| f_{\hat{D}}-f_{\rho} \|_{\rho_{X_{\mu}}} \leqslant c_1 & ( \log \frac{12}{\delta} )^{2r+4}  ( 1 + \lambda^{-\frac{1}{4}} \mathcal{B}_{m,\lambda} )^{2r+4}  \left[\mathcal{B}_{m,\lambda} + \lambda^{\frac{r}{2}} + \lambda^{-\frac{3}{4}} N^{-\frac{h}{2}}(1+\sqrt{\log m + \gamma} )^{h^\prime
            } \right] \\
            & \left[ 1 + \lambda^{-1} N^{-\frac{h}{2}} (1+\sqrt{\log m +\gamma} )^{h^\prime} \right] \left( 1 + \| g_{\rho} \|_{\rho_{X_{\mu}}} + \lambda^{\frac{r}{2}-\frac{1}{4}}  \right),
        \end{aligned}
    \]
    where $ c_1 = c^{\prime}(t)(1-t)^{-1} c_{1}^{\prime} + c _{r} + c_1^\prime(t) $.

    \noindent When $ r >\frac{3}{2} $, there holds with probability at least $ 1 - \delta - e^{-\gamma} $
    \[
        \begin{aligned}
            \| f_{\hat{D}}-f_{\rho} \|_{\rho_{X_{\mu}}} \leqslant c_{2}&  ( \log \frac{12}{\delta} )^{5}  ( 1 + \lambda^{-\frac{1}{4}} \mathcal{B}_{m,\lambda} )^{5} \cdot \\
            & \left[ \mathcal{B}_{m,\lambda} + \lambda^{\frac{1}{4}} m^{-\frac{1}{2}} +\lambda^{\min \left\{1, \frac{r}{2}\right\}} + \lambda^{-\frac{3}{4}} N^{-\frac{h}{2}} (1+\sqrt{\log m + \gamma} )^{h^\prime} \right] \cdot \\
            & \left[ 1 + \lambda^{-1} N^{-\frac{h}{2}} (1+\sqrt{\log m +\gamma} )^{h^\prime} \right] \left( 2 + \| g_{\rho} \|_{\rho_{X_{\mu}}} +\lambda^{ \min \left\{\frac{3}{4}, \frac{r}{2}-\frac{1}{4} \right\} } \right),
        \end{aligned}
    \]
    where $ c_{2} = c^{\prime}(t)(1-t)^{-1} c_{2}^{\prime} +c _{r} + c_2^\prime(t) $.

\subsection{Proof of Corollary \ref{corollary1}}

    Recall the values of $ \lambda $ in (\ref{eq:beta}), first we need choose $ m $ sufficiently large such that
    \[
        \kappa^{4} c(t) \log ^{2}(12 m / \delta)  m^{-2} \leqslant \lambda = \kappa^{4}  m^{-\beta}\leqslant \kappa^{4},
    \]
    where the right hand side above holds clearly and the left hand side is equivalent to
    \[
        m^{2-\beta} \geqslant c(t) \log^2 (12m/ \delta).
    \]
    It is sufficient to guarantee that
    \begin{equation*}
        \left\{
        \begin{aligned}
            \frac{1}{2} m^{2-\beta} & \geqslant 2 c(t) \log^2 (12/ \delta),\\
            \frac{1}{2} m^{2-\beta} & \geqslant 2 c(t)  \log^2 m \Leftrightarrow \frac{1}{2} m^{\frac{2-\beta}{2}} \geqslant 2c(t) m^{-\frac{2-\beta}{2}} \log^2 m,
        \end{aligned}
        \right.
    \end{equation*}
    hold simultaneously.
    With simply derivative calculation, we have $ m^{-\frac{2-\beta}{2}} \left( \log m \right)^2 \leqslant 16 e^{-2} (2-\beta)^{-2} $. Hence we only choose $m$ such that
    \[
        m \geqslant  \max \left\{ \left[4 c(t) \left( \log 12 / \delta \right)^2 \right]^{\frac{1}{2-\beta}}, \left[  2^{12} e^{-4}(2 - \beta)^{-4} c^2(t)  \right]^{\frac{1}{2-\beta}} \right\}.
    \]

    Recall that the polynomial decaying condition (\ref{eq:polydecay}) in Assumption \ref{assum:capacity} implies $ \mathcal{N}(\lambda) \leqslant \frac{\alpha c_{\alpha}}{\alpha-1} \lambda^{-\frac{1}{\alpha}} $ and $ \lambda = \kappa^{4} m^{- \frac{2\alpha}{2 \alpha r +1}}  $ for $ \frac{1}{2}\leqslant r\leqslant 2 $.
    Then we have
    \[
        \mathcal{B}_{m,\lambda} = \frac{2 \kappa}{\sqrt{m}}\left\{\frac{\kappa}{\sqrt{m} \lambda^{\frac{1}{4}}}+\sqrt{\mathcal{N} (\lambda^{\frac{1}{2}} )}\right\} \leqslant 2 \left( \kappa + \sqrt{c_{\alpha}^\prime} \kappa^{1+\frac{2}{\alpha}} \right) m^{-\frac{\alpha r}{2 \alpha r +1}} ,
    \]
    and
    \[
        1 + \lambda^{-\frac{1}{4}} \mathcal{B}_{m,\lambda} \leqslant 1 + 2 \left( 1 + \sqrt{c_{\alpha}^\prime} \kappa^{\frac{2}{\alpha}} \right).
    \]
    For $ r>2, $ we have
    \[
        \mathcal{B}_{m,\lambda} \leqslant 2 \left( \kappa + \sqrt{c_{\alpha}^\prime} \kappa^{1+\frac{2}{\alpha}} \right) m^{-\frac{2 \alpha}{4 \alpha  +1}} \text{ and } 1 + \lambda^{-\frac{1}{4}} \mathcal{B}_{m,\lambda} \leqslant 1 + 2 \left( 1 + \sqrt{c_{\alpha}^\prime} \kappa^{\frac{2}{\alpha}} \right).
    \]
    With $ N $ taken values as (\ref{eq:zeta}), there holds
    \[
        \begin{aligned}
            \lambda^{-\frac{3}{4}} N^{-\frac{h}{2} } (1+\sqrt{\log m + \gamma} )^{h^\prime} &\leqslant  \kappa^{-3} (\log m)^{\frac{h^\prime-h}{2}}  m^{- \frac{\alpha \min \{ r,2 \} }{2 \alpha \min \{ r,2 \} +1} } \left( \frac{1}{\sqrt{\log m}} + \sqrt{\frac{\gamma+\log m}{\log m}} \right)^{h^\prime} \\
            & \leqslant (\log m)^{\frac{h^\prime-h}{2}} \left( 1 + \sqrt{1+\gamma}  \right)^{h^\prime} m^{- \frac{\alpha \min \{ r,2 \} }{2 \alpha \min \{ r,2 \} +1} },
        \end{aligned}
    \]
    and
    \[
        \begin{aligned}
            \lambda^{-1} N^{-\frac{h}{2} } (1+\sqrt{\log m + \gamma} )^{h^\prime} & \leqslant \kappa^{-4} (\log m)^{\frac{h^\prime-h}{2}}  m^{- \frac{\alpha \min \{ r-\frac{1}{2},\frac{3}{2} \} }{2 \alpha \min \{ r,2 \} +1} } \left( \frac{1}{\sqrt{\log m}} + \sqrt{\frac{\gamma+\log m}{\log m}} \right)^{h^\prime} \\
            & \leqslant (\log m)^{\frac{h^\prime-h}{2}} \left( 1 + \sqrt{1+\gamma}  \right)^{h^\prime},
        \end{aligned}
    \]
    provided $ m \geqslant 3 $.
    Putting these estimates into (\ref{eq:theorem1Bound}), we obtain the desired bounds (\ref{eq:coro1bound}) with
    \[
        \begin{aligned}
            \widetilde{c}_1 &=  2 c_1 ( 3 + 2 \sqrt{c_{\alpha}^\prime}\kappa^{\frac{2}{\alpha}}  )^{2r+4}  [ 1 + \kappa^{2r} + 2 ( \kappa + \sqrt{c_{\alpha}^\prime}\kappa^{1+\frac{2}{\alpha}}  )  ]  ( 1 + \| g_{\rho} \|_{\rho_{X_{\mu}}} + \kappa^{2r-1}  ), \\
            \widetilde{c}_2 &=  2c_2 ( 3 + 2 \sqrt{c_{\alpha}^\prime}\kappa^{\frac{2}{\alpha}} )^{5}  [ 1 + \kappa + \kappa^{ \max \{ 2r,4 \} } + 2 ( \kappa + \sqrt{c_{\alpha}^\prime}\kappa^{1+\frac{2}{\alpha}}  )  ]  ( 2 + \| g_{\rho} \|_{\rho_{X_{\mu}}} + \kappa^{ \max \{ 2r-1,3 \} }  ).
        \end{aligned}
    \]

\subsection{Proof of Theorem \ref{theorem: positive kernel}}

    Recall that when $ K $ is positive semi-definite, there holds $ h_1 = h_2 = h $ in Assumption \ref{assum:holder}.
    We prove Theorem \ref{theorem: positive kernel} in the following by estimating the terms in Proposition \ref{prop:general error of positive} for different regularity index $ r $.
    To this end, first we provide an upper bound of the term $ S_2(\hat{D},\lambda) $ in Proposition \ref{prop:positivesecond}.
    \begin{proposition}
        \label{prop:S2Dhat}
		Assume Assumption \ref{assum:1}, \ref{assum:2} and \ref{assum:holder} are satisfied. Then with probability at least $ 1 - \delta - m e^{-\theta} $ there hold
		\begin{equation}
            \label{eq:S2Dhat}
			\left\| \left( \sqrt{\lambda} I + L_{K_0} \right)^{-\frac{1}{2}} \left( \hat{T}_{0}^{\hat{\mathbf{x}}} - L_{K_{0}} \right) \right\| \leqslant \kappa L (1+\sqrt{\theta})^{h} \frac{2^{\frac{h+2}{2}} B_{k}^{\frac{h}{2}} }{\lambda^{\frac{1}{4}}  N^{\frac{h}{2}} } + \mathcal{B}_{m,\lambda} \log \frac{2}{\delta}
		\end{equation}
		and
		\begin{equation}
            \label{eq:S2DhatprodBound}
			\left\| \left( \sqrt{\lambda} I + L_{K_0} \right)^{\frac{1}{2}} \left( \sqrt{\lambda}I+ \hat{T}_{0}^{\hat{\mathbf{x}}} \right)^{-\frac{1}{2}} \right\| \leqslant 1+ \kappa L (1+\sqrt{\theta})^{h} \frac{2^{\frac{h+2}{2}} B_{k}^{\frac{h}{2}} }{\lambda^{\frac{1}{2}}  N^{\frac{h}{2}} } + \lambda^{- \frac{1}{4}}  \mathcal{B}_{m,\lambda} \log \frac{2}{\delta}.
		\end{equation}
	\end{proposition}
    Hence by applying Lemma \ref{lem:TxhatMinusTx} and Proposition \ref{prop:positivesecond}, \ref{prop:S2Dhat}, for any $ 0<\delta<1 $, $ \theta>0 $, with probability at least $ 1 - \frac{\delta}{2} - m e^{-\theta}  $ there holds
    \begin{equation}
        \label{eq:BoundonfDhatMinusfDPosi}
        \| f_{\hat{D}}- f_{D} \|_{\rho_{X_{\mu}}} \leqslant (12\kappa^2+2) LM \log \frac{4}{\delta} \frac{(1+\sqrt{\theta})^{h}(2 B_{k})^{\frac{h}{2}} }{\lambda^{\frac{3}{4}}N^{\frac{h}{2}} } \mathcal{A}_{ m,N,\lambda }.
    \end{equation}

    Next, we give the bound for the term $ \| f_{D}- f_{\lambda} \|_{\rho_{X_{\mu}}} $.
    Note that $ \mathcal{S}_{3}(D,\lambda) = \mathcal{S}_{3}^\prime(D,\lambda) $ and $ \mathcal{S}_{4}(D,\lambda) = \mathcal{S}_{4}^\prime(D,\lambda)  $ in Lemma \ref{lem:fDMinusflambda} when $ K $ is positive semi-definite.
    Then we have
    \[
        \| f_{D}- f_{\lambda} \|_{\rho_{X_{\mu}}} \leqslant  2 \mathcal{S}_{1}(D,\lambda) \mathcal{S}_{2}^{2} (D,\lambda) + \mathcal{S}_{2}(D,\lambda) \mathcal{S}_{4}^\prime(D,\lambda) \| g_{\rho} \|_{\rho_{X_{\mu}}} \lambda^{r_2} .
    \]

    \noindent Under Assumption \ref{assum:regular} with $ r >0 $, there holds
    \[
		\begin{aligned}
			\| f_{\lambda} \|_{\infty} & \leqslant \sup_{\mu_{x} \in X_{\mu}} \sqrt{K_0(\mu_{x},\mu_{x})} \| f_{\lambda} \|_{K_0} = \kappa \| (\lambda I + L_{K_0}^2)^{-1} L_{K_0}^{2+r} g_{\rho} \|_{K_0} \\
			& = \kappa\left\|\left(\lambda I+L_{K_0}^{2}\right)^{-1} L_{K_0}^{\frac{3}{2}+r} L_{K_0}^{\frac{1}{2}} g_{\rho}\right\|_{K_0} \leqslant \kappa^{\max \left\{1, 2r \right\}}\left\|g_{\rho}\right\|_{\rho_{X}} \lambda^{\min \left\{0, \frac{r}{2}-\frac{1}{4}\right\}}
		\end{aligned}
	\]
    where the last inequality follows from the bound (\ref{eq:NormflambdaLargeR}) if $ r \geqslant \frac{1}{2} $ and
    \begin{equation}
        \label{eq:normLK0LessR}
		\left\| \left( \lambda I + L_{K_{0}}^{2} \right)^{-1} L_{K_{0}}^{\frac{3}{2}+r} \right\| \leqslant \left\|  \left( \lambda I + L_{K_{0}}^{2} \right)^{-\frac{1}{4}+\frac{r}{2}} \right\| \left\| \left( \lambda I + L_{K_{0}}^{2} \right)^{-\frac{3}{4}-\frac{r}{2}} L_{K_{0}}^{\frac{3}{2}+r}  \right\| \leqslant \lambda^{\frac{r}{2}-\frac{1}{4}}, \, \text{if } 0 < r <\frac{1}{2}.
	\end{equation}
    We can derive that
    \[
	    \mathcal{S}_{1}(D, \lambda) \leqslant c_{\rho} \kappa^{\max \left\{0, 2r-1 \right\}} \lambda^{\min \left\{0, \frac{r}{2}-\frac{1}{4}\right\}} \mathcal{B}_{m, \lambda} \log \frac{4}{\delta}+c_{r} \lambda^{\min \left\{1, \frac{r}{2}\right\}}, \quad \forall \mathbf{z} \in D^{m} \backslash D_{\delta, 1}^{m}.
    \]
    It follows from the proof of Theorem \ref{theorem: indefinite kernel} that there exists subsets $ D^{m}_{\delta,\ell}(\ell=2,3) $ with $ \rho(D_{\delta,\ell}^{m} ) \leqslant \frac{\delta}{2} $ such that for $ D \in D^{m} \backslash (D_{\delta,1}^{m} \cup D_{\delta,2}^{m}) $, $ \| f_{D}- f_{\lambda} \|_{\rho_{X_{\mu}}} $ can be bounded as
    \begin{equation}
        \label{eq:BoundonfDminusfrhoposi1}
        \begin{cases}
            c_{0}^\prime \left( 1+\lambda^{-\frac{1}{4}} \mathcal{B}_{m, \lambda} \right)^3 \left( \log \frac{4}{\delta} \right)^3 \lambda^{\frac{r}{2}},  & \text{ if } 0 < r < \frac{1}{2}, \\
            c_{1}^\prime \left( 1+\lambda^{-\frac{1}{4}} \mathcal{B}_{m, \lambda} \right)^{2 \max\{1,r\}} \left( \log \frac{4}{\delta} \right)^{\max\{3,2r+1\}} \left( \mathcal{B}_{m,\lambda} + \lambda^{\frac{r}{2}}  \right) & \text{ if } \frac{1}{2} \leqslant  r \leqslant  \frac{3}{2}, ,\\
        \end{cases}
    \end{equation}
    and for $D \in Z^{m} \backslash\left(D_{\delta, 1}^{m} \cup D_{\delta, 2}^{m} \cup D_{\delta, 3}^{m}\right)$ and $r>\frac{3}{2}$,
    \begin{equation}
        \label{eq:BoundonfDminusfrhoposi2}
        \| f_{D}- f_{\lambda} \|_{\rho_{X_{\mu}}} \leqslant c_{2}^{\prime} \left( 1+\lambda^{-\frac{1}{4}} \mathcal{B}_{m, \lambda} \right)^{2} \left( \log \frac{4}{ \delta} \right)^3 \left( \mathcal{B}_{m, \lambda} +\lambda^{\frac{1}{4}} m^{-\frac{1}{2}} +\lambda^{\min \left\{1, \frac{r}{2}\right\}} \right),
    \end{equation}
    where $ c_{0}^\prime = 2(c_{\rho} + \| g_{\rho} \|_{\rho_{X_{\mu}}}) + 2c_{r} $.

    Finally, by combining (\ref{eq:BoundonfDhatMinusfDPosi}), (\ref{eq:BoundonfDminusfrhoposi1}) with Lemma \ref{lem:approxerror} and applying the error decomposition (\ref{eq:DecomfDhatminusfrho}), we can verify the desired bounds by scaling $\frac{3\delta}{2}$ to $ \delta $ and $ m e^{-\theta}  $ to $ e^{-\gamma} $, while the last case can be verified by combining the bound of (\ref{eq:BoundonfDhatMinusfDPosi}) and (\ref{eq:BoundonfDminusfrhoposi2}) and scaling $ 2 \delta $ to $ \delta $.

\subsection{Proof of Corollary \ref{corollary2}}
    With the values of $ \lambda $ given by (\ref{eq:beta}), for $ 0<r\leqslant 2 $ we have
    \begin{equation*}
        \mathcal{B}_{m, \lambda}=\frac{2 \kappa}{\sqrt{m}}\left\{\frac{\kappa}{\sqrt{m} \lambda^{\frac{1}{4}}}+\sqrt{\mathcal{N} ( \lambda^{\frac{1}{2}} )}\right\} \leqslant 2 (\kappa^2 + \kappa\sqrt{c} ) m^{-\frac{ \alpha \max\{ \frac{1}{2},r \} }{ \alpha \max\{1,2r\} + 1 }},
    \end{equation*}
    and
    \[
        \lambda^{-\frac{1}{4}} \mathcal{B}_{m,\lambda} \leqslant 2(\kappa^2+\kappa\sqrt{c} ).
    \]
    Substituting the values of $ \lambda $ and $ N $ into the error bounds (\ref{eq:thm2}) yields that when $ 0<r\leqslant \frac{1}{2} $,
    \[
        \lambda^{-\frac{3}{4}} N^{-\frac{h}{2}} \left( 1 + \sqrt{\gamma+\log m} \right)^{h} = m^{-\frac{ \alpha r}{ \alpha + 1}} \left( \frac{1}{\sqrt{\log m}} + \sqrt{\frac{\gamma+\log m}{\log m}} \right)^{h} \leqslant \left( 1+ \sqrt{1+\gamma}  \right)^{h}m^{-\frac{ \alpha r}{ \alpha + 1}},
    \]
    and
    \[
        \lambda^{-\frac{1}{2}} N^{-\frac{h}{2}} \left( 1 + \sqrt{\gamma+\log m} \right)^{h} \leqslant \left( 1+ \sqrt{1+\gamma}  \right)^{h}.
    \]
    Then we have
    \[
        \left\|f_{\hat{D}}-f_{\rho}\right\|_{\rho_{X_{\mu}}} \leqslant \widetilde{d}_{0} \left( \log \frac{12}{\delta} \right)^{3} \left(1+\sqrt{1+\gamma} \right)^{4h} m^{-\frac{ \alpha r}{ \alpha + 1}},
    \]
    where $ \widetilde{d}_{0} = 2 d_0 [ 1 + 2(\kappa^2+\kappa\sqrt{c} ) + 2 \kappa L (2B_{k})^{\frac{h}{2}}]^3 $.

    \noindent Similarly, when $ \frac{1}{2} < r \leqslant \frac{3}{2} $, there holds
    \[
        \left\|f_{\hat{D}}-f_{\rho}\right\|_{\rho_{X_{\mu}}} \leqslant \widetilde{d}_{1} \left( \log \frac{12}{\delta} \right)^{\max\{ 3,2r+1 \}  } \left(1+\sqrt{1+\gamma} \right)^{h \max\{ 3,2r+1 \} } m^{-\frac{ \alpha r}{2 \alpha r+ 1 }},
    \]
    where $ \widetilde{d}_{1} = 6 d_{1} (\kappa^2+\kappa\sqrt{c} ) [ 1 + 2(\kappa^2+\kappa\sqrt{c} ) + 2 \kappa L (2B_{k})^{\frac{h}{2}}]^{2\max\{ 1,r \} } $.

    \noindent When $ \frac{3}{2} < r \leqslant 2 $, we can also obtain that
    \[
        \begin{aligned}
            & \left\|f_{\hat{D}}-f_{\rho}\right\|_{\rho_{X_{\mu}}} \leqslant d_{2} \mathcal{A}^{2}_{m,N,\lambda} \left( \log \frac{16}{\delta} \right)^{3} \left[\mathcal{B}_{m, \lambda}+\lambda^{\frac{1}{4}} m^{-\frac{1}{2}}+\lambda^{ \frac{r}{2}} + \lambda^{-\frac{3}{4}} N^{-\frac{h}{2}}\left( 1 + \sqrt{\gamma+\log m} \right)^{h}  \right] \\
            \leqslant & d_2 \left[ 1 + 2(\kappa^2+\kappa\sqrt{c} ) + 2 \kappa L (2B_{k})^{\frac{h}{2}} \right]^3 \left( \log \frac{16}{\delta} \right)^{3} \left(1+\sqrt{1+\gamma} \right)^{3h} \big( m^{- \frac{\alpha r}{2 \alpha r+1}} + m^{- \frac{\alpha + 2 \alpha r+1}{4 \alpha r+2}} \\
            &+ m^{- \frac{ \alpha r}{2 \alpha r+1}} + m^{- \frac{ \alpha r}{2 \alpha r+1}} \big) \\
            \leqslant & \widetilde{d}_{2} \left( \log \frac{16}{\delta} \right)^{3} \left(1+\sqrt{1+\gamma} \right)^{3h} m^{- \frac{ \alpha r}{2 \alpha r+ 1 }},
        \end{aligned}
    \]
    where $ \widetilde{d}_{2} = 4 d_{2} \left[ 1 + 2(\kappa^2+\kappa\sqrt{c} ) + 2 \kappa L (2B_{k})^{\frac{h}{2}} \right]^3 $. For $ r > 2 $, we have
    \[
        \mathcal{B}_{m,\lambda} \leqslant 2(\kappa^2+\kappa\sqrt{c} ) m^{- \frac{2 \alpha}{4 \alpha + 1}} \; \text{and } \left\|f_{\hat{D}}-f_{\rho}\right\|_{\rho_{X_{\mu}}} \leqslant \widetilde{d}_{2} \left( \log \frac{16}{\delta} \right)^{3} \left(1+\sqrt{1+\gamma} \right)^{3h} m^{- \frac{2 \alpha }{4\alpha +1}}.
    \]
    Finally, we complete the proof of Corollary \ref{corollary2} by combining the above results.

\bibliography{main}
\bibliographystyle{abbrv}
\section{Appendix}

	This section presents some detailed proofs of the operator representation of $ f_{\hat{D}} $ and estimates stated in Section \ref{section: preliminary results and error decomposition}.

\subsection{Proof of Proposition \ref{prop:fDhatop}}

	The optimization scheme (\ref{eq:fDhat}) can be represented by sampling operators mentioned in Section \ref{subsection: preliminary results}, i.e., $ f_{\hat{D}} = m U \hat{S}_{1}^{*} \alpha_{\hat{D}} $ with
	\[
		\alpha_{\hat{D}} = \arg\min_{\alpha \in \mathbb{R}^{m} } \left\{ \frac{1}{m} \left\| m \hat{S}_{0} U \hat{S}_{1}^{*} \alpha - \mathbf{y} \right\|_{2}^{2} + \lambda m \alpha ^{T} \alpha  \right\}.
	\]
	This optimization problem is strictly convex, hence the stationary point is the unique solution.
	A direct derivation with respect to $ \alpha $ yields
	\[
		\alpha_{\hat{D}} = \left( \lambda \mathbb{I}_{m} +  \hat{S}_{1} U^{*} \hat{S}_{0}^{*} \hat{S}_{0} U \hat{S}_{1}^{*} \right)^{-1} \frac{1}{m} \hat{S}_{1} U^{*} \hat{S}_{0}^{*}\mathbf{y}.
	\]

	Further, we complete the proof by substituting the representations of $ T^{ \hat{\mathbf{x}}}  $ and $ T_{*}^{ \hat{\mathbf{x}}}  $ in (\ref{eq:factsAboutThat}),
	\[
		\begin{aligned}
			f_{\hat{D}} &= m U \hat{S}_{1}^{*} \left( \lambda \mathbb{I}_{m} +  \hat{S}_{1} U^{*} \hat{S}_{0}^{*} \hat{S}_{0} U \hat{S}_{1}^{*} \right)^{-1} \frac{1}{m} \hat{S}_{1} U^{*} \hat{S}_{0}^{*}\mathbf{y} \\
			&= \left( \lambda I +  U \hat{S}_{1}^{*}\hat{S}_{1} U^{*} \hat{S}_{0}^{*} \hat{S}_{0} \right)^{-1}   U \hat{S}_{1}^{*} \hat{S}_{1} U^{*} \hat{S}_{0}^{*}\mathbf{y}  \\
			&= \left(\lambda I+T^{\hat{\mathbf{x}}} T_{*}^{\hat{\mathbf{x}}}\right)^{-1} T^{\hat{\mathbf{x}}} U^{*} \hat{S}_{0}^{*} \mathbf{y}.
		\end{aligned}
	\]
	Meanwhile, we can obtain the result in (\ref{eq:fDhat explicit form with matrix}) with $ \hat{\mathbb{K}}_{m} = m \hat{S}_{0} U \hat{S}_{1}^{*} $.

\subsection{Proof of Proposition \ref{prop:indefinitesecond}}
	This part aims at estimating the term $ \| f_{\hat{D}} - f_{D} \|_{\rho_{X_{\mu}}} $ corresponding to the difference between the first-stage sampling and the second-stage sampling.
	By (\ref{eq:fDoperator}) and (\ref{eq:fDhatOperRepre}), we can make a decomposition for $ f_{\hat{D}} - f_{D} $ as
	\[
		\begin{aligned}
			f_{\hat{D}} - f_{D} = & \left( \lambda I + \hat{T}_{0}^{\hat{\mathbf{x}}} T_{0}^{\hat{\mathbf{x}}}  \right)^{-1} \hat{T}_{0}^{\hat{\mathbf{x}}} \hat{S}_{0}^{*}\mathbf{y} - \left( \lambda I + \hat{T}_{0}^{\mathbf{x}} T_{0}^{\mathbf{x}} \right)^{-1} \hat{T}_{0}^{\mathbf{x}}  S_0^{*} \mathbf{y} \\
			= & \left( \lambda I + \hat{T}_{0}^{\hat{\mathbf{x}}} T_{0}^{\hat{\mathbf{x}}}  \right)^{-1} \hat{T}_{0}^{\hat{\mathbf{x}}} \left(  \hat{S}_{0}^{*} \mathbf{y} -  S_0^{*} \mathbf{y} \right) + \left[ \left( \lambda I + \hat{T}_{0}^{\hat{\mathbf{x}}} T_{0}^{\hat{\mathbf{x}}}  \right)^{-1} \hat{T}_{0}^{\hat{\mathbf{x}}} - \left( \lambda I + \hat{T}_{0}^{\mathbf{x}} T_{0}^{\mathbf{x}} \right)^{-1} \hat{T}_{0}^{\mathbf{x}} \right]  S_0^{*} \mathbf{y}
		\end{aligned}
	\]
	It follows that
	\begin{equation}
		\label{eq:fDhatminusfD}
		\begin{aligned}
			\left\| f_{\hat{D}} - f_{D} \right\|_{\rho_{X_{\mu}}}  = & \left\| L_{K_{0}}^{\frac{1}{2}} \left(f_{\hat{D}} - f_{D} \right) \right\|_{K_0} \leqslant  \mathcal{T}_{3} + \mathcal{T}_{4},
		\end{aligned}
	\end{equation}
	where
	\[
		\begin{aligned}
			\mathcal{T}_{3} = & \left\| L_{K_{0}}^{\frac{1}{2}} \left( \lambda I + \hat{T}_{0}^{\hat{\mathbf{x}}} T_{0}^{\hat{\mathbf{x}}}  \right)^{-1} \hat{T}_{0}^{\hat{\mathbf{x}}} \left(  \hat{S}_{0}^{*} \mathbf{y} -  S_0^{*} \mathbf{y} \right) \right\|_{K_0}, \\
			\mathcal{T}_{4} = & \left\| L_{K_{0}}^{\frac{1}{2}} \left[ \left( \lambda I + \hat{T}_{0}^{\hat{\mathbf{x}}} T_{0}^{\hat{\mathbf{x}}}  \right)^{-1} \hat{T}_{0}^{\hat{\mathbf{x}}} - \left( \lambda I + \hat{T}_{0}^{\mathbf{x}} T_{0}^{\mathbf{x}} \right)^{-1} \hat{T}_{0}^{\mathbf{x}} \right]  S_0^{*} \mathbf{y} \right\|_{K_0}.
		\end{aligned}
	\]
	First let us introduce some basic estimates which can be utilized to bound $ \mathcal{T}_{3} $ and $ \mathcal{T}_{4} $.
	\begin{lemma}
		For $ \lambda >0 $, we have
		\begin{equation}
			\label{eq:txhAhatinvsthxhBound}
			\left\| T_{0}^{\hat{\mathbf{x}}} \left( \lambda I + \hat{T}_{0}^{\hat{\mathbf{x}}} T_{0}^{\hat{\mathbf{x}}} \right)^{-1} \hat{T}_{0}^{\hat{\mathbf{x}}}  \right\| \leqslant 2 \kappa^{4} \lambda^{-1},
		\end{equation}
		and
		\begin{equation}
			\label{eq:txhAhatinvsBound}
			\left\| T_{0}^{\hat{\mathbf{x}}} \left( \lambda I + \hat{T}_{0}^{\hat{\mathbf{x}}} T_{0}^{\hat{\mathbf{x}}} \right)^{-1} \right\| \leqslant 2 \kappa^{2} \lambda^{-1}.
		\end{equation}
	\end{lemma}
	\begin{proof}
		We see that there holds
		\[
			\left( \lambda I + \hat{T}_{0}^{\hat{\mathbf{x}}} T_{0}^{\hat{\mathbf{x}}}  \right)^{-1} = \frac{1}{\lambda} \left[ I - \hat{T}_{0}^{\hat{\mathbf{x}}} ( T_{0}^{\hat{\mathbf{x}}} )^{\frac{1}{2}} \left( \lambda I + ( T_{0}^{\hat{\mathbf{x}}} )^{\frac{1}{2}} \hat{T}_{0}^{\hat{\mathbf{x}}} ( T_{0}^{\hat{\mathbf{x}}} )^{\frac{1}{2}} \right)^{-1} (T_{0}^{\hat{\mathbf{x}}} )^{\frac{1}{2}}  \right].
		\]
		Further,
		\begin{equation*}
			\begin{aligned}
				& \left\| T_{0}^{\hat{\mathbf{x}}} \left( \lambda I + \hat{T}_{0}^{\hat{\mathbf{x}}} T_{0}^{\hat{\mathbf{x}}} \right)^{-1} \hat{T}_{0}^{\hat{\mathbf{x}}}  \right\| \\
				=& \frac{1}{\lambda} \left\| T_{0}^{\hat{\mathbf{x}}} \hat{T}_{0}^{\hat{\mathbf{x}}}  - ( T_{0}^{\hat{\mathbf{x}}} )^{\frac{1}{2}} ( T_{0}^{\hat{\mathbf{x}}} )^{\frac{1}{2}} \hat{T}_{0}^{\hat{\mathbf{x}}} ( T_{0}^{\hat{\mathbf{x}}} )^{\frac{1}{2}} \left( \lambda I + ( T_{0}^{\hat{\mathbf{x}}} )^{\frac{1}{2}} \hat{T}_{0}^{\hat{\mathbf{x}}} ( T_{0}^{\hat{\mathbf{x}}} )^{\frac{1}{2}} \right)^{-1} (T_{0}^{\hat{\mathbf{x}}} )^{\frac{1}{2}} \hat{T}_{0}^{\hat{\mathbf{x}}} \right\| \\
				\leqslant & \frac{1}{\lambda} \left[ \left\| T_{0}^{\hat{\mathbf{x}}} \hat{T}_{0}^{\hat{\mathbf{x}}}  \right\| + \left\| (T_{0}^{\hat{\mathbf{x}}} )^{\frac{1}{2}} \right\| \left\| ( T_{0}^{\hat{\mathbf{x}}} )^{\frac{1}{2}} \hat{T}_{0}^{\hat{\mathbf{x}}} ( T_{0}^{\hat{\mathbf{x}}} )^{\frac{1}{2}} \left( \lambda I + ( T_{0}^{\hat{\mathbf{x}}} )^{\frac{1}{2}} \hat{T}_{0}^{\hat{\mathbf{x}}} ( T_{0}^{\hat{\mathbf{x}}} )^{\frac{1}{2}} \right)^{-1} \right\| \left\| (T_{0}^{\hat{\mathbf{x}}} )^{\frac{1}{2}} \hat{T}_{0}^{\hat{\mathbf{x}}} \right\| \right] \\
				\leqslant & \frac{1}{\lambda} \left[ \left\| T_{0}^{\hat{\mathbf{x}}} \right\| \left\| \hat{T}_{0}^{\hat{\mathbf{x}}} \right\| + \left\| (T_{0}^{\hat{\mathbf{x}}} )^{\frac{1}{2}} \right\|^2 \left\| \hat{T}_{0}^{\hat{\mathbf{x}}} \right\| \right] \leqslant 2 \kappa^4 \lambda^{-1}.
			\end{aligned}
		\end{equation*}
		Similarly, we have
		\begin{equation*}
			\begin{aligned}
				& \left\| T_{0}^{\hat{\mathbf{x}}} \left( \lambda I + \hat{T}_{0}^{\hat{\mathbf{x}}} T_{0}^{\hat{\mathbf{x}}}  \right)^{-1}  \right\| = \frac{1}{\lambda} \left\| T_{0}^{\hat{\mathbf{x}}}  - ( T_{0}^{\hat{\mathbf{x}}} )^{\frac{1}{2}} ( T_{0}^{\hat{\mathbf{x}}} )^{\frac{1}{2}} \hat{T}_{0}^{\hat{\mathbf{x}}} ( T_{0}^{\hat{\mathbf{x}}} )^{\frac{1}{2}} \left( \lambda I + ( T_{0}^{\hat{\mathbf{x}}} )^{\frac{1}{2}} \hat{T}_{0}^{\hat{\mathbf{x}}} ( T_{0}^{\hat{\mathbf{x}}} )^{\frac{1}{2}} \right)^{-1} (T_{0}^{\hat{\mathbf{x}}} )^{\frac{1}{2}} \right\|  \\
				\leqslant & \frac{1}{\lambda} \left[ \left\| T_{0}^{\hat{\mathbf{x}}} \right\| + \left\| (T_{0}^{\hat{\mathbf{x}}} )^{\frac{1}{2}} \right\| \left\| ( T_{0}^{\hat{\mathbf{x}}} )^{\frac{1}{2}} \hat{T}_{0}^{\hat{\mathbf{x}}} ( T_{0}^{\hat{\mathbf{x}}} )^{\frac{1}{2}} \left( \lambda I + ( T_{0}^{\hat{\mathbf{x}}} )^{\frac{1}{2}} \hat{T}_{0}^{\hat{\mathbf{x}}} ( T_{0}^{\hat{\mathbf{x}}} )^{\frac{1}{2}} \right)^{-1} \right\| \left\| (T_{0}^{\hat{\mathbf{x}}} )^{\frac{1}{2}} \right\| \right] \\
				\leqslant & 2 \kappa^2 \lambda^{-1}.
			\end{aligned}
		\end{equation*}
	\end{proof}

	Now we are in a position to estimate $ \mathcal{T}_{3} $ and $ \mathcal{T}_{4} $.
	For convenience, let
	\[
		A :=  \lambda I + \hat{T}_{0}^{\mathbf{x}} T_{0}^{\mathbf{x}}, \qquad
		\hat{A} := \lambda I + \hat{T}_{0}^{\hat{\mathbf{x}}} T_{0}^{\hat{\mathbf{x}}}.
	\]
	Then we can make a decomposition for $ \mathcal{T}_{3} $ as
	\begin{equation}
		\mathcal{T}_{3} \leqslant \mathcal{T}_{3.1} + \mathcal{T}_{3.2},
	\end{equation}
	where
	\[
		\begin{aligned}
			\mathcal{T}_{3.1} = & \left\| L_{K_{0}}^{\frac{1}{2}} \left( \hat{A}^{-1} - A^{-1} \right) \hat{T}_{0}^{\hat{\mathbf{x}}} \left(  \hat{S}_{0}^{*} \mathbf{y} -  S_0^{*} \mathbf{y} \right)  \right\|_{K_0}, \\
			\mathcal{T}_{3.2} = & \left\| L_{K_{0}}^{\frac{1}{2}} A^{-1} \hat{T}_{0}^{\hat{\mathbf{x}}} \left(  \hat{S}_{0}^{*} \mathbf{y} - S_0^{*} \mathbf{y} \right)  \right\|_{K_0}.
		\end{aligned}
	\]
	Applying the identity $ \hat{A}^{-1} - A^{-1} = A^{-1} (A^{-1} -\hat{A}^{-1}) \hat{A}^{-1} $ to $ \mathcal{T}_{3.1} $ yields
	\[
		\begin{aligned}
			\mathcal{T}_{3.1} = & \left\| L_{K_{0}}^{\frac{1}{2}}  A^{-1} \left[ \hat{T}_{0}^{\mathbf{x}} \left( T_{0}^{\mathbf{x}} - T_{0}^{\hat{\mathbf{x}}} \right) + \left( \hat{T}_{0}^{\mathbf{x}}  - \hat{T}_{0}^{\hat{\mathbf{x}}} \right) T_{0}^{\hat{\mathbf{x}}}  \right] \hat{A}^{-1} \hat{T}_{0}^{\hat{\mathbf{x}}} \left(  \hat{S}_{0}^{*} \mathbf{y} -  S_0^{*} \mathbf{y} \right)  \right\|_{K_0} \\
			\leqslant & \mathcal{T}_{3.1.1} + \mathcal{T}_{3.1.2},
		\end{aligned}
	\]
	where
	\[
		\begin{aligned}
			\mathcal{T}_{3.1.1} =& \left\| L_{K_{0}}^{\frac{1}{2}}  A^{-1} \hat{T}_{0}^{\mathbf{x}} \left( T_{0}^{\mathbf{x}} - T_{0}^{\hat{\mathbf{x}}} \right)  \hat{A}^{-1} \hat{T}_{0}^{\hat{\mathbf{x}}} \left(  \hat{S}_{0}^{*} \mathbf{y} -  S_0^{*} \mathbf{y} \right)  \right\|_{K_0}, \\
			\mathcal{T}_{3.1.2} =& \left\| L_{K_{0}}^{\frac{1}{2}}  A^{-1} \left( \hat{T}_{0}^{\mathbf{x}}  - \hat{T}_{0}^{\hat{\mathbf{x}}} \right) T_{0}^{\hat{\mathbf{x}}}   \hat{A}^{-1} \hat{T}_{0}^{\hat{\mathbf{x}}} \left(  \hat{S}_{0}^{*} \mathbf{y} -  S_0^{*} \mathbf{y} \right)  \right\|_{K_0}.
		\end{aligned}
	\]
    Considering that $ \| L_{K_{0}}^{\frac{1}{2}} (\sqrt{\lambda} I + L_{K_{0}})^{-\frac{1}{2}} \| \leqslant 1 $, we have
	\begin{equation}
		\label{eq:49}
		\begin{aligned}
			\left\| L_{K_{0}}^{\frac{1}{2}}  A^{-1} \hat{T}_{0}^{\mathbf{x}} \right\| \leqslant & \left\| (\sqrt{\lambda} I + L_{K_{0}})^{\frac{1}{2}} (\sqrt{\lambda} I + \hat{T}_{0}^{\mathbf{x}} )^{-\frac{1}{2}} \right\| \left\| (\sqrt{\lambda} I + \hat{T}_{0}^{\mathbf{x}} )^{\frac{1}{2}} A^{-1}  \hat{T}_{0}^{\mathbf{x}} (\sqrt{\lambda} I + \hat{T}_{0}^{\mathbf{x}} )^{\frac{1}{2}}  \right\| \\
			& \qquad \cdot \left\| (\sqrt{\lambda} I + \hat{T}_{0}^{\mathbf{x}} )^{-\frac{1}{2}} \right\| \\
			\leqslant & \lambda^{-\frac{1}{4}} \mathcal{S}_{2}(D,\lambda) \mathcal{S}_{3}(D,\lambda) \leqslant 2 \lambda^{-\frac{1}{4}} \mathcal{S}_{2}(D,\lambda) \mathcal{S}_{5}(D,\lambda).
		\end{aligned}
	\end{equation}
    where the last inequality follows from (\ref{eq:S3'bound}) and (\ref{eq:S3S4}).
	Combining the bound above and the result derived by Lemma \ref{lem:reversibility} that for $ \lambda > 0 $,
	\begin{equation}
		\label{eq:50}
		\left\| \hat{A}^{-1} \right\| = \left\| \left( \lambda I + \hat{T}_{0}^{\hat{\mathbf{x}}} T_{0}^{\hat{\mathbf{x}}}  \right)^{-1}  \right\| \leqslant \frac{1}{\lambda} \left( 1 + \frac{\kappa^2}{\sqrt{\lambda} } \right) \leqslant 2 \kappa^2 \lambda^{-\frac{3}{2}},
	\end{equation}
	we see that the term $ \mathcal{T}_{3.1.1} $ can be further bounded as
	\[
		\begin{aligned}
			\mathcal{T}_{3.1.1} \leqslant & \left\| L_{K_{0}}^{\frac{1}{2}}  A^{-1} \hat{T}_{0}^{\mathbf{x}} \right\| \left\| T_{0}^{\mathbf{x}} - T_{0}^{\hat{\mathbf{x}}} \right\| \left\| \hat{A}^{-1} \right\| \left\| \hat{T}_{0}^{\hat{\mathbf{x}}} \right\| \left\|  \hat{S}_{0}^{*} \mathbf{y} -  S_0^{*} \mathbf{y} \right\|_{K_0} \\
			\leqslant & 2 \lambda^{-\frac{1}{4}} \mathcal{S}_{2}(D,\lambda) \mathcal{S}_{5}(D,\lambda) \left\| T_{0}^{\mathbf{x}} - T_{0}^{\hat{\mathbf{x}}} \right\| 2 \kappa^2 \lambda^{-\frac{3}{2}} \kappa^2 \left\|  \hat{S}_{0}^{*} \mathbf{y} -  S_0^{*} \mathbf{y} \right\|_{K_0} \\
			\leqslant & 4 \kappa^4 \lambda^{-\frac{7}{4}} \mathcal{S}_{2}(D, \lambda) \mathcal{S}_{5}(D, \lambda) \left\| T_{0}^{\mathbf{x}} - T_{0}^{\hat{\mathbf{x}}} \right\| \left\|  \hat{S}_{0}^{*} \mathbf{y} -  S_0^{*} \mathbf{y} \right\|_{K_0}.
		\end{aligned}
	\]

	For the second term $ \mathcal{T}_{3.1.2} $, combining (\ref{eq:txhAhatinvsthxhBound}) and the fact that
	\begin{equation}
		\label{eq:52}
		\begin{aligned}
			\left\| L_{K_{0}}^{\frac{1}{2}}  A^{-1} \right\| \leqslant & \left\| (\sqrt{\lambda} I + L_{K_{0}})^{\frac{1}{2}} (\sqrt{\lambda} I + \hat{T}_{0}^{\mathbf{x}} )^{-\frac{1}{2}} \right\| \left\| (\sqrt{\lambda} I + \hat{T}_{0}^{\mathbf{x}} )^{\frac{1}{2}} \left(\lambda I+  \hat{T}_{0}^{\mathbf{x}} T_{0}^{\mathbf{x}} \right)^{-1} \right\| \\
			\leqslant & \mathcal{S}_{2}(D,\lambda) \mathcal{S}_{5}(D,\lambda) \left\| (\sqrt{\lambda} I + \hat{T}_{0}^{\mathbf{x}} )^{\frac{1}{2}} \left(\lambda I+  \hat{T}_{0}^{\mathbf{x}} \hat{T}_{0}^{\mathbf{x}}  \right)^{-1} \right\| \\
			\leqslant & 2 \lambda^{-\frac{3}{4}} \mathcal{S}_{2}(D,\lambda) \mathcal{S}_{5}(D,\lambda),
		\end{aligned}
	\end{equation}
	we can derive
	\[
		\begin{aligned}
			\mathcal{T}_{3.1.2} \leqslant & \left\| L_{K_{0}}^{\frac{1}{2}}  A^{-1} \right\| \left\| \hat{T}_{0}^{\mathbf{x}}  - \hat{T}_{0}^{\hat{\mathbf{x}}} \right\| \left\| T_{0}^{\hat{\mathbf{x}}}  \hat{A}^{-1} \hat{T}_{0}^{\hat{\mathbf{x}}} \right\| \left\| \hat{S}_{0}^{*} \mathbf{y} -  S_0^{*} \mathbf{y} \right\|_{K_0} \\
			\leqslant & 2 \lambda^{-\frac{3}{4}} \mathcal{S}_{2}(D, \lambda) \mathcal{S}_{5}(D, \lambda) \left\| \hat{T}_{0}^{\mathbf{x}}  - \hat{T}_{0}^{\hat{\mathbf{x}}} \right\| 2 \kappa^{4} \lambda^{-1} \left\|  \hat{S}_{0}^{*} \mathbf{y} - S_0^{*} \mathbf{y} \right\|_{K_0} \\
			\leqslant & 4 \kappa^{4} \lambda^{-\frac{7}{4}}  \mathcal{S}_{2}(D, \lambda) \mathcal{S}_{5}(D, \lambda) \left\| \hat{T}_{0}^{\mathbf{x}}  - \hat{T}_{0}^{\hat{\mathbf{x}}} \right\| \left\| \hat{S}_{0}^{*} \mathbf{y} -  S_0^{*} \mathbf{y} \right\|_{K_0}.
		\end{aligned}
	\]

	In addition, $ \mathcal{T}_{3.2} $ can be bounded by (\ref{eq:52}) as
	\[
			\mathcal{T}_{3.2} \leqslant \left\| L_{K_{0}}^{\frac{1}{2}} A^{-1} \right\| \left\| \hat{T}_{0}^{\hat{\mathbf{x}}} \right\| \left\|  \hat{S}_{0}^{*} \mathbf{y} -  S_0^{*} \mathbf{y} \right\|_{K_0} \leqslant 2 \kappa^2 \lambda^{-\frac{3}{4}} \mathcal{S}_{2}(D, \lambda) \mathcal{S}_{5}(D, \lambda) \left\|  \hat{S}_{0}^{*} \mathbf{y} -  S_0^{*} \mathbf{y} \right\|_{K_0}.
	\]
	Combining the bounds of $ \mathcal{T}_{3.1.1} $, $ \mathcal{T}_{3.1.2} $ and  $ \mathcal{T}_{3.2} $ together yields
	\begin{equation}
		\label{eq:T3}
		\mathcal{T}_{3} \leqslant 4 \kappa^4 \lambda^{-\frac{3}{4}} \mathcal{S}_{2}(D, \lambda) \mathcal{S}_{5}(D, \lambda) \mathcal{S}_{6}(\hat{D}, \lambda)  \left\|  \hat{S}_{0}^{*} \mathbf{y} -  S_0^{*} \mathbf{y} \right\|_{K_0},
	\end{equation}
	with $ \mathcal{S}_{6}(\hat{D},\lambda) = \lambda^{-1} \left\| T_{0}^{\mathbf{x}} - T_{0}^{\hat{\mathbf{x}}} \right\| + \lambda^{-1}  \left\| \hat{T}_{0}^{\mathbf{x}}  - \hat{T}_{0}^{\hat{\mathbf{x}}} \right\| + 1 $.

	The second term of (\ref{eq:fDhatminusfD}) can be decomposed as
	\begin{equation}
		\label{eq:T4decom}
		\begin{aligned}
			\mathcal{T}_{4} = & \left\| L_{K_{0}}^{\frac{1}{2}} \left[ \hat{A}^{-1} \hat{T}_{0}^{\hat{\mathbf{x}}} - A^{-1} \hat{T}_{0}^{\mathbf{x}} \right]   S_0^{*} \mathbf{y} \right\|_{K_0} \\
			= &  \left\| L_{K_{0}}^{\frac{1}{2}} \left[ \hat{A}^{-1} \left( \hat{T}_{0}^{\hat{\mathbf{x}}} - \hat{T}_{0}^{\mathbf{x}} \right) +  \left( \hat{A}^{-1} - A^{-1} \right) \hat{T}_{0}^{\mathbf{x}}  \right]  S_0^{*} \mathbf{y} \right\|_{K_0} \\
			= &  \left\| L_{K_{0}}^{\frac{1}{2}} \left[ \left( \hat{A}^{-1} - A^{-1} + A^{-1} \right)  \left( \hat{T}_{0}^{\hat{\mathbf{x}}} - \hat{T}_{0}^{\mathbf{x}} \right) +  \left( \hat{A}^{-1} - A^{-1} \right) \hat{T}_{0}^{\mathbf{x}}  \right]  S_0^{*} \mathbf{y} \right\|_{K_0} \\
			\leqslant & \mathcal{T}_{4.1} +\mathcal{T}_{4.2} +\mathcal{T}_{4.3},
		\end{aligned}
	\end{equation}
	where
	\[
		\begin{aligned}
			\mathcal{T}_{4.1} = & \left\| L_{K_{0}}^{\frac{1}{2}} \left( \hat{A}^{-1} - A^{-1} \right)  \left( \hat{T}_{0}^{\hat{\mathbf{x}}} - \hat{T}_{0}^{\mathbf{x}} \right)    S_0^{*} \mathbf{y} \right\|_{K_0}, \\
			\mathcal{T}_{4.2} = & \left\| L_{K_{0}}^{\frac{1}{2}} A^{-1} \left( \hat{T}_{0}^{\hat{\mathbf{x}}} - \hat{T}_{0}^{\mathbf{x}} \right)  S_0^{*} \mathbf{y} \right\|_{K_0}, \\
			\mathcal{T}_{4.3} = & \left\| L_{K_{0}}^{\frac{1}{2}} \left( \hat{A}^{-1} - A^{-1} \right) \hat{T}_{0}^{\mathbf{x}}   S_0^{*} \mathbf{y} \right\|_{K_0}.
		\end{aligned}
	\]
    We can apply the identity $ \hat{A}^{-1} - A^{-1} = A^{-1} (A^{-1} -\hat{A}^{-1}) \hat{A}^{-1} $ again to bound $ \mathcal{T}_{4.1} $ as
	\[
		\begin{aligned}
			\mathcal{T}_{4.1} = & \left\| L_{K_{0}}^{\frac{1}{2}} A^{-1} \left[ \hat{T}_{0}^{\mathbf{x}} \left( T_{0}^{\mathbf{x}} - T_{0}^{\hat{\mathbf{x}}} \right) + \left( \hat{T}_{0}^{\mathbf{x}}  - \hat{T}_{0}^{\hat{\mathbf{x}}} \right) T_{0}^{\hat{\mathbf{x}}}  \right] \hat{A}^{-1}  \left( \hat{T}_{0}^{\hat{\mathbf{x}}} - \hat{T}_{0}^{\mathbf{x}} \right)    S_0^{*} \mathbf{y} \right\|_{K_0}, \\
			\leqslant & \mathcal{T}_{4.1.1} +\mathcal{T}_{4.1.2},
		\end{aligned}
	\]
	where
	\[
		\begin{aligned}
			\mathcal{T}_{4.1.1} =& \left\| L_{K_{0}}^{\frac{1}{2}}  A^{-1} \hat{T}_{0}^{\mathbf{x}} \left( T_{0}^{\mathbf{x}} - T_{0}^{\hat{\mathbf{x}}} \right)  \hat{A}^{-1} \left( \hat{T}_{0}^{\hat{\mathbf{x}}} - \hat{T}_{0}^{\mathbf{x}} \right)  S_0^{*} \mathbf{y}  \right\|_{K_0}, \\
			\mathcal{T}_{4.1.2} =& \left\| L_{K_{0}}^{\frac{1}{2}}  A^{-1} \left( \hat{T}_{0}^{\mathbf{x}}  - \hat{T}_{0}^{\hat{\mathbf{x}}} \right) T_{0}^{\hat{\mathbf{x}}}   \hat{A}^{-1} \left( \hat{T}_{0}^{\hat{\mathbf{x}}} - \hat{T}_{0}^{\mathbf{x}} \right)  S_0^{*} \mathbf{y}  \right\|_{K_0}.
		\end{aligned}
	\]
	$ \mathcal{T}_{4.1.1} $ and $ \mathcal{T}_{4.1.2} $ can be bounded as
	\[
		\begin{aligned}
			\mathcal{T}_{4.1.1} \leqslant &  \left\| L_{K_{0}}^{\frac{1}{2}}  A^{-1} \hat{T}_{0}^{\mathbf{x}} \right\| \left\| T_{0}^{\mathbf{x}} - T_{0}^{\hat{\mathbf{x}}} \right\| \left\| \hat{A}^{-1} \right\| \left\|  \hat{T}_{0}^{\hat{\mathbf{x}}} - \hat{T}_{0}^{\mathbf{x}} \right\| \left\|  S_0^{*} \mathbf{y}  \right\|_{K_0} \\
			\leqslant & 2 \lambda^{-\frac{1}{4}} \mathcal{S}_{2}(D,\lambda) \mathcal{S}_{5}(D,\lambda) \left\| T_{0}^{\mathbf{x}} - T_{0}^{\hat{\mathbf{x}}} \right\| 2 \kappa^2 \lambda^{-\frac{3}{2}}\left\|  \hat{T}_{0}^{\hat{\mathbf{x}}} - \hat{T}_{0}^{\mathbf{x}} \right\|  \kappa M \\
			\leqslant & 4 \kappa^3 M \lambda^{-\frac{7}{4}} \mathcal{S}_{2}(D, \lambda) \mathcal{S}_{5}(D, \lambda) \left\| T_{0}^{\mathbf{x}} - T_{0}^{\hat{\mathbf{x}}} \right\| \left\|  \hat{T}_{0}^{\hat{\mathbf{x}}} - \hat{T}_{0}^{\mathbf{x}} \right\|,
		\end{aligned}
	\]
	and
	\[
		\begin{aligned}
			\mathcal{T}_{4.1.2} \leqslant &  \left\| L_{K_{0}}^{\frac{1}{2}}  A^{-1} \right\| \left\| \hat{T}_{0}^{\mathbf{x}}  - \hat{T}_{0}^{\hat{\mathbf{x}}} \right\| \left\| T_{0}^{\hat{\mathbf{x}}}  \hat{A}^{-1} \right\| \left\|  \hat{T}_{0}^{\hat{\mathbf{x}}} - \hat{T}_{0}^{\mathbf{x}} \right\| \left\|  S_0^{*} \mathbf{y}  \right\|_{K_0} \\
			\leqslant & 2 \lambda^{-\frac{3}{4}} \mathcal{S}_{2}(D, \lambda) \mathcal{S}_{5}(D, \lambda) \left\|  \hat{T}_{0}^{\hat{\mathbf{x}}} - \hat{T}_{0}^{\mathbf{x}} \right\|^2 2 \kappa^2 \lambda^{-1} \cdot \kappa M \\
			\leqslant & 4 \kappa^{3}  M \lambda^{-\frac{7}{4}} \mathcal{S}_{2}(D, \lambda) \mathcal{S}_{5}(D, \lambda) \left\|  \hat{T}_{0}^{\hat{\mathbf{x}}} - \hat{T}_{0}^{\mathbf{x}} \right\|^2.
		\end{aligned}
	\]
	respectively with (\ref{eq:txhAhatinvsBound}), (\ref{eq:49}), (\ref{eq:50}), (\ref{eq:52}) and $ \left\| S_{0}^{*} \mathbf{y} \right\|_{K_0} \leqslant \kappa M $.
	Further,
	\begin{equation}
		\label{eq:T4.1}
		\mathcal{T}_{4.1} \leqslant 4 \kappa^3 M \lambda^{-\frac{7}{4}} \mathcal{S}_{2}(D, \lambda) \mathcal{S}_{5}(D, \lambda) \left\|  \hat{T}_{0}^{\hat{\mathbf{x}}} - \hat{T}_{0}^{\mathbf{x}} \right\| \left( \left\| T_{0}^{\mathbf{x}} - T_{0}^{\hat{\mathbf{x}}} \right\| + \left\|  \hat{T}_{0}^{\hat{\mathbf{x}}} - \hat{T}_{0}^{\mathbf{x}} \right\| \right).
	\end{equation}

	For $ \mathcal{T}_{4.2} $, we obtain from (\ref{eq:52}) that
	\begin{equation}
		\label{eq:T4.2}
		\begin{aligned}
			\mathcal{T}_{4.2} \leqslant & \left\| L_{K_{0}}^{\frac{1}{2}}  A^{-1} \right\| \left\| \hat{T}_{0}^{\mathbf{x}}  - \hat{T}_{0}^{\hat{\mathbf{x}}} \right\|  \left\| S_0^{*} \mathbf{y}  \right\|_{K_0} \\
			\leqslant & 2 \kappa M \lambda^{-\frac{3}{4}} \mathcal{S}_{2}(D, \lambda) \mathcal{S}_{5}(D, \lambda) \left\|  \hat{T}_{0}^{\hat{\mathbf{x}}} - \hat{T}_{0}^{\mathbf{x}} \right\|.
		\end{aligned}
	\end{equation}

	$ \mathcal{T}_{4.3} $ can be decomposed with $ f_{D} $ in (\ref{eq:fDoperator}) as
	\[
		\begin{aligned}
			\mathcal{T}_{4.3} = & \left\| L_{K_{0}}^{\frac{1}{2}} \left( \hat{A}^{-1} - A^{-1} + A^{-1} \right)  \left[ \hat{T}_{0}^{\mathbf{x}} \left( T_{0}^{\mathbf{x}} - T_{0}^{\hat{\mathbf{x}}} \right) + \left( \hat{T}_{0}^{\mathbf{x}}  - \hat{T}_{0}^{\hat{\mathbf{x}}} \right) T_{0}^{\hat{\mathbf{x}}} \right] A^{-1} \hat{T}_{0}^{\mathbf{x}}  S_0^{*} \mathbf{y} \right\|_{K_0} \\
			\leqslant & \left( \mathcal{T}_{4.3.1} +\mathcal{T}_{4.3.2} +\mathcal{T}_{4.3.3} +\mathcal{T}_{4.3.4} \right) \left\| f_{D} \right\|_{K_0},
		\end{aligned}
	\]
	where
	\[
		\begin{aligned}
			\mathcal{T}_{4.3.1} = & \left\| L_{K_{0}}^{\frac{1}{2}} \left( \hat{A}^{-1} - A^{-1} \right) \hat{T}_{0}^{\mathbf{x}} \left( T_{0}^{\mathbf{x}} - T_{0}^{\hat{\mathbf{x}}} \right) \right\|, \\
			\mathcal{T}_{4.3.2} = & \left\| L_{K_{0}}^{\frac{1}{2}} \left( \hat{A}^{-1} - A^{-1} \right) \left( \hat{T}_{0}^{\mathbf{x}}  - \hat{T}_{0}^{\hat{\mathbf{x}}} \right) T_{0}^{\hat{\mathbf{x}}} \right\|, \\
			\mathcal{T}_{4.3.3} = & \left\| L_{K_{0}}^{\frac{1}{2}} A^{-1} \hat{T}_{0}^{\mathbf{x}}  \left( T_{0}^{\mathbf{x}} - T_{0}^{\hat{\mathbf{x}}}  \right)  \right\|, \\
			\mathcal{T}_{4.3.4} = & \left\| L_{K_{0}}^{\frac{1}{2}} A^{-1} \left( \hat{T}_{0}^{\mathbf{x}} - \hat{T}_{0}^{\hat{\mathbf{x}}} \right) T_{0}^{\hat{\mathbf{x}}}  \right\|.
		\end{aligned}
	\]
	We estimate the four terms above respectively in the following.
	First, note that the only difference between $ \mathcal{T}_{4.3.1} $ and $ \mathcal{T}_{3.1} $ is the last term, thus we can bound $ \mathcal{T}_{4.3.1} $ similarly as
	\[
			\mathcal{T}_{4.3.1} \leqslant 4 \kappa^4 \lambda^{-\frac{7}{4}} \mathcal{S}_{2}(D, \lambda) \mathcal{S}_{5}(D, \lambda) \left\| T_{0}^{\mathbf{x}} - T_{0}^{\hat{\mathbf{x}}} \right\| \left( \left\| T_{0}^{\mathbf{x}} - T_{0}^{\hat{\mathbf{x}}} \right\| + \left\| \hat{T}_{0}^{\mathbf{x}}  - \hat{T}_{0}^{\hat{\mathbf{x}}} \right\| \right).
	\]
	$ \mathcal{T}_{4.3.2} $ can be decomposed in a similar way to $ \mathcal{T}_{4.1} $, which implies that
	\[
		\begin{aligned}
			\mathcal{T}_{4.3.2} &\leqslant  \left\| L_{K_{0}}^{\frac{1}{2}} \left( \hat{A}^{-1} - A^{-1} \right) \left( \hat{T}_{0}^{\mathbf{x}}  - \hat{T}_{0}^{\hat{\mathbf{x}}} \right) \right\| \left\| T_{0}^{\hat{\mathbf{x}}} \right\| \\
			& \leqslant 4 \kappa^4 \lambda^{-\frac{7}{4}} \mathcal{S}_{2}(D, \lambda) \mathcal{S}_{5}(D, \lambda) \left\|  \hat{T}_{0}^{\hat{\mathbf{x}}} - \hat{T}_{0}^{\mathbf{x}} \right\| \left( \left\| T_{0}^{\mathbf{x}} - T_{0}^{\hat{\mathbf{x}}} \right\| + \left\|  \hat{T}_{0}^{\hat{\mathbf{x}}} - \hat{T}_{0}^{\mathbf{x}} \right\| \right).
		\end{aligned}
	\]
	Analogously, we can apply (\ref{eq:49}) and (\ref{eq:52}) to obtain
	\[
		\begin{aligned}
			\mathcal{T}_{4.3.3} & \leqslant \left\| L_{K_{0}}^{\frac{1}{2}} A^{-1} \hat{T}_{0}^{\mathbf{x}} \right\| \left\| T_{0}^{\mathbf{x}} - T_{0}^{\hat{\mathbf{x}}} \right\| \\
			& \leqslant 2 \lambda^{-\frac{1}{4}} \mathcal{S}_{2}(D, \lambda) \mathcal{S}_{5}(D, \lambda) \left\| T_{0}^{\mathbf{x}} - T_{0}^{\hat{\mathbf{x}}} \right\|
		\end{aligned}
	\]
	and
	\[
		\begin{aligned}
			\mathcal{T}_{4.3.4} &\leqslant  \left\| L_{K_{0}}^{\frac{1}{2}} A^{-1} \right\| \left\| \hat{T}_{0}^{\mathbf{x}} - \hat{T}_{0}^{\hat{\mathbf{x}}} \right\| \left\| T_{0}^{\hat{\mathbf{x}}}  \right\| \\
			&\leqslant  2 \kappa^2 \lambda^{-\frac{3}{4}} \mathcal{S}_{2}(D, \lambda) \mathcal{S}_{5}(D, \lambda) \left\| \hat{T}_{0}^{\mathbf{x}}  - \hat{T}_{0}^{\hat{\mathbf{x}}} \right\|.
		\end{aligned}
	\]
	Thus $ \mathcal{T}_{4.3} $ can be bounded by combining these estimates as
	\begin{equation}
		\label{eq:T4.3}
		\begin{aligned}
			\mathcal{T}_{4.3} \leqslant & 4 \kappa^{4} \mathcal{S}_{2}(D, \lambda) \mathcal{S}_{5}(D, \lambda) \| f_{D} \|_{K_0} \cdot \\
            & \quad \left[  \lambda^{-\frac{7}{4}} \left( \left\|  \hat{T}_{0}^{\hat{\mathbf{x}}} - \hat{T}_{0}^{\mathbf{x}} \right\| + \left\| T_{0}^{\mathbf{x}} - T_{0}^{\hat{\mathbf{x}}} \right\| \right)^2   + \lambda^{-\frac{1}{4}}  \left\| T_{0}^{\mathbf{x}} - T_{0}^{\hat{\mathbf{x}}} \right\| + \lambda^{-\frac{3}{4}}  \left\|  \hat{T}_{0}^{\hat{\mathbf{x}}} - \hat{T}_{0}^{\mathbf{x}} \right\| \right].
		\end{aligned}
	\end{equation}
	Further, combining (\ref{eq:T4decom}), (\ref{eq:T4.1}), (\ref{eq:T4.2}) with (\ref{eq:T4.3}), we have
	\begin{equation}
		\label{eq:T4}
		\begin{aligned}
			\mathcal{T}_{4} &\leqslant  4 \kappa^{3} (M+ \kappa) \mathcal{S}_{2}(D,\lambda) \mathcal{S}_{5}(D,\lambda) \left[ \lambda^{-\frac{7}{4}} \left( \|  \hat{T}_{0}^{\hat{\mathbf{x}}} - \hat{T}_{0}^{\mathbf{x}} \|  \| T_{0}^{\mathbf{x}} - T_{0}^{\hat{\mathbf{x}}} \| + \|  \hat{T}_{0}^{\hat{\mathbf{x}}} - \hat{T}_{0}^{\mathbf{x}} \|^2 \right)+ \lambda^{-\frac{3}{4}} \|  \hat{T}_{0}^{\hat{\mathbf{x}}} - \hat{T}_{0}^{\mathbf{x}} \|  \right. \\
			& \qquad \left. + \| f_{D} \|_{K_0} \left( \lambda^{-\frac{7}{4}} \left( \left\|  \hat{T}_{0}^{\hat{\mathbf{x}}} - \hat{T}_{0}^{\mathbf{x}} \right\| + \left\| T_{0}^{\mathbf{x}} - T_{0}^{\hat{\mathbf{x}}} \right\| \right)^2   + \lambda^{-\frac{1}{4}}  \left\| T_{0}^{\mathbf{x}} - T_{0}^{\hat{\mathbf{x}}} \right\| + \lambda^{-\frac{3}{4}}  \left\|  \hat{T}_{0}^{\hat{\mathbf{x}}} - \hat{T}_{0}^{\mathbf{x}} \right\| \right) \right] \\
			&\leqslant  4 \kappa^{3} (M+ \kappa) \lambda^{-\frac{3}{4}} \mathcal{S}_{2}(D,\lambda) \mathcal{S}_{5}(D,\lambda) \mathcal{S}_{6}(\hat{D},\lambda) \left[  \|  \hat{T}_{0}^{\hat{\mathbf{x}}} - \hat{T}_{0}^{\mathbf{x}} \| + \| f_{D} \|_{K_0} \left( \left\|  \hat{T}_{0}^{\hat{\mathbf{x}}} - \hat{T}_{0}^{\mathbf{x}} \right\| + \left\| T_{0}^{\mathbf{x}} - T_{0}^{\hat{\mathbf{x}}} \right\| \right)  \right]
		\end{aligned}
	\end{equation}
    Finally substituting (\ref{eq:T3}) and (\ref{eq:T4}) to (\ref{eq:fDhatminusfD}) yields
	\begin{equation}
		\begin{aligned}
			\left\| f_{\hat{D}} - f_{D} \right\|_{\rho_{X_{\mu}}} &\leqslant  4 \kappa^{3} (M+ \kappa) \lambda^{-\frac{3}{4}} \mathcal{S}_{2}(D,\lambda) \mathcal{S}_{5}(D,\lambda) \mathcal{S}_{6}(\hat{D},\lambda) \cdot \\
			& \quad \left[ \left\| \hat{S}_{0}^{*} \mathbf{y} - S_0^{*} \mathbf{y} \right\|_{K_0} + \| \hat{T}_{0}^{\hat{\mathbf{x}}} - \hat{T}_{0}^{\mathbf{x}} \| + \| f_{D} \|_{K_0} \left( \left\|  \hat{T}_{0}^{\hat{\mathbf{x}}} - \hat{T}_{0}^{\mathbf{x}} \right\| + \left\| T_{0}^{\mathbf{x}} - T_{0}^{\hat{\mathbf{x}}} \right\| \right)  \right],
		\end{aligned}
	\end{equation}
	which proves Proposition \ref{prop:indefinitesecond}.

\subsection{Proof of Proposition \ref{prop:positivesecond}}

	Recall that $ \hat{T}^{\mathbf{x}}_{0} = T^{\mathbf{x}}_{0} $ and $ \hat{T}^{ \hat{ \mathbf{x}}}_{0} = T^{ \hat{ \mathbf{x}}}_{0} $ when the kernel $ K $ is positive semi-definite.
	Hence we can make a decomposition for the term $  f_{\hat{D}} - f_{D} $ as
	\begin{equation}
		\label{eq:fDhatminusfDPosi}
			\left\| f_{\hat{D}} - f_{D} \right\|_{\rho_{X_{\mu}}} =  \left\| L_{K_{0}}^{\frac{1}{2}} \left(f_{\hat{D}} - f_{D} \right) \right\|_{K_0} \leqslant \mathcal{T}_{7} + \mathcal{T}_{8},
	\end{equation}
	where
	\[
		\begin{aligned}
			\mathcal{T}_{7} &= \left\| L_{K_{0}}^{\frac{1}{2}} \left( \lambda I + \hat{T}_{0}^{\hat{\mathbf{x}}} \hat{T}_{0}^{\hat{\mathbf{x}}}  \right)^{-1} \hat{T}_{0}^{\hat{\mathbf{x}}} \left( \hat{S}_{0}^{*} \mathbf{y} - S_0^{*} \mathbf{y} \right) \right\|_{K_0}, \\
			\mathcal{T}_{8} &= \left\| L_{K_{0}}^{\frac{1}{2}} \left[ \left( \lambda I + \hat{T}_{0}^{\hat{\mathbf{x}}} \hat{T}_{0}^{\hat{\mathbf{x}}}  \right)^{-1} \hat{T}_{0}^{\hat{\mathbf{x}}} - \left( \lambda I + \hat{T}_{0}^{\mathbf{x}} \hat{T}_{0}^{\mathbf{x}} \right)^{-1} \hat{T}_{0}^{\mathbf{x}} \right]  S_0^{*} \mathbf{y} \right\|_{K_0}.
		\end{aligned}
	\]
	Here we introduce a new decomposition to estimate $ \mathcal{T}_{7} $ and $ \mathcal{T}_{8} $, which is different from that in the case with indefinite kernel and conduces to deriving the bound under regularity condition with $ 0 < r <\frac{1}{2} $ additionally.
	Define
	\begin{equation}
		S_2(\hat{D},\lambda) = \left\| \left( \sqrt{\lambda} I + L_{K_0} \right)^{\frac{1}{2}} \left( \sqrt{\lambda}I+ \hat{T}_{0}^{\hat{\mathbf{x}}} \right)^{-\frac{1}{2}} \right\|.
	\end{equation}
	Then $ \mathcal{T}_{7} $ can be bounded as
	\begin{equation}
		\label{eq:decomtau7}
		\begin{aligned}
			\mathcal{T}_{7} &\leqslant  \left\| \left( \sqrt{\lambda} I + L_{K_0} \right)^{\frac{1}{2}} \left( \sqrt{\lambda} I + \hat{T}_{0}^{\hat{\mathbf{x}}} \right)^{-\frac{1}{2}} \right\| \left\| \left( \sqrt{\lambda} I + \hat{T}_{0}^{\hat{\mathbf{x}}} \right)^{\frac{1}{2}} \left( \lambda I + \hat{T}_{0}^{\hat{\mathbf{x}}} \hat{T}_{0}^{\hat{\mathbf{x}}}  \right)^{-1} \hat{T}_{0}^{\hat{\mathbf{x}}}  \right\| \\
			& \qquad \cdot \left\|  \hat{S}_{0}^{*} \mathbf{y} - S_{0}^{*} \mathbf{y}  \right\|_{K_0} \\
			&\leqslant  2 \lambda^{-\frac{1}{4}} S_2(\hat{D},\lambda) \left\|  \hat{S}_{0}^{*} \mathbf{y} - S_{0}^{*} \mathbf{y}  \right\|_{K_0},
		\end{aligned}
	\end{equation}
	where we have used the fact that
	\[
		\begin{aligned}
			& \left\| \left( \sqrt{\lambda} I + \hat{T}_{0}^{\hat{\mathbf{x}}} \right)^{\frac{1}{2}} \left( \lambda I + \hat{T}_{0}^{\hat{\mathbf{x}}} \hat{T}_{0}^{\hat{\mathbf{x}}}  \right)^{-1} \hat{T}_{0}^{\hat{\mathbf{x}}} \right\| \\
			\leqslant & \left\| \lambda^{\frac{1}{4}} \left( \lambda I + \hat{T}_{0}^{\hat{\mathbf{x}}} \hat{T}_{0}^{\hat{\mathbf{x}}} \right)^{-1}  \hat{T}_{0}^{\hat{\mathbf{x}}} \right\| + \left\| \left( \lambda I + \hat{T}_{0}^{\hat{\mathbf{x}}} \hat{T}_{0}^{\hat{\mathbf{x}}} \right)^{-1}  (\hat{T}_{0}^{\hat{\mathbf{x}}} )^{\frac{3}{2}}   \right\| \leqslant 2 \lambda^{-\frac{1}{4}} .
		\end{aligned}
	\]

	For the second term, there holds
	\begin{equation}
		\begin{aligned}
			\mathcal{T}_{8} &\leqslant  \left\| L_{K_0}^{\frac{1}{2}} \left\{ \left( \lambda I + \hat{T}_{0}^{\hat{\mathbf{x}}} \hat{T}_{0}^{\hat{\mathbf{x}}}  \right)^{-1} \left( \hat{T}_{0}^{\hat{\mathbf{x}}} - \hat{T}_{0}^{\mathbf{x}} \right) + \left[ \left( \lambda I + \hat{T}_{0}^{\hat{\mathbf{x}}} \hat{T}_{0}^{\hat{\mathbf{x}}}  \right)^{-1} - \left( \lambda I + \hat{T}_{0}^{\mathbf{x}} \hat{T}_{0}^{\mathbf{x}} \right)^{-1} \right]  \hat{T}_{0}^{\mathbf{x}} \right\}  S_{0}^{*} \mathbf{y} \right\|_{K_0} \\
			& \leqslant \left\| L_{K_0}^{\frac{1}{2}} \left( \lambda I + \hat{T}_{0}^{\hat{\mathbf{x}}} \hat{T}_{0}^{\hat{\mathbf{x}}}  \right)^{-1} \left( \hat{T}_{0}^{\hat{\mathbf{x}}} - \hat{T}_{0}^{\mathbf{x}} \right) S_{0}^{*} \mathbf{y} \right\|_{K_0} \\
			& \qquad + \left\| L_{K_0}^{\frac{1}{2}} \left( \lambda I + \hat{T}_{0}^{\hat{\mathbf{x}}} \hat{T}_{0}^{\hat{\mathbf{x}}}  \right)^{-1} \left[ \hat{T}_{0}^{\hat{\mathbf{x}}} \left( \hat{T}_{0}^{\mathbf{x}} - \hat{T}_{0}^{\hat{\mathbf{x}}} \right) + \left( \hat{T}_{0}^{\mathbf{x}} -\hat{T}_{0}^{\hat{\mathbf{x}}} \right) \hat{T}_{0}^{\mathbf{x}}  \right] \left( \lambda I + \hat{T}_{0}^{\mathbf{x}} \hat{T}_{0}^{\mathbf{x}} \right)^{-1}  \hat{T}_{0}^{\mathbf{x}} S_{0}^{*} \mathbf{y} \right\|_{K_0} \\
			& \leqslant \left\| L_{K_0}^{\frac{1}{2}} \left( \lambda I + \hat{T}_{0}^{\hat{\mathbf{x}}} \hat{T}_{0}^{\hat{\mathbf{x}}}  \right)^{-1} \left( \hat{T}_{0}^{\hat{\mathbf{x}}} - \hat{T}_{0}^{\mathbf{x}} \right) S_{0}^{*} \mathbf{y} \right\|_{K_0} \\
			& \qquad + \left\| L_{K_0}^{\frac{1}{2}} \left( \lambda I + \hat{T}_{0}^{\hat{\mathbf{x}}} \hat{T}_{0}^{\hat{\mathbf{x}}}  \right)^{-1} \hat{T}_{0}^{\hat{\mathbf{x}}} \left( \hat{T}_{0}^{\mathbf{x}} - \hat{T}_{0}^{\hat{\mathbf{x}}} \right) \left( \lambda I + \hat{T}_{0}^{\mathbf{x}} \hat{T}_{0}^{\mathbf{x}} \right)^{-1}  \hat{T}_{0}^{\mathbf{x}} S_{0}^{*} \mathbf{y} \right\|_{K_0} \\
			& \qquad \qquad + \left\| L_{K_0}^{\frac{1}{2}} \left( \lambda I + \hat{T}_{0}^{\hat{\mathbf{x}}} \hat{T}_{0}^{\hat{\mathbf{x}}}  \right)^{-1}\left( \hat{T}_{0}^{\mathbf{x}} -\hat{T}_{0}^{\hat{\mathbf{x}}} \right) \hat{T}_{0}^{\mathbf{x}} \left( \lambda I + \hat{T}_{0}^{\mathbf{x}} \hat{T}_{0}^{\mathbf{x}} \right)^{-1}  \hat{T}_{0}^{\mathbf{x}} S_{0}^{*} \mathbf{y} \right\|_{K_0},
		\end{aligned}
	\end{equation}
	where we use the identity $ A^{-1} - B^{-1} = A^{-1} (B-A) B^{-1} $ with $ A = \lambda I + \hat{T}_{0}^{\hat{\mathbf{x}}} \hat{T}_{0}^{\hat{\mathbf{x}}} $ and $ B = \lambda I + \hat{T}_{0}^{\mathbf{x}} \hat{T}_{0}^{\mathbf{x}} $ in the second inequality.
	Finally $ \mathcal{T}_{8} $ can be bounded as
	\begin{equation}
		\label{eq:decomtau8}
		\begin{aligned}
			\mathcal{T}_{8} &\leqslant   S_2(\hat{D},\lambda) \left\| \hat{T}_{0}^{\hat{\mathbf{x}}} - \hat{T}_{0}^{\mathbf{x}}  \right\| \left\| S_{0}^{*} \mathbf{y} \right\|_{K_0} \left( 2 \left\| \left( \sqrt{\lambda} I + \hat{T}_{0}^{\hat{\mathbf{x}}} \right)^{\frac{1}{2}} \left( \lambda I + \hat{T}_{0}^{\hat{\mathbf{x}}} \hat{T}_{0}^{\hat{\mathbf{x}}}  \right)^{-1} \right\|  \right. \\
			& \qquad + \left. \left\| \left( \sqrt{\lambda} I + \hat{T}_{0}^{\hat{\mathbf{x}}} \right)^{\frac{1}{2}} \left( \lambda I + \hat{T}_{0}^{\hat{\mathbf{x}}} \hat{T}_{0}^{\hat{\mathbf{x}}} \right)^{-1} \hat{T}_{0}^{\hat{\mathbf{x}}}  \right\| \left\| \left( \lambda I + \hat{T}_{0}^{\mathbf{x}} \hat{T}_{0}^{\mathbf{x}} \right)^{-1} \hat{T}_{0}^{\mathbf{x}} \right\|   \right) \\
			&\leqslant  6 \kappa M \lambda^{-\frac{3}{4}}  S_2(\hat{D},\lambda) \left\| \hat{T}_{0}^{\hat{\mathbf{x}}} - \hat{T}_{0}^{\mathbf{x}}  \right\|,
		\end{aligned}
	\end{equation}
	with the fact that $ \| L_{K_0}^{\frac{1}{2}} ( \sqrt{\lambda} I + L_{K_0} )^{-\frac{1}{2}} \| \leqslant 1  $ and
	\[
		\left\| \left( \sqrt{\lambda} I + \hat{T}_{0}^{\hat{\mathbf{x}}} \right)^{\frac{1}{2}} \left( \lambda I + \hat{T}_{0}^{\hat{\mathbf{x}}} \hat{T}_{0}^{\hat{\mathbf{x}}}  \right)^{-1} \right\|\leqslant \left\| \lambda^{\frac{1}{4}}  \left( \lambda I + \hat{T}_{0}^{\hat{\mathbf{x}}} \hat{T}_{0}^{\hat{\mathbf{x}}}   \right)^{-1} \right\| + \left\| (\hat{T}_{0}^{\hat{\mathbf{x}}})^{\frac{1}{2}} \left( \lambda I + \hat{T}_{0}^{\hat{\mathbf{x}}} \hat{T}_{0}^{\hat{\mathbf{x}}} \right)^{-1} \right\| \leqslant 2 \lambda^{-\frac{3}{4}},
	\]
	\[
		\left\| \left( \sqrt{\lambda} I + \hat{T}_{0}^{\hat{\mathbf{x}}} \right)^{\frac{1}{2}} \left( \lambda I + \hat{T}_{0}^{\hat{\mathbf{x}}} \hat{T}_{0}^{\hat{\mathbf{x}}} \right)^{-1} \hat{T}_{0}^{\hat{\mathbf{x}}} \right\| \leqslant  \left\| \lambda^{\frac{1}{4}} \left( \lambda I + \hat{T}_{0}^{\hat{\mathbf{x}}} \hat{T}_{0}^{\hat{\mathbf{x}}} \right)^{-1}  \hat{T}_{0}^{\hat{\mathbf{x}}} \right\| + \left\| \left( \lambda I + \hat{T}_{0}^{\hat{\mathbf{x}}} \hat{T}_{0}^{\hat{\mathbf{x}}} \right)^{-1}  (\hat{T}_{0}^{\hat{\mathbf{x}}} )^{\frac{3}{2}}   \right\| \leqslant 2 \lambda^{-\frac{1}{4}},
	\]
	\[
		\left\| \hat{T}_{0}^{\mathbf{x}}  \left( \lambda I + \hat{T}_{0}^{\mathbf{x}} \hat{T}_{0}^{\mathbf{x}} \right)^{-1} \hat{T}_{0}^{\mathbf{x}}  \right\| \leqslant 1 , \qquad \left\| \left( \lambda I + \hat{T}_{0}^{\mathbf{x}} \hat{T}_{0}^{\mathbf{x}} \right)^{-1} \hat{T}_{0}^{\mathbf{x}}  \right\| \leqslant  \lambda^{-\frac{1}{2}} , \qquad \left\| S_{0}^{*} \mathbf{y} \right\|_{K_0} \leqslant \kappa M.
	\]
	We complete the proof of Proposition \ref{prop:positivesecond} by substituting (\ref{eq:decomtau7}) and (\ref{eq:decomtau8}) to (\ref{eq:fDhatminusfDPosi}).

\subsection{Proof of Proposition \ref{prop:boundK0}}

	First we can decompose $ \| f_{D} \|_{K_0} $ as
	\begin{equation}
		\label{eq:decomfDKnorm}
		\| f_{D} \|_{K_0} \leqslant \| f_{D} - f_{\lambda} \|_{K_0} + \| f_{\lambda} - f_{\rho} \|_{K_0} + \| f_{\rho} \|_{K_0}.
	\end{equation}
	The last two terms in (\ref{eq:decomfDKnorm}) can be bounded by Lemma \ref{lem:approxerror} and Assumption \ref{assum:regular} with $ r \geqslant \frac{1}{2} $, which asserts that $ \| f_{\lambda} - f_{\rho} \|_{K_0} \leqslant c_{r}^\prime \lambda^{\min \{ 1,\frac{2r-1}{4} \} } $ and $ \| f_{\rho} \|_{K_0}\leqslant \kappa^{2r-1} \| g_{\rho} \|_{\rho_{X_{\mu}}} $.
    Putting the decomposition (\ref{eq:decomfDminusflambda}) into the first term in (\ref{eq:decomfDKnorm}), we have
	\begin{equation}
		\| f_{D} - f_{\lambda} \|_{K_0} \leqslant \mathcal{T}_{5} + \mathcal{T}_{6},
	\end{equation}
	where
	\[
		\mathcal{T}_{5}= \left\|  \left(\lambda I+  \hat{T}_{0}^{\mathbf{x}} T_{0}^{\mathbf{x}} \right)^{-1}  \hat{T}_{0}^{\mathbf{x}} \left( S_{0}^{*} \mathbf{y} - T_{0}^{\mathbf{x}} f_{\lambda} \right) \right\|_{K_{0}},
	\]
	\[
		\mathcal{T}_{6} = \left\| \lambda \left( \lambda I + \hat{T}_{0}^{\mathbf{x}} T_{0}^{\mathbf{x}}  \right) ^{-1}  f_{\lambda} \right\|_{K_{0}}.
	\]
	For $ \mathcal{T}_{5} $, we can use the bound  $ \left\| (\sqrt{\lambda} I + \hat{T}_{0}^{\mathbf{x}} )^{-\frac{1}{2}} \right\| \leqslant \lambda^{-\frac{1}{4}} $ to obtain that
	\begin{equation}
		\begin{aligned}
			\mathcal{T}_{5} \leq & \left\| (\sqrt{\lambda} I + \hat{T}_{0}^{\mathbf{x}} )^{-\frac{1}{2}} \right\| \left\| (\sqrt{\lambda} I + \hat{T}_{0}^{\mathbf{x}} )^{\frac{1}{2}} \left(\lambda I+  \hat{T}_{0}^{\mathbf{x}} T_{0}^{\mathbf{x}} \right)^{-1}  \hat{T}_{0}^{\mathbf{x}} (\sqrt{\lambda} I + \hat{T}_{0}^{\mathbf{x}} )^{\frac{1}{2}}  \right\| 	\\
			& \left\| (\sqrt{\lambda} I + \hat{T}_{0}^{\mathbf{x}} )^{-\frac{1}{2}} (\sqrt{\lambda}I + L_{K_{0}} )^{\frac{1}{2}}  \right\| \left\| (\sqrt{\lambda}I + L_{K_{0}} )^{-\frac{1}{2}}  \left( S_{0}^{*} \mathbf{y} - T_{0}^{\mathbf{x}} f_{\lambda} \right)  \right\|_{K_{0}} \\
			\leqslant  & \lambda^{-\frac{1}{4}} \mathcal{S}_{1}(D,\lambda) \mathcal{S}_{2}(D,\lambda) \mathcal{S}_{3}(D,\lambda) \leqslant 2 \lambda^{-\frac{1}{4}}  \mathcal{S}_{1}(D,\lambda) \mathcal{S}_{2}(D,\lambda) \mathcal{S}_{5}(D,\lambda).
		\end{aligned}
	\end{equation}
    In the same way, $ \mathcal{T}_{6} $ can be bounded as
	\begin{equation}
		\begin{aligned}
			\mathcal{T}_{6} \leqslant & \left\| (\sqrt{\lambda} I + \hat{T}_{0}^{\mathbf{x}} )^{-\frac{1}{2}} \right\| \left\| \lambda(\sqrt{\lambda} I + \hat{T}_{0}^{\mathbf{x}} )^{\frac{1}{2}}  \left( \lambda I + \hat{T}_{0}^{\mathbf{x}} T_{0}^{\mathbf{x}}  \right) ^{-1}  f_{\lambda} \right\|_{K_{0}} \\
			\leqslant & \lambda^{\frac{3}{4}} \left\| (\sqrt{\lambda} I + \hat{T}_{0}^{\mathbf{x}} )^{\frac{1}{2}}  \left( \lambda I + \hat{T}_{0}^{\mathbf{x}} T_{0}^{\mathbf{x}} \right) ^{-1}  L_{K_{0}}^{r-\frac{1}{2}} L_{K_{0}}^{\frac{1}{2}} (\lambda I + L_{K_{0}}^{2} )^{-1} L_{K_{0}}^{2} g_{\rho} \right\|_{K_{0}} \\
			\leqslant & \lambda^{\frac{3}{4}} \left\| g_{\rho} \right\|_{\rho_{X_{\mu}}}  \left\| (\sqrt{\lambda} I + \hat{T}_{0}^{\mathbf{x}} )^{\frac{1}{2}}  \left( \lambda I + \hat{T}_{0}^{\mathbf{x}} T_{0}^{\mathbf{x}} \right) ^{-1}  L_{K_{0}}^{r-\frac{1}{2}} \right\| \\
			\leqslant & \lambda^{\frac{3}{4}} \left\| g_{\rho} \right\|_{\rho_{X_{\mu}}} \mathcal{S}_{5}(D,\lambda) \left\| (\sqrt{\lambda} I + \hat{T}_{0}^{\mathbf{x}} )^{\frac{1}{2}}  \left( \lambda I + \hat{T}_{0}^{\mathbf{x}} \hat{T}_{0}^{\mathbf{x}} \right) ^{-1}  L_{K_{0}}^{r-\frac{1}{2}} \right\| \\
			= &  \lambda^{\frac{3}{4}} \left\| g_{\rho} \right\|_{\rho_{X_{\mu}}} \mathcal{S}^\prime_{4}(D,\lambda) \mathcal{S}_{5}(D,\lambda).
		\end{aligned}
	\end{equation}
    Note that the last equality follows from the definition of $ \mathcal{S}^\prime_{4}(D,\lambda) $ in (\ref{eq:36}) with $ r \geqslant \frac{1}{2} $.
	Finally, combining the estimates above with (\ref{eq:decomfDKnorm}) yields (\ref{eq:fDK0bound}), which completes the proof of Proposition \ref{prop:boundK0}.

\subsection{Proof of Proposition \ref{prop:S2Dhat}}

    First $ \| ( \sqrt{\lambda} I + L_{K_0} )^{-\frac{1}{2}} ( \hat{T}_{0}^{\hat{\mathbf{x}}}  - L_{K_0} ) \| $ can be decomposed as
	\[
		\begin{aligned}
			&\left\| \left( \sqrt{\lambda} I + L_{K_0} \right)^{-\frac{1}{2}} \left( \hat{T}_{0}^{\hat{\mathbf{x}}} - L_{K_0} \right) \right\| \\
            \leqslant & \left\| \left( \sqrt{\lambda} I + L_{K_0}  \right)^{-\frac{1}{2}} \left( \hat{T}_{0}^{\hat{\mathbf{x}}} - \hat{T}_{0}^{\mathbf{x}}  \right) \right\| + \left\| \left( \sqrt{\lambda} I + L_{K_0} \right)^{-\frac{1}{2}} \left( \hat{T}_{0}^{\mathbf{x}} - L_{K_0} \right) \right\| \\
			\leqslant & \lambda^{-\frac{1}{4}} \left\| \hat{T}_{0}^{\hat{\mathbf{x}}} - \hat{T}_{0}^{\mathbf{x}} \right\| + \left\| \left( \sqrt{\lambda} I + L_{K_{1}} \right)^{-\frac{1}{2}} \left( T_{1}^{\mathbf{x}} - L_{K_{1}} \right) \right\|,
		\end{aligned}
	\]
	where the last inequality holds with the facts that $ \| ( \sqrt{\lambda} I + L_{K_0}  )^{-\frac{1}{2}}  \| \leqslant \lambda^{-\frac{1}{4}} $, $ \hat{T}_{0}^{\mathbf{x}} = U T_1^{\mathbf{x}} U^{*} $ and $ L_{K_{0}} = U L_{K_1} U^{*} $.
	Combining the results in Lemma \ref{lem: T0X-LK0} and Lemma \ref{lem:TxhatMinusTx}, we obtain the bound of $ \| ( \sqrt{\lambda} I + L_{K_0} )^{-\frac{1}{2}} ( \hat{T}_{0}^{\hat{\mathbf{x}}} - L_{K_0} ) \| $.

	Applying the fact that for positive operators $ A $ and $ B $ on a Hilbert space and $ s \in [0,1] $, $ \| A^{s} B^{s}  \| \leqslant \| AB \|^{s}  $ with $ s=\frac{1}{2} $, we have
    \begin{equation}
        \label{eq:S2Dhatprod}
        \left\| \left( \sqrt{\lambda} I + L_{K_0} \right)^{\frac{1}{2}} \left( \sqrt{\lambda}I+ \hat{T}_{0}^{\hat{\mathbf{x}}} \right)^{-\frac{1}{2}} \right\| \leqslant \left\| \left( \sqrt{\lambda} I + L_{K_0} \right) \left( \sqrt{\lambda}I+ \hat{T}_{0}^{\hat{\mathbf{x}}} \right)^{-1} \right\|^{\frac{1}{2}}.
    \end{equation}
    Recall that the second order decomposition on the difference of operator inverses (\ref{Second order}) implies that
    \begin{equation}
        \label{eq:ProdSecondorder}
        BA^{-1} = (B-A) B^{-1} (B-A) A^{-1} + (B-A) B^{-1} + I.
    \end{equation}
	Then by substituting $ B = \sqrt{\lambda} I+L_{K_0} $ and $ A = \sqrt{\lambda} I+\hat{T}_{0}^{\hat{\mathbf{x}}} $, we have
    \[
		\begin{aligned}
			& \left\| \left( \sqrt{\lambda} I+L_{K_0} \right) \left( \sqrt{\lambda} I+\hat{T}_{0}^{\hat{\mathbf{x}}} \right)^{-1} \right\| \leqslant  1 + \left\| \left( L_{K_0}-\hat{T}_{0}^{\hat{\mathbf{x}}} \right) \left(\sqrt{\lambda} I+L_{K_0}\right)^{-1} \right\| \\
            & \qquad \qquad \qquad + \left\| \left( L_{K_0}-\hat{T}_{0}^{\hat{\mathbf{x}}} \right) \left(\sqrt{\lambda} I+L_{K_0}\right)^{-1}\left( L_{K_0}-\hat{T}_{0}^{\hat{\mathbf{x}}} \right)  \left( \sqrt{\lambda} I+\hat{T}_{0}^{\hat{\mathbf{x}}} \right)^{-1} \right\| \\
			\leqslant & 1 + \lambda^{-\frac{1}{4}} \left\| \left(\sqrt{\lambda} I+L_{K_0}\right)^{-\frac{1}{2}} \left( L_{K_0}-\hat{T}_{0}^{\hat{\mathbf{x}}} \right)  \right\| + \lambda^{-\frac{1}{2}} \left\| \left( \sqrt{\lambda} I+L_{K_0} \right)^{-\frac{1}{2}} \left( L_{K_0}-\hat{T}_{0}^{\mathrm{x}} \right) \right\|^{2}  \\
			= & \left[ 1 + \lambda^{-\frac{1}{4}} \left\| \left(\sqrt{\lambda} I+L_{K_0}\right)^{-\frac{1}{2}} \left( L_{K_0}-\hat{T}_{0}^{\hat{\mathbf{x}}} \right)  \right\| \right]^2.
		\end{aligned}
	\]
    Finally, applying (\ref{eq:S2Dhat}) and putting the above estimate back into (\ref{eq:S2Dhatprod}) yields (\ref{eq:S2DhatprodBound}) in Proposition \ref{prop:S2Dhat}. This completes the proof.

\subsection{Proof of Lemma \ref{lem:fDMinusflambda}}

    With the definition of $ f_{D} $ in (\ref{eq:fDoperator}), we have the following decomposition
	\begin{equation}
		\label{eq:decomfDminusflambda}
		\begin{aligned}
			f_{D}- f_{\lambda} =& \left(\lambda I+  \hat{T}_{0}^{\mathbf{x}} T_{0}^{\mathbf{x}} \right)^{-1} \hat{T}_{0}^{\mathbf{x}}  S_{0}^{*} \mathbf{y} - f_{\lambda} \\
			=& \left(\lambda I+  \hat{T}_{0}^{\mathbf{x}} T_{0}^{\mathbf{x}} \right)^{-1} \left[ \hat{T}_{0}^{\mathbf{x}} \left(   S_{0}^{*} \mathbf{y}- T_{0}^{\mathbf{x}} f_{\lambda}  \right) - \lambda f_{\lambda} \right] \\
			=&  \left(\lambda I+  \hat{T}_{0}^{\mathbf{x}} T_{0}^{\mathbf{x}} \right)^{-1} \hat{T}_{0}^{\mathbf{x}} \left(  S_{0}^{*} \mathbf{y} - T_{0}^{\mathbf{x}} f_{\lambda} \right) - \lambda \left( \lambda I + \hat{T}_{0}^{\mathbf{x}} T_{0}^{\mathbf{x}}  \right) ^{-1}  f_{\lambda}.
		\end{aligned}
	\end{equation}
	Note that $ \| g \|_{\rho_{X_{\mu}}} = \| L_{K_0}^{\frac{1}{2}} g \| _{K_0}$ for any $ g \in L_{\rho_{X_{\mu}}}^2 $.
	Then there holds
	\begin{equation}
		\label{eq:fDMinusflambdaRhonorm}
		\| f_{D} - f_{\lambda} \|_{\rho_{X_{\mu}}} = \| L_{K_{0}}^{\frac{1}{2}} (f_{D}-f_{\lambda}) \|_{K_{0}} \leqslant \mathcal{T}_{1} + \mathcal{T}_{2},
	\end{equation}
	where
	\[
		\mathcal{T}_{1}= \left\|  L_{K_{0}}^{\frac{1}{2}} \left(\lambda I+  \hat{T}_{0}^{\mathbf{x}} T_{0}^{\mathbf{x}} \right)^{-1}  \hat{T}_{0}^{\mathbf{x}} \left( S_{0}^{*} \mathbf{y} - T_{0}^{\mathbf{x}} f_{\lambda} \right) \right\|_{K_{0}},
	\]
	\[
		\mathcal{T}_{2} = \left\| \lambda L_{K_{0}}^{\frac{1}{2}}  \left( \lambda I + \hat{T}_{0}^{\mathbf{x}} T_{0}^{\mathbf{x}}  \right) ^{-1}  f_{\lambda} \right\|_{K_{0}}.
	\]

	For $ \mathcal{T}_{1},$ we have
	\begin{equation}
		\label{eq:boundT1}
		\begin{aligned}
			\mathcal{T}_{1} \leq & \left\| (\sqrt{\lambda} I + L_{K_{0}})^{\frac{1}{2}} \left(\lambda I+  \hat{T}_{0}^{\mathbf{x}} T_{0}^{\mathbf{x}} \right)^{-1}  \hat{T}_{0}^{\mathbf{x}} \left(  S_{0}^{*} \mathbf{y} - T_{0}^{\mathbf{x}} f_{\lambda} \right) \right\|_{K_{0}}	\\
			\leq & \left\| (\sqrt{\lambda} I + L_{K_{0}})^{\frac{1}{2}} (\sqrt{\lambda} I + \hat{T}_{0}^{\mathbf{x}} )^{-\frac{1}{2}} \right\| \left\| (\sqrt{\lambda} I + \hat{T}_{0}^{\mathbf{x}} )^{\frac{1}{2}} \left(\lambda I+  \hat{T}_{0}^{\mathbf{x}} T_{0}^{\mathbf{x}} \right)^{-1}  \hat{T}_{0}^{\mathbf{x}} (\sqrt{\lambda} I + \hat{T}_{0}^{\mathbf{x}} )^{\frac{1}{2}}  \right\| 	\\
			& \left\| (\sqrt{\lambda} I + \hat{T}_{0}^{\mathbf{x}} )^{-\frac{1}{2}} (\sqrt{\lambda}I + L_{K_{0}} )^{\frac{1}{2}}  \right\| \left\| (\sqrt{\lambda}I + L_{K_{0}} )^{-\frac{1}{2}}  \left(  S_{0}^{*} \mathbf{y} - T_{0}^{\mathbf{x}} f_{\lambda} \right)  \right\|_{K_{0}} \\
			= & \left\| (\sqrt{\lambda}I + L_{K_{0}} )^{-\frac{1}{2}}  \left(  S_{0}^{*} \mathbf{y} - T_{0}^{\mathbf{x}} f_{\lambda} \right) \right\|_{K_{0}} \left\| (\sqrt{\lambda} I + L_{K_{0}})^{\frac{1}{2}} (\sqrt{\lambda} I + \hat{T}_{0}^{\mathbf{x}} )^{-\frac{1}{2}} \right\|^{2} \\
			& \left\| (\sqrt{\lambda} I + \hat{T}_{0}^{\mathbf{x}} )^{\frac{1}{2}} \left(\lambda I+  \hat{T}_{0}^{\mathbf{x}} T_{0}^{\mathbf{x}} \right)^{-1}  \hat{T}_{0}^{\mathbf{x}} (\sqrt{\lambda} I + \hat{T}_{0}^{\mathbf{x}} )^{\frac{1}{2}}  \right\|
		\end{aligned}
	\end{equation}
 	with the facts that $ \| L_{K_{0}}^{\frac{1}{2}} (\sqrt{\lambda} I + L_{K_{0}} )^{-\frac{1}{2}} \| \leqslant 1 $ for any $ \lambda>0 $ and $ \| AB \| = \| BA \| $ for any self-adjoint operators $ A $, $ B $ on a Hilbert space.
	
	$ \mathcal{T}_{2} $ can be decomposed in the same way as
	\begin{equation}
		\begin{aligned}
			\mathcal{T}_{2} \leq \left\| (\sqrt{\lambda} I + L_{K_{0}})^{\frac{1}{2}} (\sqrt{\lambda} I + \hat{T}_{0}^{\mathbf{x}} )^{-\frac{1}{2}} \right\| \left\| \lambda(\sqrt{\lambda} I + \hat{T}_{0}^{\mathbf{x}} )^{\frac{1}{2}}  \left( \lambda I + \hat{T}_{0}^{\mathbf{x}} T_{0}^{\mathbf{x}}  \right) ^{-1}  f_{\lambda} \right\|_{K_{0}}.
		\end{aligned}
	\end{equation}
	Recall that $ f_{\lambda} =  \left( \lambda I + L_{K_{0}}^{2} \right)^{-1} L_{K_{0}}^{2+r} g_{\rho} $. If $ 0 < r <\frac{1}{2} $, there holds
	\begin{equation}
		\| f_{\lambda} \|_{K_{0}} =  \left\| \left( \lambda I + L_{K_{0}}^{2} \right)^{-1} L_{K_{0}}^{2+r} g_{\rho} \right\|_{K_{0}} =  \left\| \left( \lambda I + L_{K_{0}}^{2} \right)^{-1} L_{K_{0}}^{\frac{3}{2}+r} L_{K_{0}}^{\frac{1}{2}}  g_{\rho} \right\|_{K_{0}} \leq \left\| g_{\rho} \right\|_{\rho_{X_{\mu}}} \lambda^{\frac{r}{2}-\frac{1}{4}},
	\end{equation}
	where the last inequality holds due to (\ref{eq:normLK0LessR}).
	If $ r \geqslant \frac{1}{2} $, we have
	\[
		\begin{aligned}
			& \left\| \lambda(\sqrt{\lambda} I + \hat{T}_{0}^{\mathbf{x}} )^{\frac{1}{2}}  \left( \lambda I + \hat{T}_{0}^{\mathbf{x}} T_{0}^{\mathbf{x}}  \right) ^{-1}  f_{\lambda} \right\|_{K_{0}}	\\
			=& \left\| \lambda(\sqrt{\lambda} I + \hat{T}_{0}^{\mathbf{x}} )^{\frac{1}{2}}  \left( \lambda I + \hat{T}_{0}^{\mathbf{x}} T_{0}^{\mathbf{x}} \right) ^{-1}  L_{K_{0}}^{r-\frac{1}{2}} L_{K_{0}}^{\frac{1}{2}} (\lambda I + L_{K_{0}}^{2} )^{-1} L_{K_{0}}^{2} g_{\rho}    \right\|_{K_{0}}	\\
			\leqslant& \left\| (\sqrt{\lambda} I + \hat{T}_{0}^{\mathbf{x}} )^{\frac{1}{2}}  \left( \lambda I + \hat{T}_{0}^{\mathbf{x}} T_{0}^{\mathbf{x}} \right) ^{-1} L_{K_{0}}^{r-\frac{1}{2}}  \right\| \left\| g_{\rho} \right\|_{\rho_{X_{\mu}}} \lambda.
		\end{aligned}
	\]
	Therefore, combining the results above yields
	\begin{equation}
		\label{eq:boundT2}
		\mathcal{T}_{2} \leqslant \left\| (\sqrt{\lambda} I + L_{K_{0}})^{\frac{1}{2}} (\sqrt{\lambda} I + \hat{T}_{0}^{\mathbf{x}} )^{-\frac{1}{2}} \right\| \left\| (\sqrt{\lambda} I + \hat{T}_{0}^{\mathbf{x}} )^{\frac{1}{2}}  \left( \lambda I + \hat{T}_{0}^{\mathbf{x}} T_{0}^{\mathbf{x}} \right) ^{-1} L_{K_{0}}^{r_{1}}  \right\| \| g_{\rho} \|_{\rho_{X_{\mu}}} \lambda^{r_2},
	\end{equation}
	where $ r_1 = \max \{ 0, r-\frac{1}{2} \} $ and $ r_2 = \min \{1,\frac{r}{2}+\frac{3}{4}\} $.

	Finally, we complete the proof of Lemma \ref{lem:fDMinusflambda} by substituting (\ref{eq:boundT1}) and (\ref{eq:boundT2}) to (\ref{eq:fDMinusflambdaRhonorm}).

\subsection{Proof of Lemma \ref{lem:S1}}

	Recall the definition of $ \mathcal{S}_{1} (D,\lambda) $ in Lemma \ref{lem:fDMinusflambda}.
	We first bound $ \mathcal{S}_{1} (D,\lambda) $ by the following three terms
	\[
		\mathcal{S}_{1}(D,\lambda) \leqslant \mathcal{S}_{1}^{\mathrm{I}} (D,\lambda) + \mathcal{S}_{1}^{\mathrm{II}} (D,\lambda) + \mathcal{S}_{1}^{\mathrm{III}} (D,\lambda),
	\]
	where
	\[
		\begin{aligned}
			\mathcal{S}_{1}^{\mathrm{I}} (D,\lambda) =& \left\| \left( \sqrt{\lambda} I + L_{K_{0}} \right)^{-\frac{1}{2}} \left( S_{0}^{*} \mathbf{y} - L_{K_{0}} f_{\rho} \right)  \right\|_{K_{0}},	\\
			\mathcal{S}_{1}^{\mathrm{II}} (\lambda) =& \left\| \left( \sqrt{\lambda}I + L_{K_{0}} \right)^{-\frac{1}{2}} L_{K_{0}} \left( f_{\rho} - f_{\lambda} \right)  \right\|_{K_{0}},	\\
			\mathcal{S}_{1}^{\mathrm{III}} (D,\lambda) =& \left\| \left( \sqrt{\lambda}I + L_{K_{0}} \right)^{-\frac{1}{2}} \left( L_{K_{0}} f_{\lambda} - T_{0}^{\mathbf{x}} f_{\lambda}  \right)  \right\|_{K_{0}}.
		\end{aligned}
	\]

	For the first term $ \mathcal{S}_{1}^{\mathrm{I}}(D,\lambda)  $, we define the random variable $ \xi_{1}^{\mathrm{I}}  $ as
	\[
		\xi_{1}^{\mathrm{I}}(z) = \left( \sqrt{\lambda} I + L_{K_{0}}  \right) ^{-\frac{1}{2}} y K_{0}(\mu_{x},\cdot), \quad z = (\mu_{x},y) \in Z.
	\]
	which takes values in $ \mathcal{H}_{K_{0}} $ with expectation
	\[
		\mathbb{E} \xi_{1}^{\mathrm{I}}  = \int_{X_{\mu}} \left( \sqrt{\lambda} I + L_{K_{0}} \right)^{-\frac{1}{2}} K_{0}(\mu_{x},\cdot) \int_{Y} y \mathrm{d} \rho(y|\mu_{x}) \mathrm{d} \rho_{X_{\mu}}(\mu_{x}) =   \left( \sqrt{\lambda}I + L_{K_{0}} \right)^{-\frac{1}{2}} L_{K_{0}} f_{\rho}.
	\]

	Then $ \mathcal{S}_{1}^{\mathrm{I}}(D,\lambda) = \| \frac{1}{m} \sum_{i=1}^{m} \xi_{1}^{\mathrm{I}}(z_{i}) - \mathbb{E} \xi_{1}^{\mathrm{I}} \|_{K_{0}}  $ can be bounded by the following concentration inequality \cite{pinelisRemarksInequalitiesLarge1986}.

	\begin{lemma}
		\label{lem:concentration}
		Let $(\Omega, \mathcal{F}, P)$ be a probability space and $\xi$ be a random variable on $\Omega$ taking value in a separable Hilbert space $\mathcal{H}$. Assume that $\mu=\mathbb{E} \xi$ exists and $\xi$ is uniformly bounded which satisfies $\|\xi(\omega)\|_{\mathcal{H}} \leq \widetilde{M}, \forall \omega \in$ $\Omega$ and $\mathbb{E}\|\xi\|_{\mathcal{H}}^{2} \leq \sigma^{2}$,
		Then for random samples $\left\{\omega_{i}\right\}_{i=1}^{m}$ independent drawn according to $P$, there holds with probability at least $1-\delta$,
		\[
			\left\|\frac{1}{m} \sum_{i=1}^{m} \xi\left(\omega_{i}\right)-\mathbb{E} \xi\right\|_{\mathcal{H}} \leq \frac{2 \widetilde{M} \log (2 / \delta)}{m}+\sqrt{\frac{2 \sigma^{2} \log (2 / \delta)}{m}} .
		\]	
	\end{lemma}
	
	To this end, we need to estimate $  \| \xi_{1}^{\mathrm{I}}  \|_{K_{0}}^{2} $ denoted by
	\[
		\| \xi_{1}^{\mathrm{I}}  \|_{K_{0}}^{2} = y^2 \left\| \left( \sqrt{\lambda} I + L_{K_{0}}  \right)^{-\frac{1}{2}} K_{0}(\mu_{x},\cdot)  \right\|_{K_{0}}^{2}.
	\]
	Since $ K_{0}(\mu_{x},\cdot) = \sum_{i \geqslant 1} \sigma_{i} \phi_{i}(\mu_{x}) \phi_{i}(\cdot) $ and $ \{ \sqrt{\sigma_{i}} \phi_{i}  \}_{\sigma>0} $ forms an orthonormal basis of $ \mathcal{H}_{K_{0}} $, we have
	\[
		\left( \sqrt{\lambda} I + L_{K_{0}}  \right)^{-\frac{1}{2}} K_{0}(\mu_{x},\cdot) = \sum_{i \geqslant 1} \left( \sqrt{\lambda} + \sigma_{i}  \right)^{-\frac{1}{2}} \sigma_{i} \phi_{i}(\mu_{x}) \phi_{i}(\cdot)
	\]
	and
	\[
		\left\| \left( \sqrt{\lambda} I + L_{K_{0}}  \right)^{-\frac{1}{2}} K_{0}(\mu_{x},\cdot)  \right\|_{K_{0}}^{2} = \sum_{i \geqslant 1} \left( \sqrt{\lambda} + \sigma_{i} \right)^{-1} \sigma_{i} \phi_{i}^{2} (\mu_{x}).
	\]
	Therefore,
	\[
		\sup_{\mu_{x} \in X_{\mu}} \left\| \left( \sqrt{\lambda} I + L_{K_{0}}  \right)^{-\frac{1}{2}} K_{0}(\mu_{x},\cdot)  \right\|_{K_{0}}^{2} \leqslant \kappa^{2} \lambda^{-\frac{1}{2}}.
	\]
	One can check that
	\begin{equation*}
		\begin{aligned}
			& \int_{X_{\mu}} \| \left( \sqrt{\lambda} I + L_{K_{0}}  \right)^{-\frac{1}{2}} K_{0}(\mu_{x},\cdot)   \|_{K_{0}}^{2} d \rho_{X_{\mu}} (\mu_{x})  \\
			= & \int_{X_{\mu}} \left< K_{0}(\mu_{x},\cdot),  \left( \sqrt{\lambda} I + L_{K_{0}}  \right)^{-1} K_{0}(\mu_{x},\cdot) \right>_{K_{0}} d \rho_{X_{\mu}} (\mu_{x})  \\
			= & \int_{X_{\mu}} \left< \sum_{i\geqslant 1} \sigma_{i}\phi_{i}(\mu_{x}) \phi_{i} ,  \sum_{i\geqslant 1} \left( \sqrt{\lambda} I + L_{K_{0}}  \right)^{-1} \sigma_{i}\phi_{i}(\mu_{x}) \phi_{i}\right>_{K_{0}} d \rho_{X_{\mu}} (\mu_{x})  \\
			=& \sum_{i\geqslant 1} \sigma_{i} (\sqrt{\lambda} + \sigma_{i})^{-1} \int_{X_{\mu}} \phi_{i}^{2} (\mu_{x})  d \rho_{X_{\mu}} (\mu_{x}) =  \sum_{i\geqslant 1} \sigma_{i} (\sqrt{\lambda} + \sigma_{i})^{-1} = \mathcal{N}( \lambda^{\frac{1}{2}}  ).
		\end{aligned}
	\end{equation*}
	Combining these estimates with the boundedness assumption on output $ y $, we obtain that
	\[
		\| \xi_{1}^{\mathrm{I}}  \|_{K_{0}} \leqslant \kappa M \lambda^{-\frac{1}{4}}, \qquad \forall z = (\mu_{x},y) \in Z
	\]
	and
	\[
		\mathbb{E} \| \xi_{1}^{\mathrm{I}} \|_{K_{0}}^2 \leqslant M^2 \mathcal{N}(\lambda^{\frac{1}{2}} ).
	\]
	Then there exists a subspace $ V_{1} $ of $ Z^{m}  $ with measure at most $ \frac{\delta}{2} $ such that
	\[
		\mathcal{S}_{1}^{\mathrm{I}} (D,\lambda) \leqslant M \kappa ^{-1} \mathcal{B}_{m,\lambda} \log(\frac{4}{\delta}), \quad \forall D \in Z^{m} \backslash V_1,
	\]
	by Lemma \ref{lem:concentration} with $ \sigma =  M \sqrt{\mathcal{N}(\lambda^{\frac{1}{2}} )} $ and $ \widetilde{M} = \kappa M \lambda^{-\frac{1}{4}}  $.

	Since $ \mathcal{S}_{1}^{\mathrm{II}}(\lambda)  $ is independent of the sample $ D $, it can be directly bounded as
	\[
		\mathcal{S}_{1}^{\mathrm{II}} (\lambda) = \left\| \left( \sqrt{\lambda} I + L_{K_{0}}  \right)^{-\frac{1}{2}} L_{K_{0}}^{\frac{1}{2}} L_{K_{0}}^{\frac{1}{2}} (f_{\rho} - f_{\lambda}) \right\|_{K_0} \leqslant \left\| f_{\rho} - f_{\lambda} \right\|_{\rho_{X_{\mu}}}.
	\]

	We introduce the random variable
	\[
		\xi_{1}^{\mathrm{III}}(z) = \left( \sqrt{\lambda}I + L_{K_{0}}  \right)^{-\frac{1}{2}} K_{0}(\mu_{x},\cdot) f_{\lambda}(\mu_{x}),\quad z = (\mu_{x},y) \in Z.
	\]
	to estimate $ \mathcal{S}_{1}^{\mathrm{III}}(D,\lambda) $, which is similar to the estimation of $ \mathcal{S}_{1}^{\mathrm{I}}(D,\lambda) $.
	It also takes values in $ \mathcal{H}_{K_{0}} $ and
	\[
		\mathbb{E} \xi_{1}^{\mathrm{III}} = \left( \sqrt{\lambda}I + L_{K_{0}}  \right)^{-\frac{1}{2}} L_{K_{0}} f_{\lambda}.
	\]
	It follows that for any $ z=(\mu_{x},y) \in Z $,
	\begin{equation*}
		\begin{aligned}
			\left\| \xi_{1}^{\mathrm{III}}(z)  \right\|_{K_{0}}^{2} =& f_{\lambda}^2(\mu_{x}) \left\| \left( \sqrt{\lambda}I + L_{K_{0}}  \right)^{-\frac{1}{2}} K_{0}(\mu_{x},\cdot) \right\|_{K_{0}}^{2}
			\leqslant  \| f_{\lambda} \|_{\infty}^2 \kappa^2 \lambda^{-\frac{1}{2}},
		\end{aligned}
	\end{equation*}
	and
	\[
		\mathbb{E} \| \xi_{1}^{\mathrm{III}} \|_{K_0}^2 \leqslant \| f_{\lambda} \|_{\infty}^2 \mathcal{N}(\sqrt{\lambda}).
	\]
	Then by Lemma \ref{lem:concentration} with $ \sigma = \| f_{\lambda} \|_{\infty} \sqrt{\mathcal{N}(\sqrt{\lambda} )}  $ and $ \widetilde{M} = \kappa \| f_{\lambda} \|_{\infty} \lambda^{-\frac{1}{4}} $, there exists a subset $ V_2 $ of $ Z^{m} $ with measure at most $ \frac{\delta}{2} $ such that
	\begin{equation*}
		\begin{aligned}
			\mathcal{S}_{1}^{\mathrm{III}}(D,\lambda) = & \left\| \frac{1}{m} \sum_{i=1}^{m} \xi_1^{\mathrm{III}}(z_{i}) - \mathbb{E}\xi_{1}^{\mathrm{III}} \right\|_{K_{0}} \leqslant  \kappa ^{-1} \| f_{\lambda} \|_{\infty} \mathcal{B}_{m,\lambda} \log(\frac{4}{\delta}), \quad \forall D \in Z^{m} \backslash V_2.
		\end{aligned}
	\end{equation*}
	Note that the estimates above all hold for any $ D \in Z^{m} \backslash (V_1 \cup V_2) $. Hence we complete the proof of Lemma \ref{lem:S1} by combining these estimates together.

\subsection{Proof of Lemma \ref{lem:S2}}

	Recall the definition of $ \mathcal{S}_{2}(D,\lambda) $ in Lemma \ref{lem:fDMinusflambda}.
	Actually, it can be derived that
	\[	
		\mathcal{S}_{2}(D, \lambda) \leq \left\| \left( \sqrt{\lambda} I+L_{K_{0}} \right)\left( \sqrt{\lambda} I+\hat{T}_{0}^{\mathbf{x}} \right)^{-1} \right\|^{\frac{1}{2}}
	\]
	where we used the Cordes inequality \cite{fujii1993norm} for positive operators $ A $ and $ B $ on a Hilbert space and $ s =\frac12 $, i.e., $ \| A^{s} B^{s}  \| \leqslant \| AB \|^{s}  $ with $ s=\frac{1}{2} $.
	Note that $ \| AB \| = \| BA \| $ also holds true.
    We substitute $ A = \sqrt{\lambda} I+\hat{T}_{0}^{\mathbf{x}} $ and $ B = \sqrt{\lambda} I+L_{K_{0}} $ into (\ref{eq:ProdSecondorder}) to obtain
    \begin{equation*}
		\begin{aligned}
			& \left\| \left( \sqrt{\lambda} I+L_{K_{0}} \right) \left( \sqrt{\lambda} I+\hat{T}_{0}^{\mathbf{x}} \right)^{-1} \right\| \leqslant  1 + \left\| \left( L_{K_{0}}-\hat{T}_{0}^{\mathbf{x}} \right) \left(\sqrt{\lambda} I+L_{K_{0}}\right)^{-1} \right\| \\
            & \qquad \qquad \qquad + \left\| \left( L_{K_{0}}-\hat{T}_{0}^{\mathbf{x}} \right) \left(\sqrt{\lambda} I+L_{K_{0}}\right)^{-1}\left( L_{K_{0}}-\hat{T}_{0}^{\mathbf{x}} \right)  \left( \sqrt{\lambda} I+\hat{T}_{0}^{\mathbf{x}} \right)^{-1} \right\| \\
			\leqslant & 1 + \lambda^{-\frac{1}{4}} \left\| \left(\sqrt{\lambda} I+L_{K_{0}}\right)^{-\frac{1}{2}} \left( L_{K_{0}}-\hat{T}_{0}^{\mathbf{x}} \right)  \right\| +  \lambda^{-\frac{1}{2}} \left\| \left( \sqrt{\lambda} I+L_{K_{0}} \right)^{-\frac{1}{2}} \left( L_{K_{0}}-\hat{T}_{0}^{\mathrm{x}} \right) \right\|^{2}.
		\end{aligned}
	\end{equation*}
	where we use the bounds $ \| ( \sqrt{\lambda} I + L_{K_0}  )^{-\frac{1}{2}} \| \leqslant \lambda^{-\frac{1}{4}} $ and $ \| ( \sqrt{\lambda} I + \hat{T}_{0}^{\mathbf{x}} )^{-1} \| \leqslant \lambda^{-\frac{1}{2}} $.
    Applying the fact that $ \hat{T}_{0}^{\mathbf{x}} = T^{\mathbf{x}} U^{*} = U T_1^{\mathbf{x}} U^{*} $ and $ L_{K_{0}} = U L_{K_1} U^{*} $ yields
	\[
		\left\| \left( \sqrt{\lambda} I+L_{K_{0}} \right)^{-\frac{1}{2}} \left( L_{K_{0}} - \hat{T}_{0}^{\mathbf{x}} \right) \right\| = \left\| \left( \sqrt{\lambda} I + L_{K_{1}} \right)^{-\frac{1}{2}} \left( L_{K_{1}} - T_{1}^{\mathbf{x}} \right) \right\|.
	\]
	Finally we apply Lemma \ref{lem: T0X-LK0} to get
	\begin{equation*}
		\begin{aligned}
		\left\| \left( \sqrt{\lambda} I + L_{K_{0}} \right) \left( \sqrt{\lambda} I + \hat{T}_{0}^{\mathbf{x}} \right)^{-1} \right\| \leqslant & 1 + \lambda^{-\frac{1}{4}} \mathcal{B}(m,\lambda) \log (\frac{2}{\delta})  + \lambda^{-\frac{1}{2}} \left(\mathcal{B}(m,\lambda) \log (\frac{2}{\delta}) \right)^2 \\
		\leqslant & \left[ 1 + \lambda^{-\frac{1}{4}} \mathcal{B}(m,\lambda) \log (\frac{2}{\delta})  \right]^2,
		\end{aligned}
	\end{equation*}
	which proves our result.

\end{document}